\documentclass{article}

\usepackage[final, nonatbib]{neurips_2020}

\usepackage[utf8]{inputenc} % allow utf-8 input
\usepackage[T1]{fontenc}    % use 8-bit T1 fonts

\usepackage{microtype}
\usepackage{graphicx}
\usepackage{caption}
\usepackage[font=normal]{subcaption}
\usepackage{booktabs} % for professional tables\
\usepackage{xcolor}

\usepackage[ruled]{algorithm2e}
\usepackage{algorithmic}
\usepackage{amsmath}
\usepackage{multicol}
\usepackage{makecell}
\usepackage{array}
\usepackage{wrapfig}
\usepackage{arydshln}

\usepackage{tabu}
\usepackage{multirow}
\usepackage{multicol}
\usepackage{hyperref}

\usepackage{url}            % simple URL typesetting
\usepackage{amsfonts}       % blackboard math symbols
\usepackage{nicefrac}       % compact symbols for 1/2, etc.

\usepackage{amsmath, amssymb, amsthm, mathtools, mathrsfs}
\usepackage{physics}

\newcommand{\R}{\mathbb{R}} % Real Numbers
\newcommand{\E}{\mathbb{E}} % Expectation
\DeclareMathOperator*{\argmin}{arg\,min} % argmin
\DeclareMathOperator*{\argmax}{arg\,max} % argmax

\newcommand{\loss}{\ell} % Loss function

\newcommand{\Set}{\mathcal{S}}

\newcommand{\Loss}{\mathcal{L}}
\newcommand{\Error}{\mathcal{E}}

\newcommand{\A}{\mathbf{A}}
\newcommand{\B}{\mathbf{B}}
\newcommand{\BB}{\mathcal{B}}

\newcommand{\W}{\mathbf{W}}

\newcommand{\TT}{\mathcal{T}}
\newcommand{\m}{\mathbf{m}}

\newcommand{\vi}{\mathbf{v}}
\newcommand{\VV}{\mathcal{V}}
\newcommand{\w}{\mathbf{w}}
\newcommand{\x}{\mathbf{x}}

\newcommand{\defeq}{\vcentcolon=}

\newcommand{\zero}{\mathbf{0}}

\newtheorem{theory}{Theorem}
\newtheorem{lemma}{Lemma}
\newtheorem{assumption}{Assumption}
\newtheorem{corollary}{Corollary}
\newtheorem{proposition}{Proposition}

\newcolumntype{P}[1]{>{\centering\arraybackslash}p{#1}}

\usepackage[textwidth=3.5cm]{todonotes}

\title{On the Loss Landscape of Adversarial Training: \\ Identifying Challenges and How to Overcome Them}

\author{Chen Liu$^1$~~~Mathieu Salzmann$^1$~~~Tao Lin$^1$~~~Ryota Tomioka$^2$~~~Sabine Süsstrunk$^1$ \\
$^1$ EPFL, Lausanne, Switzerland, \{chen.liu, mathieu.salzmann, tao.lin, sabine.susstrunk\}@epfl.ch \\
$^2$ Microsoft Research, Cambridge, UK, ryoto@microsoft.com}

\begin{document}

\maketitle

\begin{abstract}
We analyze the influence of adversarial training on the loss landscape of machine learning models.
To this end, we first provide analytical studies of the properties of adversarial loss functions under different adversarial budgets.
% To this end, we first provide an analytical study on the properties of adversarial loss function over the size of the adversarial budget\tao{``the size of the adversarial budget'' is a weird}.
We then demonstrate that the adversarial loss landscape is less favorable to optimization, due to increased curvature and more scattered gradients.
% We then demonstrate that adversarial loss landscape is less favorable to optimization, because of increased non-smoothness and gradient scattering\tao{not sure if gradient scattering is a form term or not.}.
Our conclusions are validated by numerical analyses, which show that training under large adversarial budgets impede the escape from suboptimal random initialization, cause non-vanishing gradients and make the model find sharper minima.
% Our conclusions are validated by numerical analyses, which show that larger adversarial budgets impede the escape from suboptimal random initialization and make the models' minima sharper.
Based on these observations, we show that a periodic adversarial scheduling (PAS) strategy can effectively overcome these challenges, yielding better results than vanilla adversarial training while being much less sensitive to the choice of learning rate.
%introduce periodic adversarial scheduling (PAS), which we empirically show to yield better results than vanilla adversarial training and to be much less sensitive to the choice of learning rate.
\end{abstract}

% !TEX root = ../main.tex
% !TEX spellcheck = en-US

\section{Introduction} \label{sec:intro}

% \begin{wrapfigure}{R}{0.5\textwidth}
% 	\vspace{-1cm}
% 	\begin{center}
% 		\includegraphics[scale = 0.4]{figs/learn_curve/resnet_vanilla/compare.pdf}
% 	\end{center}
% 	\caption{\small
% 		The learning curve under different settings. The solid and faded lines stand for test and training error respectively. \textit{VT-16} demonstrates the error under clean data, while the others demonstrate the error under adversarial attacks in the original dataset. More details in Appendix~\ref{subsec:hyperparam}.
% 		\itao{The caption should be more precise. e.g. the error indicates what error}
% 	}
% 	\label{fig:example}
% \end{wrapfigure}

State-of-the-art deep learning models have been found to be vulnerable to adversarial attacks~\cite{goodfellow2014explaining, moosavi2017universal, szegedy2013intriguing}.
Imperceptible perturbations of the input can make the model produce wrong predictions with high confidence.
This raises concerns about deep learning's deployment in safety-critical applications.

Although many training algorithms have been proposed to counter such adversarial attacks, most of them were observed to fail when facing stronger attacks~\cite{athalye2018obfuscated, croce2020reliable}.
Adversarial training~\cite{madry2017towards} is one of the few exceptions, so far remaining effective and thus popular.
% \ms{nonetheless constitutes a rare exception, so far remaining effective and thus popular.}
%is one popular method that remains valid so far.
It uses adversarial examples generated with the attacker's scheme to update the model parameters.
However, adversarial training and its variants~\cite{alayrac2019labels, carmon2019unlabeled, hendrycks2019using, sinha2019harnessing, zhang2019you} have been found to have a much larger generalization gap~\cite{rice2020overfitting} and to require larger model capacity for convergence~\cite{xie2019intriguing}.
Although recent works~\cite{carmon2019unlabeled, schmidt2018adversarially} show that the adversarial training error reduces to almost $0\%$ with a large enough model and that the generalization gap can be narrowed by using more training data, convergence in adversarial training remains much slower than in vanilla training on clean data.
This indicates discrepancies in the underlying optimization landscapes.
While much work has studied the loss landscape of deep networks in vanilla training~\cite{draxler2018essentially, fort2019large, garipov2018loss, geiger2019jamming, li2018visualizing}, such an analysis in adversarial training remains unaddressed.

Here we study optimization in adversarial training.
Vanilla training can be considered as a special case where no perturbation is allowed, i.e., zero adversarial budget. Therefore, we focus on the impact of the adversarial budget size on the loss landscape.
In this context, we investigate from a theoretical and empirical perspective how different adversarial budget sizes affect the loss landscape to make optimization more challenging.
Our analyses start with linear models and then generalize to nonlinear deep learning ones.
We study the whole training process and identify different behaviors in the early and final stages of training.
% We consider both linear and nonlinear models, and study how the magnitude\tao{you use ``size'' in the abstract} of the adversarial budget affects their training properties.
% Specifically, we study how quickly the model can escape from the intial suboptimal region and point out the \textit{gradient scattering} problem in adversarial training.
Based on our observations, we then introduce a scheduling strategy for the adversarial budget during training. We empirically show this scheme to yield better performance and to be less sensitive to the learning rate than vanilla adversarial training.

\paragraph{Contributions.} Our contributions can be summarized as follows.
1) From a theoretical perspective, we show that, for linear models, adversarial training under a large enough budget produces a constant classifier.
For general nonlinear models, we identify the existence of an abrupt change in the adversarial examples, which makes the loss landscape less smooth.
This causes severe \textit{gradient scattering} and slows down the convergence of training.
2) Our numerical analysis shows that training under large adversarial budgets hinders the model to escape from suboptimal initial regions, while also causing large non-vanishing gradients in the final stage of training.
Furthermore, by Hessian analysis, we evidence that the minima reached in the adversarial loss landscape are sharper when the adversarial budget is bigger.
% Our numerical analysis shows that larger adversarial budgets hinder the model to escape from the suboptimal region after initializations.
% Furthermore, the minima found in the adversarial loss landscape is sharper than their vanilla counterpart.
3) We show that a periodic adversarial scheduling (PAS) strategy, corresponding to a cyclic adversarial budget scheduling scheme with warmup, addresses these challenges. Specifically, it makes training less sensitive to the choice of learning rate and yields better robust accuracy than vanilla adversarial training without any computational overhead.

\paragraph{Notation and Terminology.}
We use plain letters, bold lowercase letters and bold uppercase letters to represent scalars, vectors and matrices, respectively.
$\|\vi\|$ represents the Euclidean norm of vector $\vi$ and $[K]$ is an abbreviation of the set $\{0, 1, 2, ..., K - 1\}$.
In a classification problem $\{(\x_i, y_i)\}_{i = 1}^N$, where $(\x_i, y_i) \in \R^m \times [K]$, the classifier consists of a logit function $f: \R^m \to \R^k$, which is usually a neural network, and a risk function $\loss: \R^k \times [K] \to \R$, which is the softmax cross-entropy loss.
The adversarial budget $\Set^{(p)}_\epsilon(\x)$ of a data point $\x$, whose size is $\epsilon$, is defined based on an $l_p$ norm-based constraint $\{\x'| \|\x - \x'\|_p \leq \epsilon\}$, and we use $\Set_\epsilon(\x)$ to denote the $l_\infty$ constraint for simplicity.

Given the model parameters $\theta \in \Theta$, we use $g(\x, \theta): \R^m \times \Theta \to \R$ to denote the loss function for an individual data point, ignoring the label $y$ for simplicity.
If we use $\Loss_\epsilon(\theta)$ to denote the adversarial loss function under adversarial budget $\Set^{(p)}_\epsilon(\x)$, adversarial training solves the min-max problem
\begin{equation}
	\begin{aligned}
		\min_\theta \Loss_\epsilon(\theta) \defeq \frac{1}{N} \sum_{i = 1}^N g_\epsilon(\x_i, \theta)\ \ \rm{where}\ \ g_\epsilon(\x_i, \theta) \defeq \max_{\x'_i \in \Set^{(p)}_\epsilon(\x_i)} g(\x'_i, \theta)\;.
	\end{aligned} \label{eq:adv_problem}
\end{equation}
$\Loss(\theta) \defeq \Loss_0(\theta)$ is the vanilla loss function.
If $\epsilon \neq 0$, the adversarial example $\x'_i$, i.e., the worst-case input in $\Set^{(p)}_\epsilon(\x_i)$, depends on the model parameters.
We call the landscape of functions $\Loss(\theta)$ and $\Loss_\epsilon(\theta)$ the vanilla and adversarial loss landscape, respectively.
Similarly, we use $\Error(\theta)$ and $\Error_\epsilon(\theta)$ to represent the clean error and robust error under adversarial budget $\Set^{(p)}_\epsilon(\x)$.
In this paper, we call a function smooth if it is $C^1$-continuous.
We use $\theta_0$ to denote the initial parameters.
``Initial plateau'' or ``suboptimal region in the early stage of training'' indicate the parameters that are close to the initial ones and have similar performance.
``Vanilla training'' means training based on clean input data, while ``vanilla adversarial training'' represents the popular adversarial training method in~\cite{madry2017towards}.

% !TEX root = ../main.tex
% !TEX spellcheck = en-US

\section{Related Work}

\paragraph{Adversarial Robustness.}
In this work, we focus on white-box attacks, in which the attackers have access to the model parameters.
Compared with black-box attacks, white box attacks better solve the inner maximization problem in (\ref{eq:adv_problem}).
% White-box attacks are stronger than black-box attacks and better solve the inner maximization problem in (\ref{eq:adv_problem}).
In this context,~\cite{goodfellow2014explaining} proposes the fast gradient sign method (FGSM) to perturb the input in the direction of its gradient: $\x' = \x + \epsilon\ \mathrm{sign}(\triangledown_\x \Loss(\theta))$.
Projected gradient descent (PGD)~\cite{madry2017towards} extends FGSM by iteratively running it with a smaller step size and projecting the perturbation back to the adversarial budget.
Furthermore, PGD introduces randomness by starting at a random initial point inside the adversarial budget.
As a result, PGD generates much stronger adversarial examples than FGSM and is believed to be the strongest attack utilizing the network's first order information~\cite{madry2017towards}.

When it comes to robustness against attacks, some methods have been proposed to train provably robust models by linear approximation~\cite{balunovic2020adversarial, kolter2017provable, wong2018scaling}, semi-definite programming~\cite{raghunathan2018certified}, interval bound propagation~\cite{gowal2018effectiveness} or randomized smoothing~\cite{cohen2019certified, salman2019provably}.
However, these methods either only apply to a specific type of network, have a significant computational overhead, or are unstable.
% \tao{minor: maybe we can add some references for each limitation}
Furthermore, compared with adversarial training, they have been found to over-regularize the model and significantly decrease the clean accuracy~\cite{zhang2019towards}.

As a result, we focus on PGD-based adversarial training, which first generates adversarial examples $\x'$ by PGD and then uses $\x'$ to optimize the model parameters $\theta$.
In all our experiments, the adversarial loss landscape is approximated by the loss of adversarial examples found by PGD.

\paragraph{Loss Landscape of Deep Neural Networks.}
Many existing works focus on the vanilla loss landscape of the objective function in deep learning.
It is challenging, because the objective $\Loss(\theta)$ of a deep neural network is a high-dimensional nonconvex function, of which we only know very few properties.
\cite{kawaguchi2016deep} proves the nonexistence of poor local minima for general deep nonlinear networks.
\cite{lee2016gradient} shows that stochastic gradient descent (SGD) can almost surely escape the saddle points and converge to a local minimum.
For over-parameterized ReLU networks, SGD is highly likely to find a monotonically decreasing trajectory from the initialization point to the global optimum~\cite{safran2016quality}.

Furthermore, some works have studied the geometric properties of local minima in the loss landscape of neural networks.
In this context,~\cite{keskar2016large, novak2018sensitivity} empirically show that sharp minima usually have larger generalization gaps than flat ones.
Specifically, to improve generalization,~\cite{yao2018hessian} uses adversarial training to avoid converging to sharp minima in large batch training.
% Specially,~\cite{yao2018hessian} uses adversarial training to ease the sharp minima problem in large batch training.
However, the correspondence between sharp minima and poor generalization is based on empirical findings and sometimes controversial.
For example,~\cite{dinh2017sharp} shows counterexamples in ReLU networks by rescaling the parameters and claims that sharp minima can generalize as well as flat ones.
% This argument became controversial when~\cite{dinh2017sharp} shows counterexamples in ReLU networks and point out that sharp minima can generalize as well as flat ones\tao{i believe \cite{dinh2017sharp} need to be better presented otherwise the message of this paragraph is a bit strange}.
Moreover, different minima of the loss function have been found to be well-connected.
That is, there exist hyper-curves connecting different minima that are flat in the loss landscape~\cite{draxler2018essentially, garipov2018loss}.
\cite{zhao2020Bridging} further shows that the learned path connection can help us to effectively repair models that are vulnerable to backdoor or error-injection attacks.
Recently, some methods have been proposed to visualize the loss landscape~\cite{li2018visualizing, skorokhodov2019loss}, leading to the observation that networks of different architectures have surprisingly different landscapes.
Compared with chaotic landscapes, smooth and locally near-convex landscapes make gradient-based optimization much easier.

All of the above-mentioned works, however, focus on networks that have been optimized with vanilla training.
Here, by contrast, we study the case of adversarial training.

% !TEX root = ../main.tex
% !TEX spellcheck = en-US

\section{Theoretical Analysis} \label{sec:theory}

% In this section, we theoretically demonstrate the difference between $\Loss_\epsilon(\theta)$ and its vanilla counterpart $\Loss(\theta)$.
In this section, we conduct an analytical study of the difference between $\Loss_\epsilon(\theta)$ and $\Loss(\theta)$.
We start with linear classification models and then discuss general nonlinear ones.

\subsection{Linear Classification Models} \label{subsec:softmax}

For the simple but special case of logistic regression, i.e., $K = 2$, we can write the analytical form of $\Loss_\epsilon(\theta)$. We defer the detailed discussion of this case to Appendix~\ref{subsec:logistic}, and here focus on
% For the sepcial case of logitstic regression, i.e. $K = 2$, where we can write the analytical form of $\Loss_\epsilon(\theta)$, the detailed discussion is deferred to Appendix~\ref{subsec:logistic}.\tao{strange sentence}
linear multi-class classification, i.e., $K \geq 3$.
We parameterize the model by $\W \defeq \{\w_{i}\}_{i = 1}^K \in \R^{m \times K}$ and use $f(\W) = [\w_1^T\x, \w_2^T\x, ..., \w_K^T\x]$ as the logit function.
Therefore, the vanilla loss function is convex as $g(\x, \W) = \log(1 + \sum_{j \neq y}\exp^{(\w_j - \w_y)^T\x})$.
% Let the logit function be $f(\W) = [\w_1^T\x, \w_2^T\x, ..., \w_K^T\x]$, and $\W \defeq \{\w_{i}\}_{i = 1}^K \in \R^{m \times K}$ are the trainable parameters.
% The vanilla loss function $\Loss(\W)$ here is a convex function $g(\x, \W) = \log(1 + \sum_{j \neq y}\exp^{(\w_j - \w_y)\x})$.
Although $g_\epsilon(\x, \W)$ is also convex, it is no longer smooth everywhere.
It is then difficult to write a unified expression of $g_\epsilon(\x, \W)$.
So we start with the \textit{version space} $\VV_\epsilon$ of $g_\epsilon(\x, \W)$ defined as $\VV_\epsilon = \left\{\W \bigg\rvert (\w_i - \w_y)\x' \leq 0, \forall i \in [K], \x' \in \Set_\epsilon(\x)\right\}$.
% \begin{equation}
% 	\begin{aligned}
% 		\VV_\epsilon = \left\{\W \bigg\rvert (\w_i - \w_y)\x' \leq 0, \forall i \in [K], \x' \in \Set_\epsilon(\x)\right\}\;.
% 	\end{aligned} \label{eq:vv}
% \end{equation}

By definition, $\VV_\epsilon$ is the smallest convex closed set containing all solutions robust under the adversarial budget $\Set_\epsilon(\x)$.
The proposition below states that the version space $\VV_\epsilon$ shrinks with larger values of $\epsilon$.

\begin{proposition} \label{prop:version_space}
	Given the definition of the version space $\VV_\epsilon$, then $\VV_{\epsilon_2} \subseteq \VV_{\epsilon_1}$ when $\epsilon_1 \leq \epsilon_2$.
\end{proposition}

The proof of Proposition~\ref{prop:version_space} is very straightforward, we put it in Appendix~\ref{sub:app_proof_v_set}.

% Proposition~\ref{prop:version_space} is easy to prove, since $g_{\epsilon_1}(\x, \W) \leq g_{\epsilon_2}(\x, \W)$ for $\epsilon_1 \leq \epsilon_2$, and it is easy to see that $\forall \W \in \VV_\epsilon$, $g_\epsilon(\x, \W) \leq \log K$.
% This shows that the set of solutions for each individual data point becomes smaller as $\epsilon$ increases.

In addition to $\VV_\epsilon$, we define the set $\TT_\epsilon$ as $\TT_\epsilon = \left\{\W \bigg\rvert 0 \in \argmin_\gamma g_\epsilon(\x, \gamma \W) \right\}$.
% \begin{equation}
% 	\begin{aligned}
% 		\TT_\epsilon = \left\{\W \bigg\rvert 0 \in \argmin_\gamma g_\epsilon(\x, \gamma \W) \right\}\;.
% 	\end{aligned} \label{eq:t_eps}
% \end{equation}
$\TT_\epsilon$ is the set of all directions in which the optimal point is the origin;
that is, the corresponding models in this direction are all no better than a constant classifier.
% $\TT_\epsilon$ is the complement space of $\VV_\epsilon$ in binary classification.
% However, this is not true anymore when $K \geq 3$.
Although we cannot write the set $\TT_\epsilon$ in roster notation, we show in the theorem below that $\TT_\epsilon$ becomes larger as $\epsilon$ increases.

\begin{theory} \label{thm:t_set}
	Given the definition of $\TT_\epsilon$, then $\TT_{\epsilon_2} \subseteq \TT_{\epsilon_1}$ when $\epsilon_1 \geq \epsilon_2$.
	In addition, $\exists \bar{\epsilon}$ such that $\forall \epsilon \geq \bar{\epsilon}, \TT_\epsilon = \R^{m \times K}$.
	In this case, $\zero \in \argmin_{\W} g_\epsilon(\x, \W)$.
\end{theory}

We defer the proof of Theorem~\ref{thm:t_set} to Appendix~\ref{sub:app_proof_t_set}, where we also provide a lower bound for $\bar{\epsilon}$.
Theorem~\ref{thm:t_set} indicates that when the adversarial budget is large enough, the optimal point is the origin.
In this case, we will get a constant classifier, and training completely fails.

$\Loss_\epsilon(\W)$ is the average of $g_\epsilon(\x, \W)$ over the dataset, so Theorem~\ref{thm:t_set} and Proposition~\ref{prop:version_space} still apply if we replace $g_\epsilon$ with $\Loss_\epsilon$ in the definition of $\VV_\epsilon$ and $\TT_\epsilon$.
For nonlinear models like deep neural networks, these conclusions will not hold because $g_\epsilon(\x, \theta)$ is no longer convex.
Nevertheless, our experiments in Section~\ref{subsec:early_training} evidence the same phenomena as indicated by the theoretical analysis above.
Larger $\epsilon$ make it harder for the optimizer to escape the initial suboptimal region.
In some cases, training fails, and we obtain a constant classifier in the end.

\subsection{General Nonlinear Classification Models} \label{sec:theory_general}

For deep nonlinear neural networks, we cannot write the analytical form of $g(\x, \theta)$ or $g_\epsilon(\x, \theta)$.
To analyze such models, we follow~\cite{sinha2018certifiable} and assume the smoothness of the function $g$.

\begin{assumption} \label{assume:lip}
The function $g$ satisfies the following Lipschitzian smoothness conditions: 
%function $g$ satisfies the following conditions:
\begin{equation}
\begin{aligned}
\|g(\x, \theta_1) - g(\x, \theta_2)\| & \leq L_\theta \|\theta_1 - \theta_2\|\;, \\
\|\triangledown_{\theta} g(\x, \theta_1) - \triangledown_{\theta} g(\x, \theta_2)\| &\leq L_{\theta\theta} \|\theta_1 - \theta_2\|\;, \\
\|\triangledown_{\theta} g(\x_1, \theta) - \triangledown_{\theta} g(\x_2, \theta)\| &\leq L_{\theta\x} \|\x_1 - \x_2\|_p\;.
\end{aligned}
\end{equation}
\end{assumption}

Based on this, we study the smoothness of $\Loss_\epsilon(\theta)$.

\begin{proposition} \label{prop:lip}
	If Assumption~\ref{assume:lip} holds, then we have \footnote{Strictly speaking, $\Loss_\epsilon(\theta)$ is not differentiable at some point, so $\triangledown_\theta \Loss_\epsilon(\theta)$ might be ill-defined. In this paper, we use $\triangledown_\theta \Loss_\epsilon(\theta)$ for simplicity. Nevertheless, the inequality holds for any subgradient $\vi \in \partial_\theta \Loss_\epsilon(\theta)$.}
	\begin{equation}
		\begin{aligned}
			\|\Loss_\epsilon(\theta_1) - \Loss_\epsilon(\theta_2)\| & \leq L_\theta \|\theta_1 - \theta_2\| \;, \\
			\|\triangledown_\theta \Loss_\epsilon(\theta_1) - \triangledown_\theta \Loss_\epsilon(\theta_2)\| & \leq L_{\theta\theta} \|\theta_1 - \theta_2\| + 2\epsilon L_{\theta\x}\;.
		\end{aligned} \label{eq:lip}
	\end{equation}
\end{proposition}

The proof is provided in Appendix~\ref{sub:app_proof_lip}, in which we can see the upper bound in Proposition~\ref{prop:lip} is tight and can be achieved in the worst cases.
Proposition~\ref{prop:lip} shows that the first-order smoothness of the objective function is preserved under adversarial attacks, but the second-order smoothness is not.
That is to say, gradients in arbitrarily small neighborhoods in the $\theta$-space can change discontinuously.

The unsatisfying second-order property arises from the maximization operator defined in the functions $g_\epsilon$ and $\Loss_\epsilon$.
For function $g_\epsilon(\x, \theta)$, the non-smooth points are those where the optimal adversarial example $\x'$ changes abruptly in a sufficiently small neighborhood.
Formally, we use $\theta_1$ and $\x'_1$ to represent the model parameters and the corresponding optimal adversarial example.
We assume different gradients of the model parameters for different inputs.
If there exists a positive number $a > 0$ such that, $\forall \delta > 0$, we can find $\theta_2 \in \{\theta | \|\theta - \theta_1\| \leq \delta\}$, and the corresponding optimal adversarial example $\x'_2$ satisfies $\|\x'_1 - \x'_2\|_p > a$, then $\lim_{\theta \to \theta_1} \triangledown_\theta g_\epsilon(\x, \theta) \neq \triangledown_\theta g_\epsilon(\x, \theta_1)$.
$\Loss_\epsilon(\theta)$ is the aggregation of $g_\epsilon(\x, \theta)$ over the dataset, so it also has such non-smooth points.
% so its non-smooth points in the parameter space are the union of the non-smooth points for each data sample.
% If we can find $a > 0$ such that $\forall \delta > 0$, $\exists \theta_2 \in \{\theta | \|\theta - \theta_1\| \leq \delta\}$ and the corresponding optimal adversarial example $\x'_2$ is such that $\|\x'_1 - \x'_2\|_p > a$, then $\lim_{\theta \to \theta_1} \triangledown_\theta g_\epsilon(\x, \theta) \neq \triangledown_\theta g_\epsilon(\x, \theta_1)$.\tao{somehow confusing}
In addition, as the $2\epsilon L_{\theta\x}$ term in the second inequality of (\ref{eq:lip}) indicates, the adversarial examples can change more under a larger adversarial budget.
As a result, the (sub)gradients $\triangledown_\theta \Loss_\epsilon(\theta)$ can change more abruptly in the neighborhood of the parameter space.
That is, the (sub)gradients are more \textit{scattered} in the adversarial loss landscape.
% We provide a sketch illustrating this phenomenon in Figure~\ref{fig:diag} of Appendix~\ref{subsec:diag}.

\begin{figure}[!ht] 
\centering
\includegraphics[scale = 0.4]{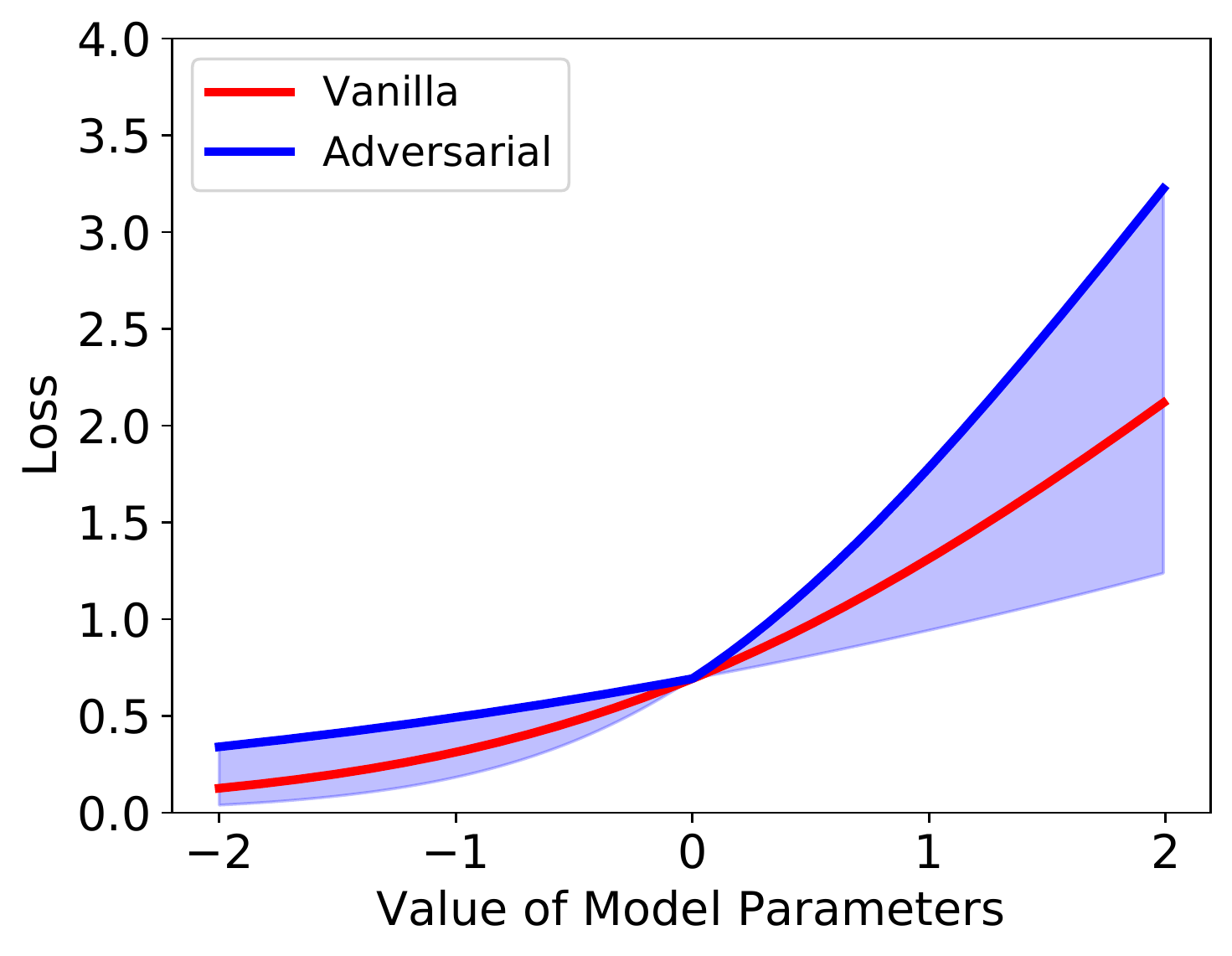}~~~~~~~~~~~~
\includegraphics[scale = 0.4]{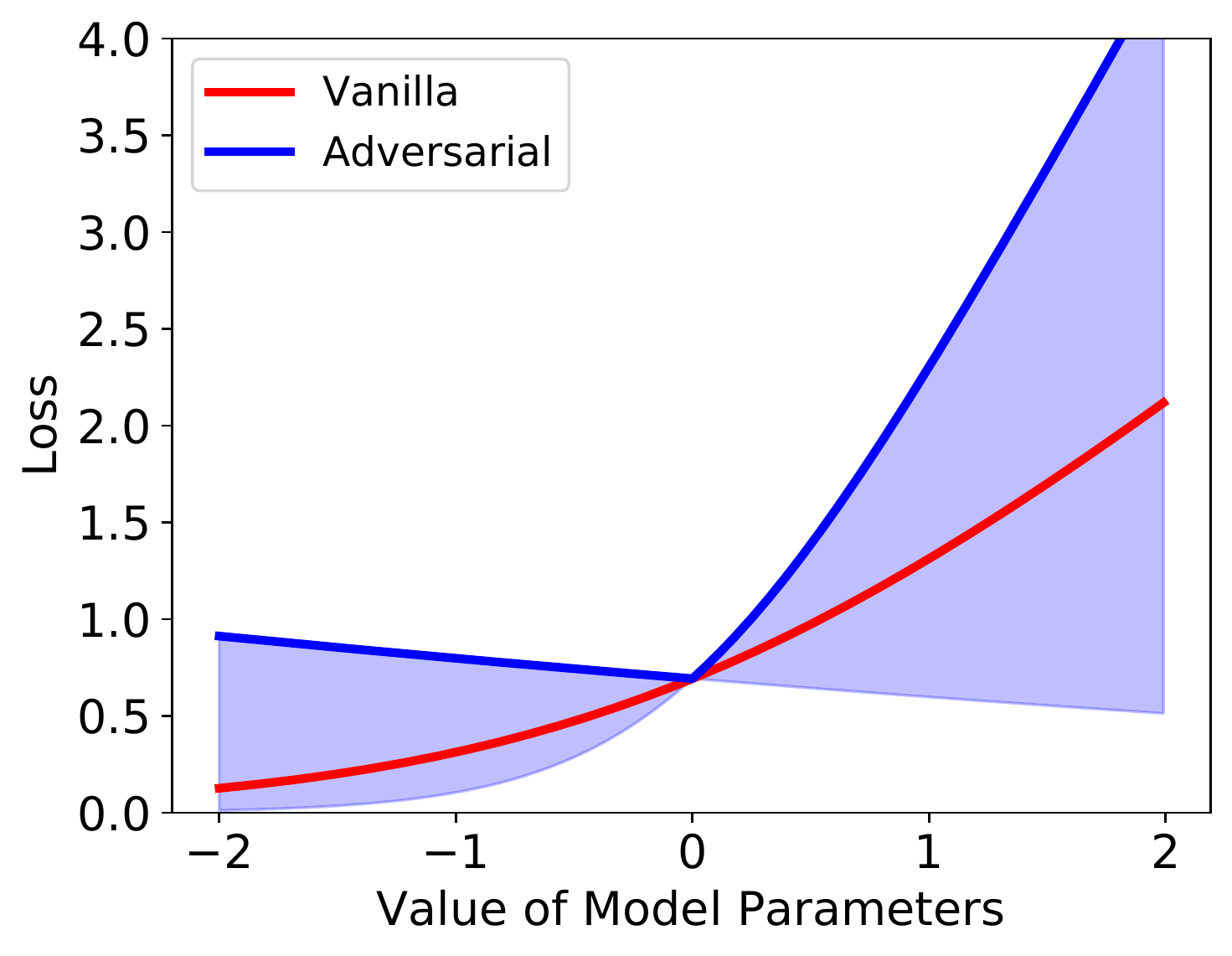}
\caption{2D sketch diagram showing the vanilla and adversarial loss landscapes. The clean input data $x$ is $1.0$ and loss function $g(x, \theta) = \log(1 + \exp(\theta x))$. The landscape is shown in the parameter interval $\theta \in [-2, 2]$ under a small adversarial budget (left, $\epsilon = 0.6$) and a large adversarial budget (right, $\epsilon = 1.2$). The function $g_\epsilon(\x, \theta)$ is not smooth at $\theta = 0$.} \label{fig:diag}
\end{figure}

Figure~\ref{fig:diag} provides a 2D sketch diagram showing the non-smoothness introduced by adversarial training.
The red curve represents the vanilla loss function $g(\x, \theta)$.
Under adversarial perturbation, the loss landscape fluctuates within the light blue band.
The blue curve represents the worst case we can encounter in the adversarial setting, i.e., $g_\epsilon(\x, \theta)$.
We can see that the blue curve is not smooth any more at the point where $\theta = 0$.
Importantly, as the light blue band becomes wider under a larger adversarial budget, the corresponding non-smooth point becomes sharper, which means that the difference between the gradients on both sides of the non-smooth point becomes larger.
% Because $\Loss_\epsilon(\theta)$ aggregates all the $g_\epsilon(\x, \theta)$s over the dataset, this property holds for $\Loss_\epsilon(\theta)$ as well.

Based on Proposition~\ref{prop:lip}, we show in the following theorem that the non-smoothness introduced by adversarial training makes the optimization by stochastic gradient descent (SGD) more difficult.

\begin{theory} \label{thm:convergence}
	Let Assumption~\ref{assume:lip} hold, the stochastic gradient $\triangledown_\theta \widehat{\Loss}_\epsilon(\theta_t)$ be unbiased and have bounded variance, and the SGD update $\theta_{t + 1} = \theta_t - \alpha_t \triangledown_\theta \widehat{\Loss}_\epsilon(\theta_t)$ use a constant step size $\alpha_t = \alpha = \frac{1}{L_{\theta\theta}\sqrt{T}}$ for $T$ iterations. Given the trajectory of the parameters during optimization $\{\theta_t\}_{t = 1}^T$, then we can bound the asymptotic probability of large gradients for a sufficient large value of $T$ as
	\begin{equation}
		\begin{aligned}
			\forall \gamma \geq 2, P(\|\triangledown_\theta \Loss_\epsilon(\theta_t)\| > \gamma \epsilon L_{\theta\x}) < \frac{4}{\gamma^2 - 2\gamma + 4}\;.
		\end{aligned} \label{eq:thm_convergence}
	\end{equation}
\end{theory}

We provide the proof in Appendix~\ref{subsec:convergence_proof}.
In vanilla training, $\epsilon$ is $0$ and $\Loss(\theta)$ is smooth, and (\ref{eq:thm_convergence}) implies that $\lim_{t \to +\infty} \|\triangledown_\theta g_\epsilon(\theta_t)\| = 0$ almost surely.
This is consistent with the fact that SGD converges to a critical point with non-convex smooth functions.
By contrast, in adversarial training, i.e., $\epsilon > 0$, we cannot guarantee convergence to a critical point.
Instead, the gradients are non-vanishing, and we can only bound the probability of obtaining gradients whose magnitude is larger than $2 \epsilon L_{\theta\x}$.
For a fixed value of $C \defeq \gamma \epsilon L_{\theta\x}$ larger than $2 \epsilon L_{\theta\x}$, the inequality (\ref{eq:thm_convergence}) indicates that the probability $P(\|\triangledown_\theta \Loss_\epsilon(\theta_t)\| > C)$ increases quadratically with $\epsilon$.

% \begin{figure} 
% \centering
% \includegraphics[scale = 0.4]{plots/small_eps.pdf}~~~~~~~~~~~~
% \includegraphics[scale = 0.4]{plots/big_eps.pdf}
% \caption{2D sketch diagram showing the vanilla and the adversarial loss landscape. The clean input data $x$ is $1.0$ and loss function $g(x, \theta) = \log(1 + \exp(\theta x))$. The landscape is shown in the parameter interval $\theta \in [-2, 2]$ under the small adversarial budget (left, $\epsilon = 0.6$) and the large adversarial budget (right, $\epsilon = 1.2$). Function $g_\epsilon(\x, \theta)$ is not smooth at $\theta = 0$.} \label{fig:diag}
% \end{figure}

% This phenomenon is illustrated by the 2D sketch diagrams of Figure~\ref{fig:diag}.
% The red curve represents the vanilla loss function $g(\x, \theta)$.
% Under adversarial perturbation, the loss landscape fluctuates within the light blue band.
% Then, the blue curve represents the worst case we can encounter in the adversarial setting, i.e., $g_\epsilon(\x, \theta)$.
% This blue curve does not satisfy the second-order smoothness anymore and the non-smooth points are highlighted by green dots.
% Importantly, the light blue band becomes wider under a larger adversarial budget, and the corresponding non-smooth points become sharper.
% As in the linear case, our conclusions on the function $g_\epsilon$ still hold for the full objective $\Loss_\epsilon$.

In deep learning practice, activation functions like sigmoid, tanh and ELU~\cite{clevert2015fast} satisfy the second-order smoothness in Assumption~\ref{assume:lip}, but the most popular ReLU function does not.
Nevertheless, adversarial training still causes \textit{gradient scattering} and makes the optimization more difficult.
That is, the bound of $\|\triangledown_\theta \Loss_\epsilon(\theta_1) - \triangledown_\theta \Loss_\epsilon(\theta_2)\|$ still increases with $\epsilon$, and the parameter gradients change abruptly in the adversarial loss landscape.
We provide a more detailed discussion of this phenomenon in Appendix~\ref{subsec:relu}, which shows that our analysis and conclusions easily extend to the ReLU case.
%, where all analysis and conclusions can be straightforwardly extended.

% Strictly speaking, ReLU networks do not satisfy the second-order smoothness in Assumption~\ref{assume:lip}.
% Nevertheless, Corollary~\ref{coro:lip} in Appendix~\ref{sub:app_coro_lip} shows that, under a weaker assumption, the upper bound of $\|\triangledown_\theta g_\epsilon(\x, \theta_1) - \triangledown_\theta g_\epsilon(\x, \theta_2)\|$ still increases with $\epsilon$.
% That is, the parameter gradient changes more abruptly in the adversarial loss landscape.

The second-order Lipchitz constant indicates the magnitude of the gradient change for a unit change in parameters.
Therefore, it is a good quantitative metric of gradient scattering.
In practice, we are more interested in the effective local Lipschitz constant, which only considers the neighborhood of the current parameters, than in the global Lipschitz constant.
In this case, the effective local second-order Lipschitz constant can be estimated by the top eigenvalues of the Hessian matrix $\triangledown^2_\theta \Loss_\epsilon(\theta)$.

% !TEX root = ../main.tex
% !TEX spellcheck = en-US

\section{Numerical Analysis} \label{sec:minima}

% In this section, we conduct experiments on MNIST and CIFAR10 to empirically validate the theorems in Section~\ref{sec:theory}.
% Detailed experimental settings are provided in Appendix~\ref{subsec:hyperparam}.
% The width of the network is denoted by a factor $w$ and is set as the value commonly used in vanilla training.
% Unless specified, we use $w = 16$ for LeNet (LeNet-16) on MNIST and $w = 8$ for ResNet18 (ResNet18-8) on CIFAR10.
% %\tao{adding a footnote here to discuss the original width in vanilla training might be helpful.}
% \MS{Why do we define $w$? We don't seem to use it at all in the experiments below?}

In this section, we conduct experiments on MNIST and CIFAR10 to empirically	validate the theorems in Section~\ref{sec:theory}.
Detailed experimental settings are provided in Appendix~\ref{subsec:hyperparam}.
Unless specified, we use LeNet models on MNIST and ResNet18 models on CIFAR10 in this and the following sections.
Our code is available on \href{https://github.com/liuchen11/AdversaryLossLandscape}{https://github.com/liuchen11/AdversaryLossLandscape}.

\subsection{Gradient Magnitude} \label{subsec:early_training}

In Section~\ref{subsec:softmax}, we have shown that the training algorithm will get stuck at the origin and yield a constant classifier for linear models under large $\epsilon$.
For deep nonlinear models, the initial value of the parameters is close to the origin under most popular initialization schemes~\cite{glorot2010understanding, he2015delving}.
%\tao{where you justify this point?}
Although Theorem~\ref{thm:t_set} is not applicable here, we are still interested in investigating how effective gradient-based optimization is at escaping from the suboptimal initial parameters.
To this end, we track the norm of the stochastic gradient $\|\triangledown_\theta \widehat{\Loss}_\epsilon(\theta)\|$, the robust error $\Error_\epsilon(\theta)$ in the training set and the distance from the initial point $\|\theta - \theta_0\|$ during the first 2000 mini-batch updates for CIFAR10 models.
Figure~\ref{subfig:resnet_grad_early},~\ref{subfig:resnet_error_early},~\ref{subfig:resnet_distance_early} evidence a clear difference between the models trained with different values of $\epsilon$.
When $\epsilon$ is small, the gradient magnitude is larger, and the model parameters move faster.
Correspondingly, the training error decreases faster, which means that the model quickly escapes the initial suboptimal region.
By contrast, when $\epsilon$ is large, the gradients are small, and the model gets stuck in the initial region.
This implies that the loss landscape under a large adversarial budget impedes the escape from initial suboptimal plateaus in the early stage of training.

\begin{figure}[!ht]
	\centering
	\begin{tabular}{cccc}
		\begin{subfigure}{0.24\textwidth}
			\includegraphics[scale = 0.22]{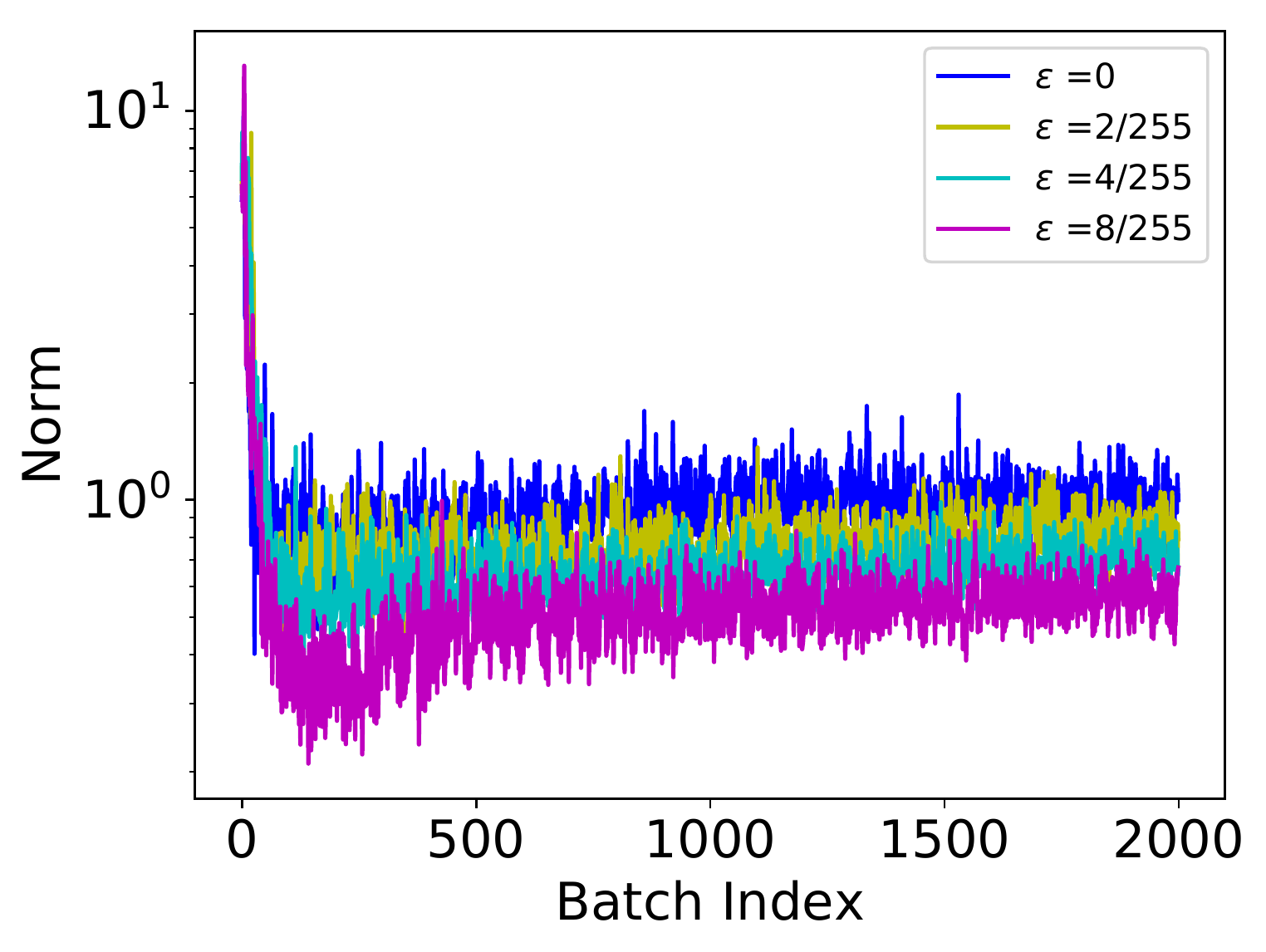}
			\captionsetup{font=small}
			\caption{$\|\triangledown_\theta \widehat{\Loss}_\epsilon(\theta)\|$, first 2000.} \label{subfig:resnet_grad_early}
		\end{subfigure}
		\begin{subfigure}{0.24\textwidth}
			\includegraphics[scale = 0.22]{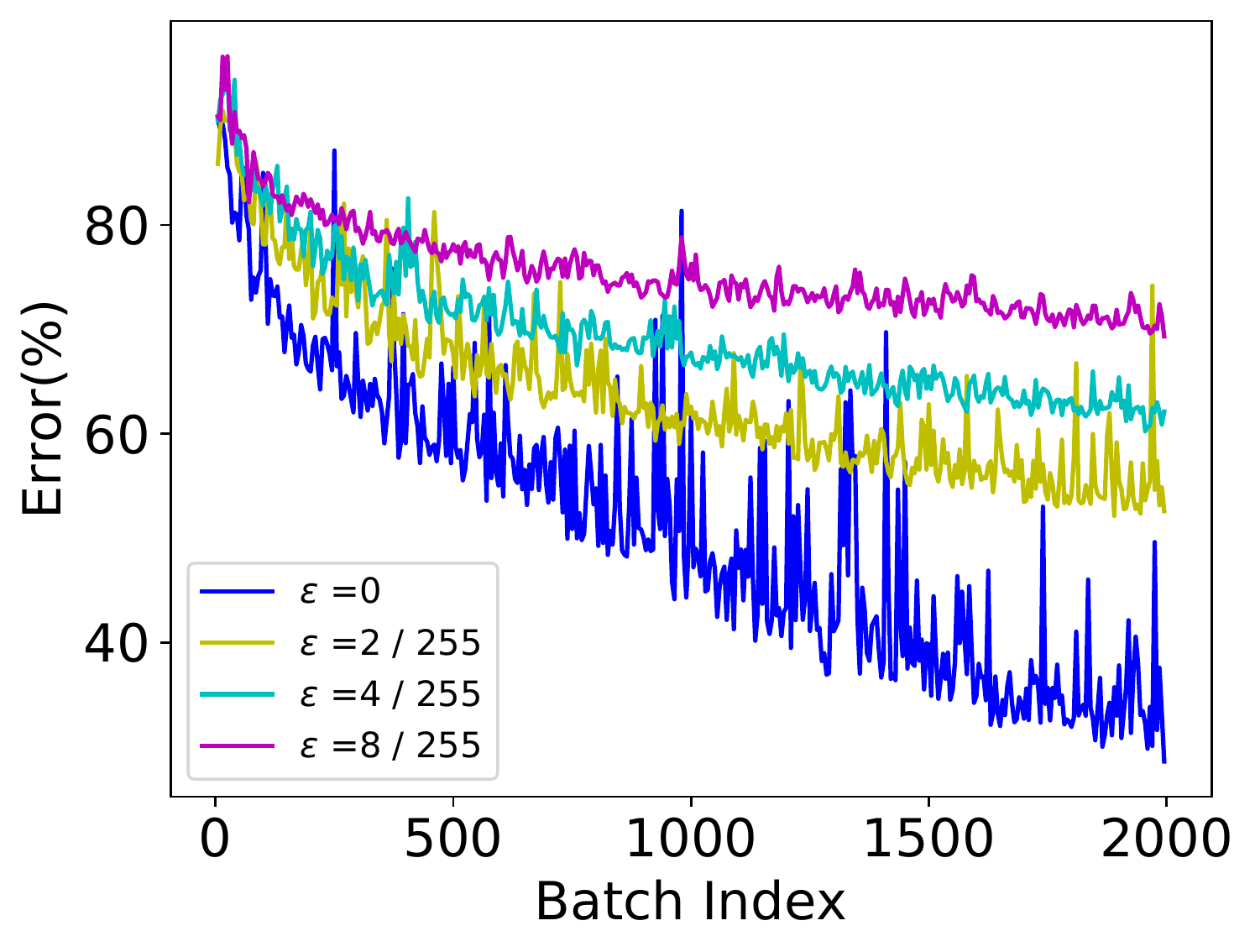}
			\captionsetup{font=small}
			\caption{$\Error_\epsilon(\theta)$, first 2000.} \label{subfig:resnet_error_early}
		\end{subfigure}
		\begin{subfigure}{0.24\textwidth}
			\includegraphics[scale = 0.22]{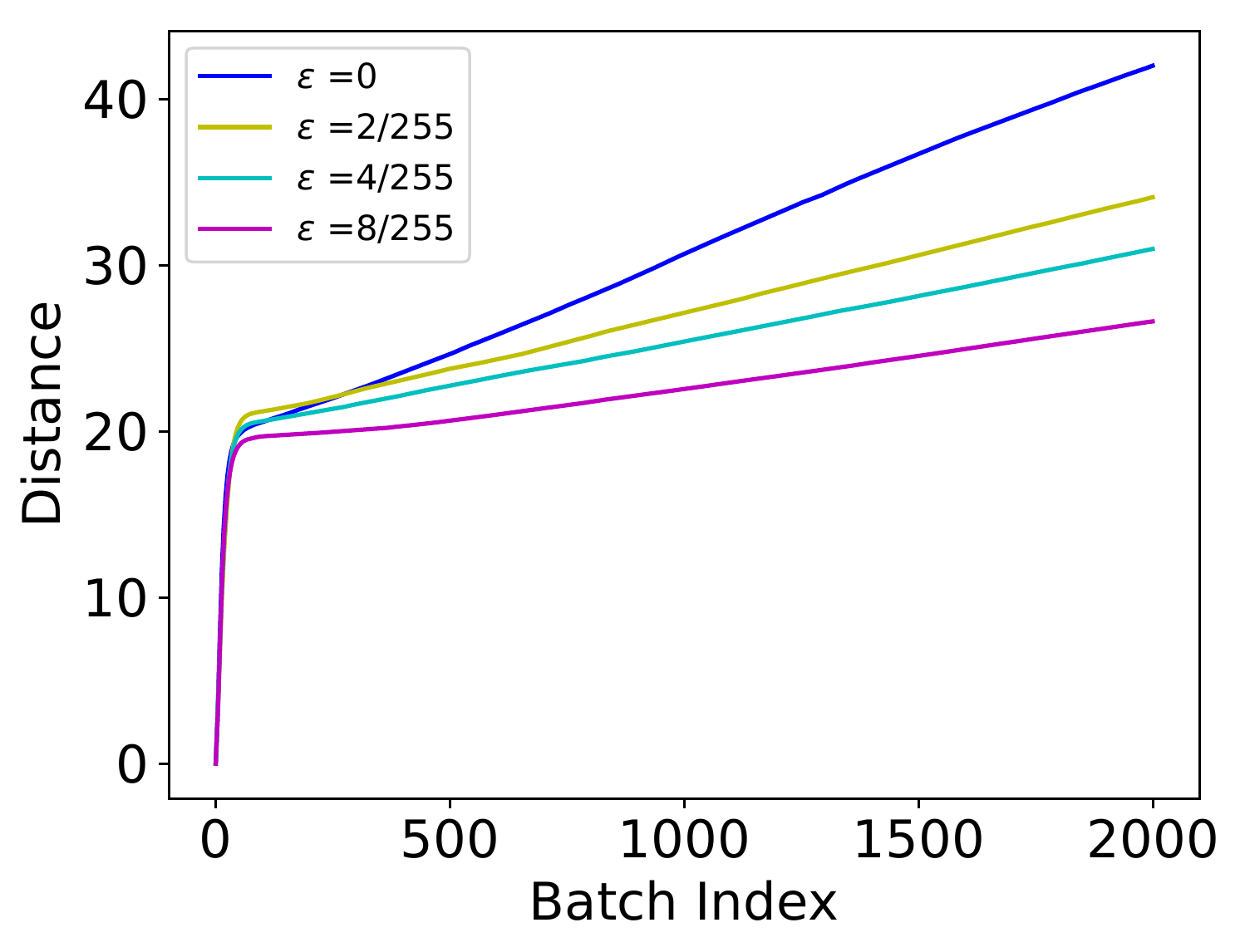}
			\captionsetup{font=small}
			\caption{$\|\theta - \theta_0\|$, first 2000.} \label{subfig:resnet_distance_early}
		\end{subfigure}
		\begin{subfigure}{0.24\textwidth}
			\includegraphics[scale = 0.22]{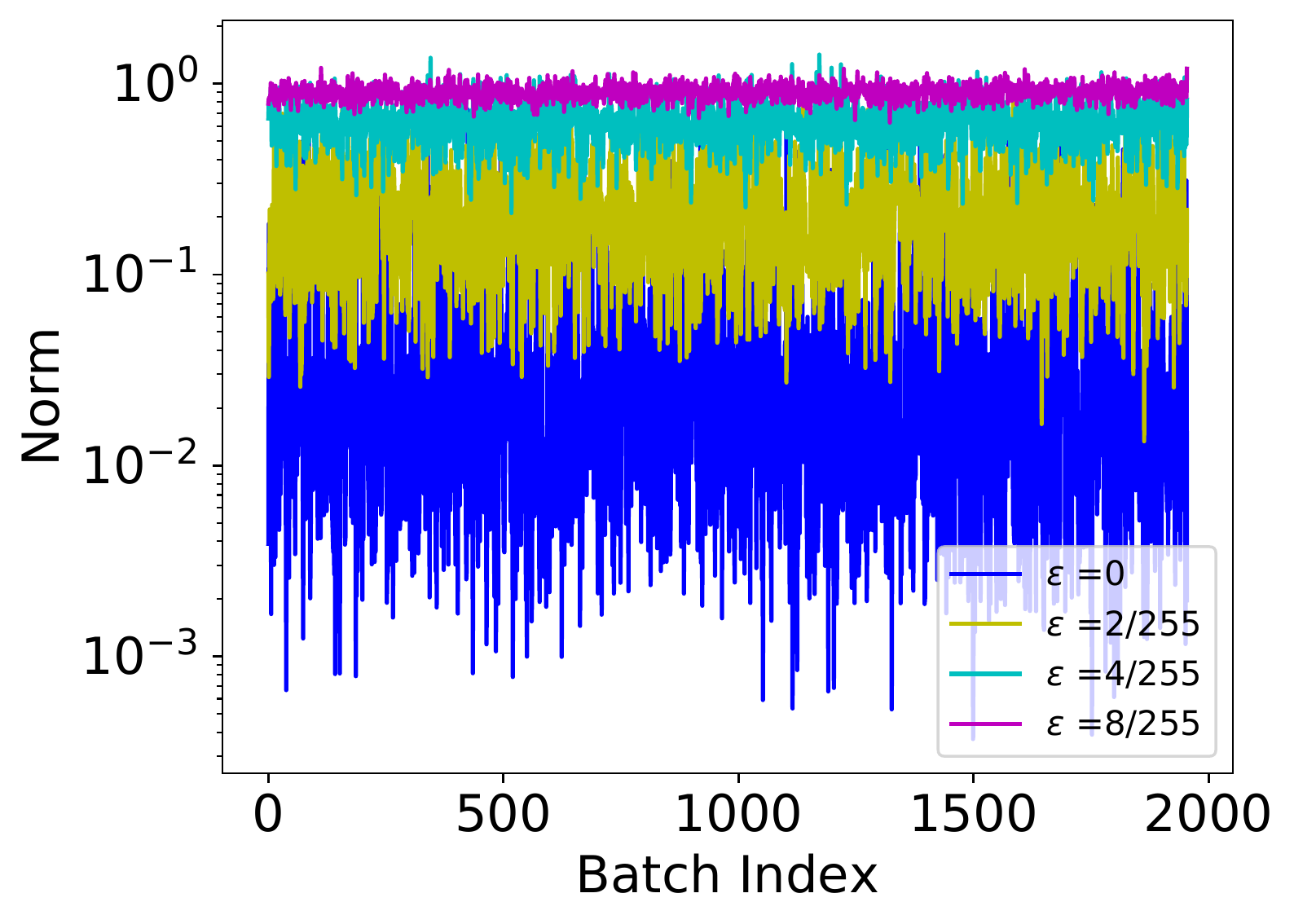}
			\captionsetup{font=small}
			\caption{$\|\triangledown_\theta \widehat{\Loss}_\epsilon(\theta)\|$, last 2000.} \label{subfig:resnet_grad_final}
		\end{subfigure}
	\end{tabular}
	\caption{Norm of the stochastic gradient $\|\triangledown_\theta \widehat{\Loss}_\epsilon(\theta)\|$, robust training error $\Error_\epsilon(\theta)$, and distance from the initial point $\|\theta - \theta_0\|$ during the first or last 2000 mini-batch updates for CIFAR10 models.} \label{fig:analysis_resnet}
\end{figure}

For ReLU networks, adversarially-trained models have been found to have sparser weights and intermediate activations~\cite{croce2018provable}, i.e., they have more dead neurons.
Dead neurons are implicitly favored by adversarial training, because the output is independent of the input perturbation.
Note that training fails when all the neurons in one layer are dead for all training instances.
The model is then effectively broken into two parts by this dead layer: the preceding layers will no longer be trained because the gradients are all blocked; the following layers do not depend on the input and thus give constant outputs.
In essence, training is then stuck in a parameter space that only includes constant classifiers.
In practice, this usually happens when the model has small width and the value of $\epsilon$ is large.
% This happens when the model is small and the value of $\epsilon$ is large\tao{mabye we need some justifications here?}.

% While the failure of adversarial training under small model capacity presented in~\cite{madry2017towards} corresponds to this case, the intrinsic reason behind this failure was not discussed.
% In fact, if one removes all ReLU functions, the corresponding linear model achieves non-trivial robustness, despite having an even smaller capacity. The true reason behind this failure therefore is not just capacity, but also the fact that the adversarial loss landscape prevents the optimizer from escaping the initial suboptimal region.
% We provide our results in Table~\ref{tbl:lenet_activation} and more detail in Appendix~\ref{subsubsec:activation}.
% Similar to our analysis of linear models in Section~\ref{subsec:softmax}, which shows that large adversarial budgets can make the optimizer converge to a constant classifier, deep nonlinear models may also get stuck in a parameter region that corresponds to constant classifiers.

Theorem~\ref{thm:convergence} indicates that the gradients are non-vanishing in adversarial training and more likely to have large magnitude under large values of $\epsilon$.
This is validated by Figure~\ref{subfig:resnet_grad_final}, in which we report the norm of the stochastic gradient $\|\triangledown_\theta \widehat{\Loss}_\epsilon(\theta)\|$ in the last 2000 mini-batch updates for CIFAR10 models.
In vanilla training, the gradient is almost zero in the end, indicating that the optimizer finds a critical point.
In this case $\|\triangledown_\theta \widehat{\Loss}_\epsilon(\theta)\|$ is dominated by the variance introduced by stochasticity.
However, $\|\triangledown_\theta \widehat{\Loss}_\epsilon(\theta)\|$ increases with $\epsilon$.
When $\epsilon$ is larger, $\|\triangledown_\theta \widehat{\Loss}_\epsilon(\theta)\|$ is also 
%considerably 
larger and non-vanishing, indicating that the model is still bouncing around the parameter space at the end of training.

The decreased gradient magnitude in the initial suboptimal region and the increased gradient magnitude in the final near-minimum region indicate that the adversarial loss landscape is not favorable to optimization when we train under large adversarial budgets.
Additional results on MNIST models are provided in Figure~\ref{fig:analysis_lenet} of Appendix~\ref{subsubsec:analysis_lenet}, where the same observations can be made.

% In Appendix~\ref{subsec:connect}, we further study the connectivity of model parameters in different runs.
% We know local minima in vanilla loss landscape is well-connected~\cite{garipov2018loss, draxler2018essentially} and we can find a flat path between them.
% However, the model parameters we obtain in adversarial training is found less-connected.

\subsection{Hessian Analysis} \label{sec:hessian_eps}

To study the effective local Lipschitz constant of $\Loss_\epsilon(\theta)$, we analyze the Hessian spectrum of models trained under different values of $\epsilon$.
It is known that the curvature in the neighborhood of model parameters is dominated by the top eigenvalues of the Hessian matrix $\triangledown^2 \Loss_\epsilon(\theta)$.
To this end, we use the power iteration method as in~\cite{yao2018hessian} to iteratively estimate the top 20 eigenvalues and the corresponding eigenvectors of the Hessian matrix.
Furthermore, to discard the effect of the scale of function $\Loss_\epsilon(\theta)$ for different $\epsilon$, we estimate the scale of $\Loss_\epsilon(\theta)$ by randomly sampling $\theta$.
We then normalize the top Hessian eigenvalues by the average value of $\Loss_\epsilon(\theta)$ on these random samples.
In addition, we show the learning curve of $\Loss_\epsilon(\theta)$ on the training set during training in Figure~\ref{fig:train_curve} of Appendix~\ref{subsubsec:add_hessian}.
It clearly show similar magnitude of $\Loss_\epsilon(\theta)$ for different values of $\epsilon$.

\begin{figure}[!ht]
	\begin{minipage}{.46\textwidth}
		\centering
		\includegraphics[height = 4cm]{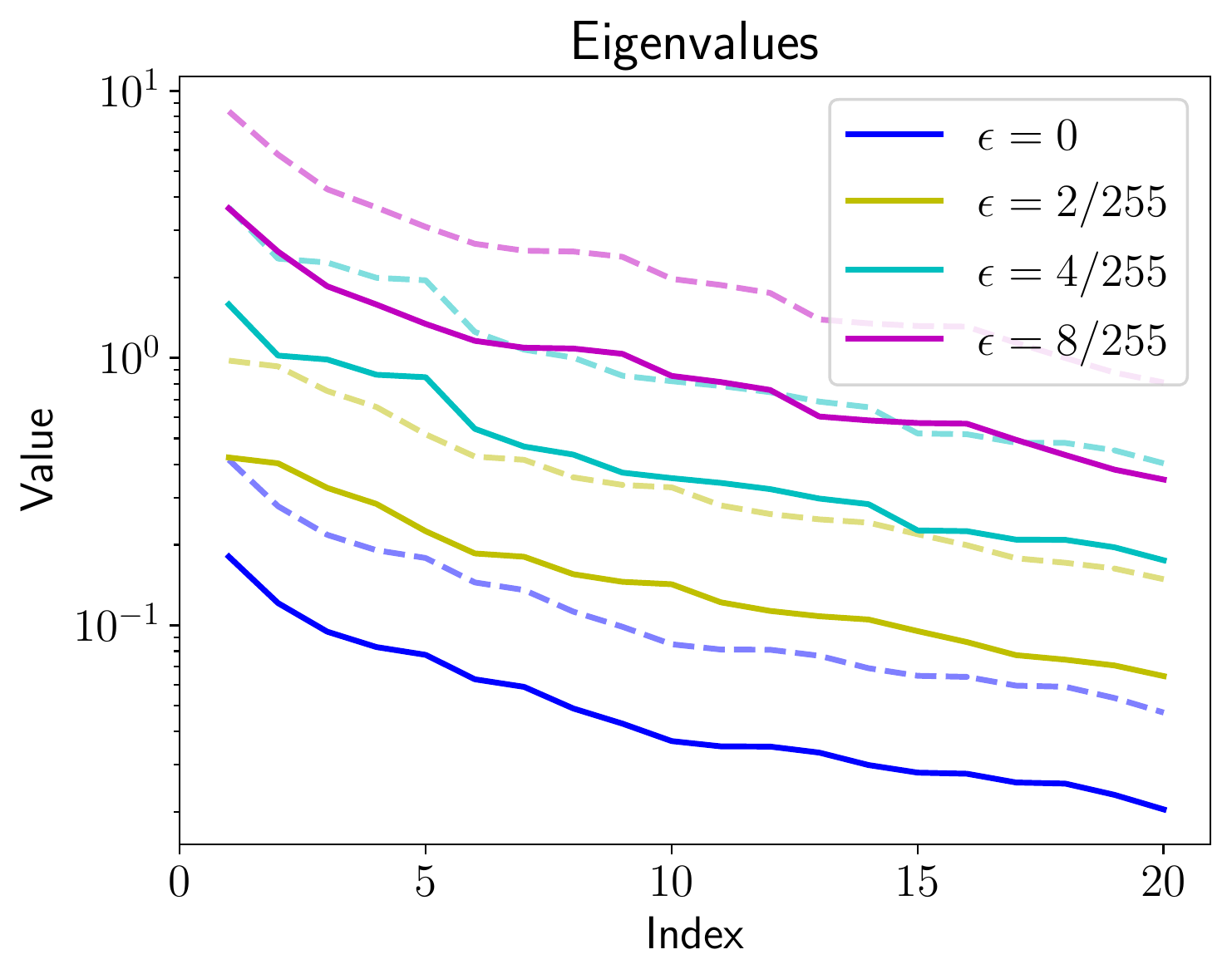}
		\caption{Top 20 eigenvalues of the Hessian matrix for ResNet18 models. We show the normalized (solid) and original (dashed) values.} \label{fig:eigen_eps_cifar10}
	\end{minipage}
	\begin{minipage}{.06\textwidth}
	~~~
	\end{minipage}
	\begin{minipage}{.46\textwidth}
		\centering
		\includegraphics[height = 4cm]{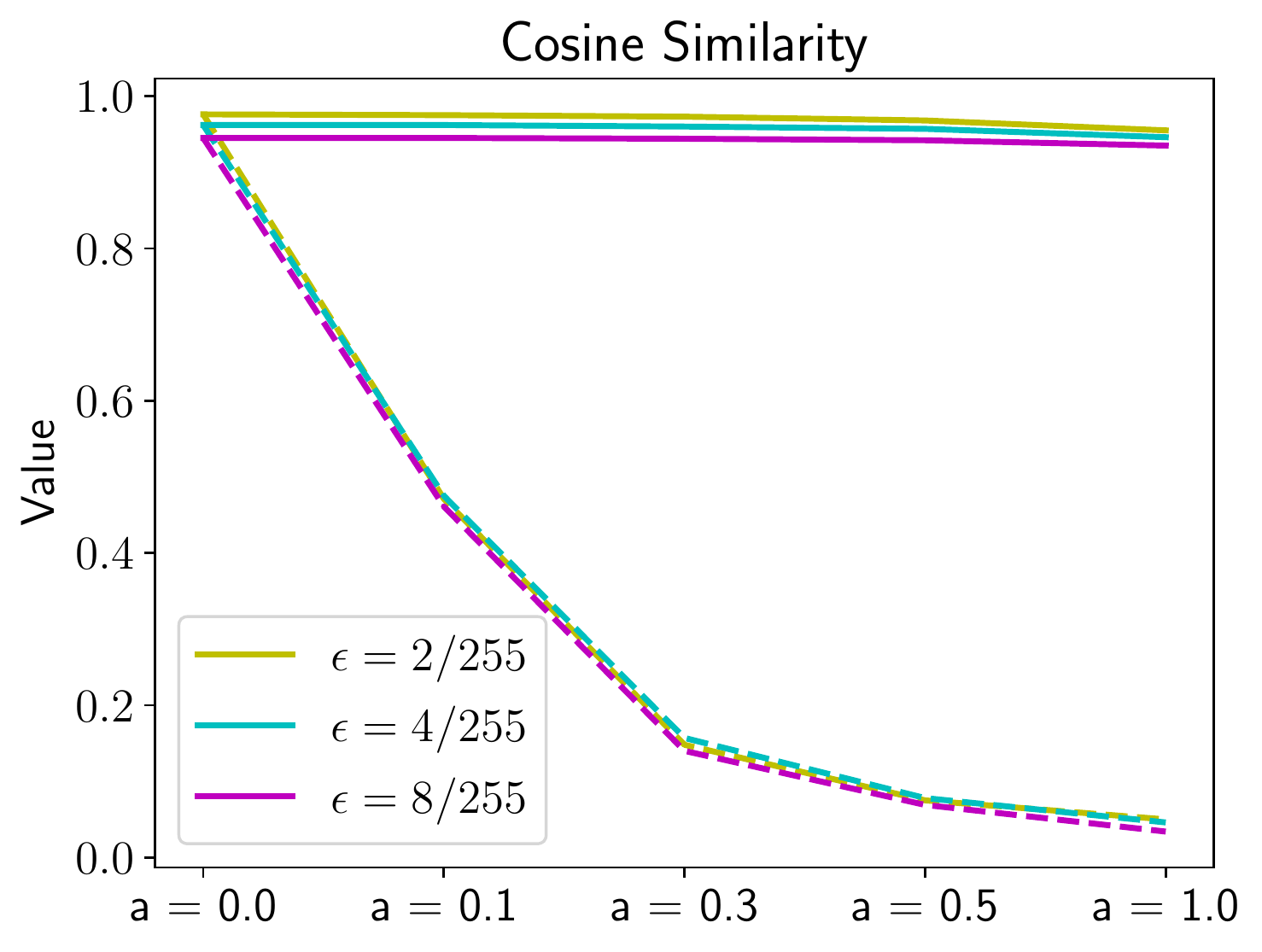}
		\caption{Cosine similarity between perturbations $\x'_{a\vi} - \x$ and $\x'_{-a\vi} - \x$. $\vi$ is either the top eigenvector (dashed) or random (solid).} \label{fig:sim_cifar10}
	\end{minipage}
\end{figure}

% \begin{figure}[!ht]
% \centering
% \begin{tabular}{cc}
% \begin{subfigure}{0.4\textwidth}
% \includegraphics[scale = 0.35]{figs/param_scan/lenet_eps/hessian_spec_compare_log.pdf}
% \captionsetup{font=normal}
% \caption{LeNet for MNIST} \label{fig:eigenvalue_eps_cifar}
% \end{subfigure}
% \begin{subfigure}{0.4\textwidth}
% \includegraphics[scale = 0.35]{figs/param_scan/resnet_eps/hessian_spec_compare_log.pdf}
% \captionsetup{font=normal}
% \caption{ResNet18 for CIFAR10} \label{fig:eigenvalue_eps_mnist}
% \end{subfigure}
% \end{tabular}
% \caption{The top 20 eigenvalues  of Hessian matrix.} \label{fig:eigenvalue_eps}
% \end{figure}

In Figure~\ref{fig:eigen_eps_cifar10}, we show the top 20 Hessian eigenvalues, both before and after normalization, of CIFAR10 models under different adversarial budgets.
We also provide 3D visualizations of the neighborhood in the directions of the top 2 eigenvectors in Figure~\ref{fig:hessian_3d} of Appendix~\ref{subsubsec:add_hessian}.
It is clear that the local effective second-order Lipschitz constant of the model obtained consistently increases with the value of $\epsilon$.
That is, the minima found in $\Loss_\epsilon(\theta)$ are sharper under larger $\epsilon$.
% Note that\tao{why we need this sentence here?} the scale of $\Loss_\epsilon(\theta)$ does not change significantly with $\epsilon$, so adversarial training finds sharper minima under larger $\epsilon$.
% In Appendix~\ref{subsubsec:add_hessian}, we provide 3D visualizations of the neighborhood in the directions of top 2 eigenvectors in Figure~\ref{fig:hessian_3d} and show that the scale of $\Loss_\epsilon(\theta)$ does not change significantly with $\epsilon$.
% It is consistent that larger adversarial budgets lead to sharper minima and thus that $\Loss_\epsilon(\theta)$ has larger local effective second-order Lipschitz constants.

% In Figure~\ref{fig:hessian_eig_eps_resnet}, we show the top 20 Hessian eigenvalues of ResNet18-8 models trained with different adversarial budgets.
% The minima reached when using larger adversarial budgets consistently have larger eigenvalues.
% We also provide 3D visualizations of the neighborhood in the directions of top 2 eigenvectors.
% This further demonstrates that larger adversarial budgets lead to sharper minima and thus that $\Loss_\epsilon(\theta)$ has larger local effective second-order Lipschitz constants.
% We provide the results of LeNet-16 models on MNIST in Figure~\ref{fig:hessian_eig_eps_lenet} of Appendix~\ref{subsubsec:add_hessian} and observe the same phenomenon.

To validate the claim in Section~\ref{sec:theory_general} that non-smoothness arises from abrupt changes of the adversarial examples, we study the similarity of adversarial perturbations generated by different model parameter values in a small neighborhood.
Specifically, we perturb the model parameters $\theta$ in opposite directions to $\theta + a\vi$ and $\theta - a\vi$, where $\vi$ is a unit vector and $a$ is a scalar.
Let $\x'_{a\vi}$ and $\x'_{-a\vi}$ represent the adversarial examples generated by the corresponding model parameters.
We then calculate the average cosine similarity between the perturbation $\x'_{a\vi} - \x$ and $\x'_{-a\vi} - \x$ over the training set.

The results on CIFAR10 models are provided in Figure~\ref{fig:sim_cifar10}.
To account for the random start in PGD, we run each experiment $4$ times and report the average value.
The variances of all experiments are smaller than $0.005$ and thus not shown in the figure.
Note that, when $\vi$ is a random unit vector, the robust error $\Error_\epsilon(\theta)$ of the parameters $\theta \pm a \vi$ on both the training and test sets remains unchanged for different values of $a$, indicating a flat landscape in the direction $\vi$.
The adversarial examples in this case are mostly similar and have very high cosine similarity.
By contrast, if $\vi$ is the top eigenvector of the Hessian matrix, i.e., the most curvy direction, then we see a sharp increase in the robust error $\Error_\epsilon(\theta)$ when we increase $a$.
Correspondingly, the cosine similarity between the adversarial perturbations is much lower, which indicates dramatic changes of the adversarial examples.
We perform the same experiments on MNIST models in Appendix~\ref{subsubsec:add_hessian} with the same observations.

% !TEX root = ../main.tex
% !TEX spellcheck = en-US

\section{Periodic Adversarial Scheduling} \label{sec:pas}

In Sections~\ref{sec:theory} and~\ref{sec:minima}, we have theoretically and empirically shown that the adversarial loss landscape becomes less favorable to optimization under large adversarial budgets.
In this section, we introduce a simple adversarial budget scheduling scheme to overcome these problems.

Inspired by the learning rate warmup heuristic used in deep learning~\cite{gotmare2018closer, huang2017snapshot}, we introduce warmup for the adversarial budget.
Let $d$ be the current epoch index and $D$ be the warmup period's length. We define a cosine scheduler $\epsilon_{cos}$ and a linear scheduler $\epsilon_{lin}$, parameterized by $\epsilon_{max}$ and $\epsilon_{min}$, as
% Let $p = \frac{d}{D}$ where $d$ is the current epoch index and $D$ is the length of warmup period, we then introduce cosine scheduler $\epsilon_{cos}$ and linear scheduler $\epsilon_{lin}$, parameterized by $\epsilon_{max}$ and $\epsilon_{min}$, as follows:
% Inspired by the learning rate warmup heuristic commonly used in deep learning~\cite{huang2017snapshot,gotmare2018closer}, we use a periodic warmup function for the adversarial budget.
% Let the increasing list $\{T_0, T_1, ..., T_M\}$ contain the epoch indices where the adversarial budget is reset, and the current epoch index $d$ satisfy $T_{i - 1} \leq d < T_i$.
% We define the ratio $p = \frac{d - T_{i - 1}}{T_{i} - T_{i - 1}}$ and introduce the two following rules, cosine scheduler $\epsilon_{cos}$ and linear scheduler $\epsilon_{lin}$, parameterized by $\epsilon_{max}$ and $\epsilon_{min}$ to schedule $\epsilon$:
\begin{equation}
	\begin{aligned}
		\epsilon_{cos}(d) = \frac{1}{2}(1 - \cos{\frac{d}{D}\pi})(\epsilon_{max} - \epsilon_{min}) + \epsilon_{min},\
		\epsilon_{lin}(d) = (\epsilon_{max} - \epsilon_{min})\frac{d}{D} + \epsilon_{min}\;.
	\end{aligned} \label{eq:eps_schedule}
\end{equation}

We clip $\epsilon_{cos}(d)$ and $\epsilon_{lin}(d)$ between $0$ and $\epsilon_{target}$, the target value of $\epsilon$.
If $\epsilon_{min} \leq 0$ and $\epsilon_{max} > \epsilon_{target}$, the value of $\epsilon$ starts from $0$, gradually increases to $\epsilon_{target}$ and remains constant then.

% Such warmup strategy is necessary sometimes, because, as mentioned in Section~\ref{subsec:early_training}, training can fail when $\epsilon$ is intially large.
This warmup strategy allows us to overcome the fact, highlighted in the previous sections, that adversarial training is more sensitive to the learning rate under a large budget because the gradients are more scattered.
%Adversarial training is found more sensitive to the learning rate under a large budget, because the gradients are more scattered.
%As a result, warmup strategy is necessary sometimes, otherwise the training can fail.
This is evidenced by Figure~\ref{fig:lr_lenet16}, which compares the robust test error of MNIST models relying on different adversarial budget scheduling schemes. For all models, we used $\epsilon = 0.4$, and report results after 100 epochs with different but constant learning rates in Adam~\cite{kingma2014adam}.
%We compare different scheduling scheme of the adversarial budget.%\tao{these experiments are under the same epoch budget? if it is, maybe point it out?}
Our linear and cosine schedulers perform better than using a constant value of $\epsilon$ during training and yield good performance for a broader range of learning rates: in the small learning rate regime, they speed up training; in the large learning rate regime, they stabilize training and avoid divergence.
Note that, as shown in Appendix~\ref{subsubsec:app_pas}, warmup of the learning rate does not yield similar benefits.
% \tao{it might be somehow unclear how to reach such conclusions; it is a bit strange for 1e-5's poor performance. is it due to the insufficient training?}
%\tao{need evidence here.}.
\begin{wrapfigure}{R}{0.45\textwidth}
	\begin{center}
		\includegraphics[scale = 0.35]{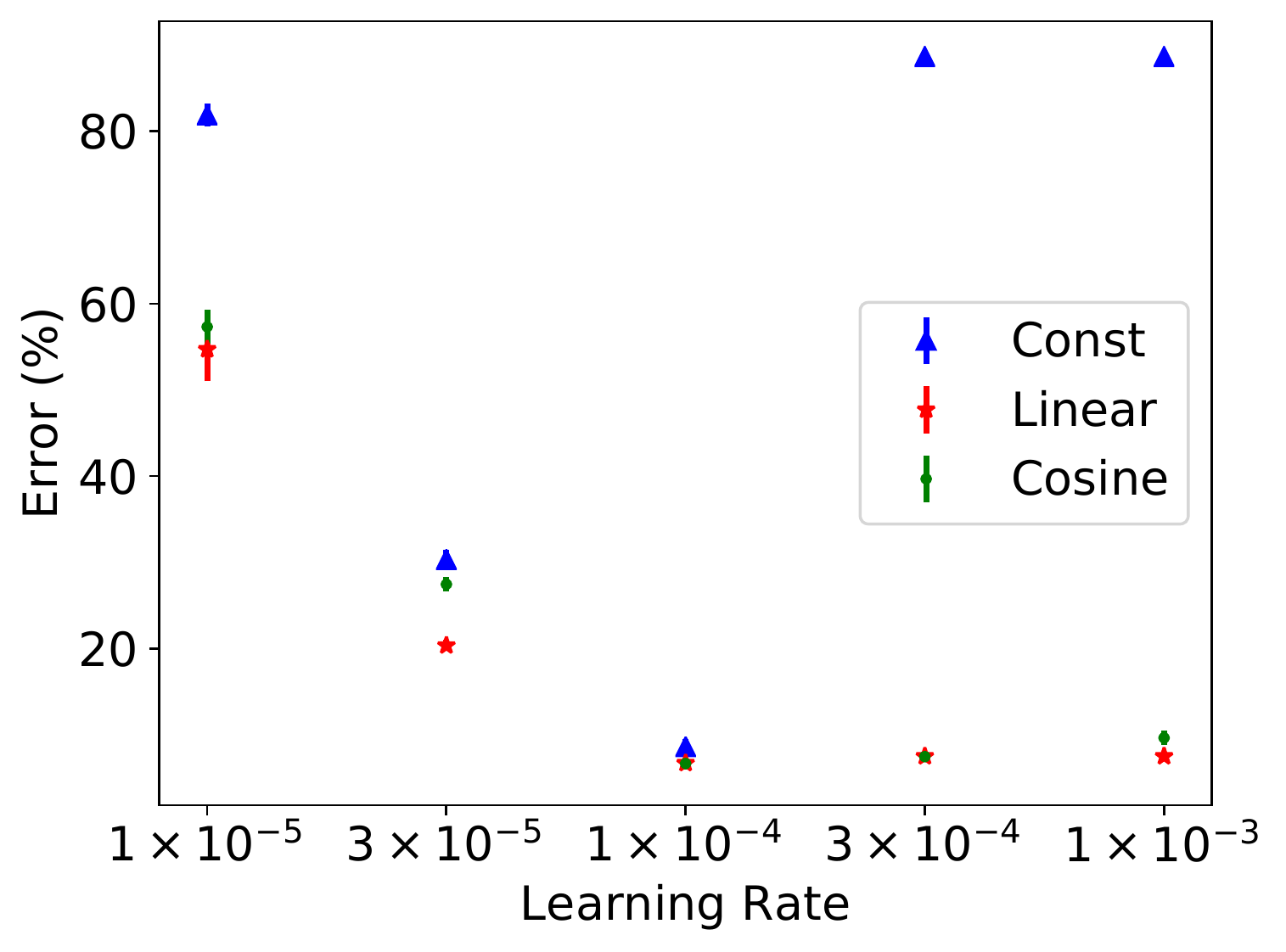}
	\end{center}
	\caption{Mean and standard deviation of the test error under different learning rates with Adam and adversarial budget scheduling.} \label{fig:lr_lenet16}
	\vspace{-0.5cm}
\end{wrapfigure}

As shown in~\cite{huang2017snapshot}, periodic learning rates enable model ensembling to improve the performance.
Here, we can follow the same strategy but also for the adversarial budget.
To this end, we divide the training phase into several periods and store one model at the end of each period.
We make final predictions based on the ensemble of these models.
This periodic scheme has no computational overhead.
We call it periodic adversarial scheduling (PAS). %\MS{This is the ensembling one that we call PAS?}

As before, we run experiments on MNIST and CIFAR10.
For MNIST, we train each model for 100 epochs and do not use a periodic scheduling for the learning rate, which we found not to improve the results even if we use a constant adversarial budget.
For CIFAR10, we train each model for 200 epochs.
When there are no learning rate resets, our results indicate the final model after $200$ epochs.
When using a periodic learning rate, we divide the $200$ epochs into $3$ periods, i.e., we reset the learning rate and the adversarial budget after $100$ and $150$ epochs, and compute the results using an ensemble of these $3$ models.
The value of learning rate and the adversarial budget size are calculated based on the ratio of the current epoch index to the current period length.
We provide more details about hyper-parameter settings in Appendix~\ref{subsec:hyperparam}.

% The results on MNIST and CIFAR10 for different learning rate and adversarial budget scheduling are summarized in Table~\ref{tbl:eps_scheduling_results}.
% For MNIST, we train the model for 100 epochs and do not use a periodic scheduling, which we found not to improve the results.
% For CIFAR10, we train the model for 200 epochs.\tao{I would suggest to clearly define vanilla and periodic settings in words. right now it is somehow confusing.}
% In vanilla settings, our results indicate the final model after $200$ epochs.
% In periodic settings, the learning rate and adversarial budget are reset after $100$ and $150$ epochs, so our results indicate the ensemble of $3$ models.
% In the vanilla setting, we have $T_0 = 0$, $T_1 = 200$, and our results indicate the final model after 200 epochs.
% In the periodic setting, we have $T_0 = 0$, $T_1 = 100$, $T_2 = 150$, $T_3 = 200$, and thus our results are based on the ensemble of $M = 3$ models.
% We reset both the learning rate and the adversarial budget at the end of each period.
% Note that we did not fully explore the hyper-parameters of the functions $\epsilon_{cos}(d)$ and $\epsilon_{lin}(d)$.
% Model architectures and other hyper-parameters are reported in Appendix~\ref{subsec:hyperparam}.

\begin{table}[ht]
\centering
\scriptsize
\begin{tabular}{|p{1.1cm}<{\centering}|p{1.0cm}<{\centering}|p{1.0cm}<{\centering}|p{1.15cm}<{\centering}|p{1.15cm}<{\centering}|p{1.15cm}<{\centering}|p{1.15cm}<{\centering}|p{1.15cm}<{\centering}|p{1.15cm}<{\centering}|}
\hline
\multirow{3}{*}{Task} & \multirow{3}{*}{\makecell{Periodic \\ Learning \\ Rate}} & \multirow{3}{*}{\makecell{$\epsilon$ \\ Scheduler}} & \multirow{3}{*}{\makecell{Clean Error \\ (\%)}} & \multicolumn{5}{c|}{Robust Error (\%)} \\
\cline{5-9}
& & & & \multirow{2}{*}{\makecell{PGD \\ (\%)}} & \multirow{2}{*}{\makecell{PGD100 \\ (\%)}} & \multirow{2}{*}{\makecell{APGD100 \\ CE (\%)}} & \multirow{2}{*}{\makecell{APGD100 \\ DLR (\%)}} & \multirow{2}{*}{\makecell{Square5K \\ (\%)}} \\
& & & & & & & & \\
\hline
\multirow{3}{*}{\makecell{MNIST \\ LeNet \\ $\epsilon = 0.4$}} & \multirow{3}{*}{No}   & Constant  & $1.56 (17)$ & $8.58 (89)$ & $10.86 (143)$ & $15.18 (155)$ & $14.70 (136)$ & $19.58 (45)$ \\
 & & Cosine & $1.08 (2)$  & $6.64 (70)$ & $8.46 (82)$ & $14.36 (134)$ & $13.46 (129)$ & $16.78 (25)$ \\
 & & Linear & $1.06 (6)$  & $6.69 (59)$ & $8.79 (116)$ & $13.91 (150)$ & $13.17 (120)$ & $17.05 (47)$ \\
\hline
\multirow{6}{*}{\makecell{CIFAR10 \\ VGG \\ $\epsilon = 8 / 255$}} & \multirow{3}{*}{No}
   & Constant  & $28.25 (47)$ & $56.22 (43)$ & $56.19 (32)$ & $58.18 (46)$ & $58.65 (69)$ & $54.37 (29)$ \\
 & & Cosine & $25.06 (19)$ & $56.06 (48)$ & $56.00 (42)$ & $57.83 (45)$ & $58.88 (16)$ & $53.95 (15)$ \\
 & & Linear & $23.56 (95)$ & $56.09 (14)$ & $55.88 (5)$ & $57.74 (16)$ & $58.39 (18)$ & $53.66 (24)$ \\
 \cdashline{2-9}
 & \multirow{3}{*}{Yes}
   & Constant  & $28.33 (81)$ & $54.24 (28)$ & $54.16 (26)$ & $55.45 (26)$ & $56.56 (4)$ & $52.85 (18)$ \\
 & & Cosine & $23.91 (21)$ & $53.18 (21)$ & $53.10 (18)$ & $54.44 (16)$ & $55.80 (24)$ & $51.41 (37)$ \\
 & & Linear & $21.88 (33)$ & $53.03 (14)$ & $52.97 (17)$ & $54.32 (17)$ & $55.63 (17)$ & $51.28 (4)$ \\
\hline
\multirow{6}{*}{\makecell{CIFAR10 \\ ResNet18 \\ $\epsilon = 8 / 255$}} & \multirow{3}{*}{No}
   & Constant  & $18.62 (6)$  & $55.00 (8)$  & $54.97 (9)$  & $57.26 (13)$ & $56.60 (25)$ & $50.59 (19)$ \\
 & & Cosine & $18.43 (26)$ & $53.95 (23)$ & $53.85 (21)$ & $56.16 (18)$ & $55.77 (24)$ & $49.60 (18)$ \\
 & & Linear & $18.55 (14)$ & $53.46 (20)$ & $53.41 (10)$ & $55.69 (17)$ & $55.45 (22)$ & $49.66 (28)$ \\
 \cdashline{2-9}
 & \multirow{3}{*}{Yes}
   & Constant  & $21.00 (5)$  & $48.98 (25)$ & $48.87 (25)$ & $50.29 (27)$ & $50.98 (6)$ & $46.84 (9)$ \\
 & & Cosine & $19.90 (18)$ & $48.57 (25)$ & $48.49 (27)$ & $49.71 (22)$ & $50.54 (9)$ & $46.19 (11)$ \\
 & & Linear & $20.26 (28)$ & $48.60 (13)$ & $48.52 (13)$ & $49.73 (9)$  & $50.68 (11)$ & $46.47 (26)$ \\
\hline
\end{tabular}
\vspace{0.5cm}
\caption{Comparison between different adversarial budget schedulers under different adversarial attacks. \textit{Cosine / Linear schedulers} are consistently better than \textit{constant schedulers}. The number between brackets indicate the standard deviation across different runs. Specifically, for example, $1.56 (17)$ stands for $1.56 \pm 0.17$.} \label{tbl:eps_scheduling_results}
\end{table}

We compare different scheduler in adversarial budget under different tasks and settings.
We evaluate the robustness of our trained models by different kinds of attacks.
First we evaluate the models under the PGD attack used in training (PGD), i.e., 50-iteration PGD for MNIST models and 10-iteration PGD for CIFAR10 models.
Then, we increase the number of iterations in PGD and compute the robust error under 100-iteration PGD.
To solve the issue of suboptimal step size, we also evaluate our models using the state-of-the-art AutoPGD attack~\cite{croce2020reliable}, which searches for the optimal step sizes.
We run AutoPGD for 100 iterations for evaluation, based on either cross-entropy loss (APGD100 CE) or the difference of logit ratio loss (APGD100 DLR).
To avoid gradient masking, we also run the state-of-the-art black-box SquareAttack~\cite{andriushchenko2019square} for 5000 iterations (Square5K).
The hyperparameter details are in Appendix~\ref{subsec:hyperparam}.

The results are summarized in Table~\ref{tbl:eps_scheduling_results}, where we compare the clean and robust accuracy under different adversarial attacks on the test set.
It is clear that our proposed cosine or linear schedulers yield better performance, in both clean accuracy and robust accuracy, than using a constant adversarial budget in all cases.
For MNIST, warmup not only makes training robust to different choices of learning rate, but also improves the final robust accuracy.
For CIFAR10, model ensembling enabled by the periodic scheduler improves the robust accuracy.

% For CIFAR10, model ensembling improves the robust accuracy by $3\%$ to $6\%$.
% Note that warmup of the adversarial budget improves the performance of VGG more than ResNet18,\tao{if we look at the robust error of the vanilla setting, then resnet with the warmup in adversarial budget shows larger improvement than vgg-8?}
% which benefits more strongly from a periodic learning rate\tao{I think the experiments results in Table 1 yet is unclear to distinguish the effects of different components; a clear analysis should be included in words}.
% The difference implies that the adversarial loss landscape of ResNet18 is more favorable to optimization than that of VGG.
% In all cases, our adversarial budget scheduling schemes, i.e., Cosine and Linear, achieve the best performance in robust accuracy.
% They also yield better clean accuracy than the baselines\tao{for resnet-18, we have larger clean error for periodic setting than vanilla; also, why we have different clean error for different schedules of the adversarial budgets?}.

% !TEX root = ../main.tex
% !TEX spellcheck = en-US

\section{Discussion}

\paragraph{Model capacity.} In addition to the size of the adversarial budget, the capacity of the model also greatly affects the adversarial loss landscape and thus the performance of adversarial training.
Adversarial training needs higher model capacity in two aspects: if we decrease the model capacity, adversarial training will fail to converge while vanilla training still works~\cite{madry2017towards}; if we increase the model capacity, the robust accuracy of adversarial training continues to rise while the clean accuracy of normal training saturates~\cite{xie2020intriguing}.
Furthermore, we show in Appendix~\ref{subsubsec:capacity_app} that smaller models are more likely to have dead layers because of their lower dimensionality.
As a result, warmup in adversarial budget is also necessary for small models.
% Furthermore, smaller models more easily have dead layers because of their lower dimensionality.
In many cases, the parameter space of small models has good minima in terms of robustness, but adversarial training with a constant value of $\epsilon$ fails to find them.
For example, one can obtain small but robust models by pruning large ones~\cite{gui2019model,ye2019adversarial}.

% //Besides large learning rates, we show in Appendix~\ref{subsubsec:app_pas} warmup is also necessary for small models.\tao{confused; why you mentioned the small models here? LeNet is a small net?}

\paragraph{Architecture.} The network architecture encodes the parameterization of the model, so it greatly affects the adversarial loss landscape.
For example, in Table~\ref{tbl:eps_scheduling_results}, ResNet18 has fewer trainable parameters but better performance than VGG on CIFAR10, indicating  that ResNet18 has a better parameterization in terms of robustness.
Since the optimal architecture for adversarial robustness is not necessarily the same as the one for clean accuracy, we believe that finding architectures inherently favorable to adversarial training is an interesting but challenging topic for future research.

% For neural networks, the parameterization is indicated by the architectures, so architectural design also affects the properties of the adversarial loss landscape.
% For example, PAS improves the performance more in LeNet and VGG than in ResNet.
% Furthermore, the results in Table~\ref{tbl:lenet_activation} of Appendix~\ref{subsubsec:activation} show that the dead layers discussed in Section~\ref{subsec:early_training} can be avoided by replacing ReLU with leaky ReLU~\cite{xu2015empirical}.
% Ultimately, we therefore believe that designing model architectures that are inherently favorable to optimization is an interesting but challenging topic for future research.

\paragraph{Connectivity of minima.} Local minima in the vanilla loss landscape are well-connected~\cite{draxler2018essentially, garipov2018loss}: there exist flat hyper curves connecting them.
In Appendix~\ref{subsec:connect}, we study the connectivity of converged model parameters in the adversarial setting.
We find that the parameters of two adversarially trained models are less connected in the adversarial loss landscape than in the vanilla setting.
That is, the path connecting them needs to go over suboptimal regions.

\paragraph{Adversarial example generation} We approximate the adversarial loss using adversarial examples generated by PGD, which is a good estimate of the inner maximization in (\ref{eq:adv_problem}).
% Despite a random start point, in Figure~\ref{fig:sim_cifar10} when $a = 0$, the cosine similarity close to 1 between adversarial perturbations in different runs indicates PGD has fairly stable outputs.
PGD-based adversarial training updates model parameters by near-optimal adversarial examples.
However, recent works~\cite{shafahi2019adversarial, wong2020fast} have shown that robust models can also be trained by suboptimal adversarial examples, which are faster to obtain.
The formulation of these methods differs from (\ref{eq:adv_problem}), because the inner maximization problem is not approximately solved.
Understanding why models (partially) trained on suboptimal adversarial examples are resistant to stronger adversarial examples needs more investigation.

% !TEX root = ../main.tex
% !TEX spellcheck = en-US

\section{Conclusion}

We have studied the properties of the loss landscape under adversarial training.
We have shown that the adversarial loss landscape is non-smooth and not favorable to optimization, due to the dependency of adversarial examples on the model parameters.
Furthermore, we have empirically evidenced that large adversarial budgets slow down training in the early stages and impedes convergence in the end.
Finally, we have demonstrated the advantages of warmup and periodic scheduling of the adversarial budget size during training.
They make training more robust to different choices of learning rate and yield better performance than vanilla adversarial training.

\section{Broader Impact}

The existence of adversarial examples has raised serious concerns about the deployment of deep learning models in safety-sensitive domains, such as medical imaging~\cite{ma2020understanding} and autonomous navigation~\cite{akhtar2018threat}.
In these domains, as in many others, adversarial training remains the most popular, effective, and general method to train robust models.
By studying the nature of optimization in adversarial training and proposing solutions to overcome the underlying challenges, our work has potential for high societal impact in these fields.
Although the robust accuracy is much lower than the clean accuracy so far, the intrinsic properties of adversarial training we have discovered open up future research directions to improve its performance.
From an ecological perspective, however, we acknowledge that the higher computational cost of adversarial training translates to higher carbon footprint than vanilla training.
Nevertheless, we believe that the potential societal benefits of robustness to attacks outweigh this drawback.

%Besides the effect of adversarial budget during training, our findings in this work gives some hints to improve the performance of adversarially-trained models.
%For example, we find different magnitudes of the gradients in both early and final stage of training, it is natural to investigate whether the current learning rate scheduling, usually the same as the vanilla training, is still optimal in adversarial cases.
%In addition, we observe the sharper minima obtained and poor generalization in adversarially-trained models.
%Therefore, it is promising to explore algorithms which favor flatter minima, such as Entropy-SGD~\cite{chaudhari2019entropy} or noisy gradients~\cite{zhu2018anisotropic}, in this case.
%We believe it will improve the generalization performance of robust models.

\section{Acknowledgements}

We thankfully acknowledge the support of the Hasler Foundation (Grant No. 16076) for this work.

\bibliographystyle{plain}
\bibliography{main}

\newpage

\appendix

% !TEX root = ../main.tex
% !TEX spellcheck = en-US

\section{Theoretical Analysis}

\subsection{Binary Logistic Regression} \label{subsec:logistic}

In this section, we discuss binary logistic regression.
In this case, $K = 2$ and the logit function is $f(\w) = \left[\w^T\x, -\w^T\x\right]$, where $\w \in \R^m$ is the only trainable parameter.
If we use $+1$ and $-1$ to label both classes, then the overall loss function for a dataset $\{(\x_i, y_i)\}_{i = 1}^N$ is $\Loss(\w) = \frac{1}{N}\sum_{i = 1}^N \log(1 + e^{-y_i\w^T\x_i})$.
Under the adversarial budget $\Set^{(p)}_\epsilon(\x)$, the corresponding adversarial loss function is $\Loss_\epsilon(\w) = \frac{1}{N} \sum_{i = 1}^N \log(1 + e^{-y_i\w^T\x_i + \epsilon\|\w\|_q})$, where $l_q$ is the dual norm of $l_p$.
Since the magnitude of $\w$ does not change the results of the classifier, we can assume $\|\w\|_q = 1$ without loss of generality.
As a result, the adversarial loss function is
\begin{equation}
\begin{aligned}
\Loss_\epsilon(\w) = \frac{1}{N} \sum_{i = 1}^N \log(1 + e^{-y_i\w^T\x_i + \epsilon})\;.
\end{aligned} \label{eq:adv_loss_logistic}
\end{equation}

The following theorem describes the properties of $\Loss_\epsilon(\w)$ for different values of $\epsilon$.

\begin{theory} \label{thm:logistic}
If the dataset $\{(\x_i, y_i)\}_{i = 1}^N$ is linearly separable under the adversarial budget $\Set_{\widehat{\epsilon}}(\x)$, then for any unit vector $\m \in \R^m$ and values $\epsilon_1$, $\epsilon_2$ such that $\epsilon_1 \leq \epsilon_2 \leq \widehat{\epsilon}$, we have $\m^T\triangledown^2_\w \Loss_{\epsilon_1}(\w)\m \leq \m^T\triangledown^2_\w \Loss_{\epsilon_2}(\w)\m$. More specifically, both the largest and smallest eigenvalue of $\triangledown^2_\w \Loss_{\epsilon_1}(\w)$ are no greater than those of $\triangledown^2_\w \Loss_{\epsilon_2}(\w)$.
\end{theory}

We provide the proof of Theorem~\ref{thm:logistic} in Appendix~\ref{sub:app_proof_logistic}.
% \MS{It is a bit strange to separate the proof from the theorem, since now both are in the supplementary material.}
Since $\m^T\triangledown^2_\w \Loss_\epsilon(\w)\m$ is the curvature of $\Loss_\epsilon(\w)$ in the direction of $\m$, Theorem~\ref{thm:logistic} shows that the curvature of $\Loss_\epsilon(\w)$ increases with $\epsilon$ in any direction if the whole dataset is linearly separable.
For an individual data point $\x$, if $\forall \x' \in \Set_{\widehat{\epsilon}}(\x)$, $\x'$ is correctly classified, then the curvature of $g_\epsilon(\x, \w)$ also increases with $\epsilon$ in any direction as long as $\epsilon \leq \widehat{\epsilon}$.
The assumption for an individual point here is much weaker than the one in Theorem~\ref{thm:logistic}.
If the overwhelming majority of the data points are correctly classified under the adversarial budget $\Set_{\widehat{\epsilon}}$, the conclusion still holds in practice.

\subsection{Discussions of ReLU Networks} \label{subsec:relu}

Unlike sigmoid or tanh, ReLU is not a smooth function.
However, it is smooth \textit{almost everywhere}, except at $0$.
As a result, we can make the following assumptions for the function $g$ represented by a ReLU network.

\begin{assumption} \label{assume:relu_lip}
The function $g$ satisfies the following conditions:
\begin{equation}
\begin{aligned}
\|g(\x, \theta_1) - g(\x, \theta_2)\| &\leq L_\theta \|\theta_1 - \theta_2\|\;, \\
\|\triangledown_{\theta} g(\x, \theta_1) - \triangledown_{\theta} g(\x, \theta_2)\| &\leq L_{\theta\theta} \|\theta_1 - \theta_2\| + D_{\theta\theta}\;, \\
\|\triangledown_{\theta} g(\x_1, \theta) - \triangledown_{\theta} g(\x_2, \theta)\| &\leq L_{\theta\x} \|\x_1 - \x_2\|_p + D_{\theta\x}\;.
\end{aligned}
\end{equation}
\end{assumption}

We adjust the second-order smoothness assumption by adding two constants $D_{\theta\theta}$, $D_{\theta\x}$.
They are the upper bound of the gradient difference in the neighborhood of non-smooth points.
Therefore, they measures how abruptly the (sub)gradients can change in a sufficiently small region in the parameter space and can be considered as a quantitative measure of \textit{gradient scattering}.

The following corollary states the properties of $g_\epsilon$ under Assumption~\ref{assume:relu_lip}.

\begin{corollary} \label{coro:lip}
If Assumption~\ref{assume:relu_lip} is satisfied, then we have

\begin{equation}
\begin{aligned}
\|\Loss_\epsilon(\theta_1) - \Loss_\epsilon(\theta_2)\| &\leq L_\theta \|\theta_1 - \theta_2\| \\
\|\triangledown_\theta \Loss_\epsilon(\theta_1) - \triangledown_\theta \Loss_\epsilon(\theta_2)\| &\leq L_{\theta\theta} \|\theta_1 - \theta_2\| + 2\epsilon L_{\theta\x} + D_{\theta\theta} + D_{\theta\x}\;.
\end{aligned} \label{eq:gen_lip}
\end{equation}
\end{corollary}

The proof directly follows the one of Proposition~\ref{prop:lip}.
As in Proposition~\ref{prop:lip}, the additional $2\epsilon L_{\theta\x}$ term in Corollary~\ref{coro:lip} evidences more severe \textit{gradient scattering} under adversarial training in the context of ReLU networks, which harms optimization.

Similarly, we can easily extend the study of the asymptotic gradient magnitude of Theorem~\ref{thm:convergence} to the account for Assumption~\ref{assume:relu_lip}.

\begin{corollary} \label{coro:relu_convergence}
Let Assumption~\ref{assume:relu_lip} hold, the stochastic gradient $\triangledown_\theta \widehat{\Loss}_\epsilon(\theta_t)$ be unbiased and have bounded variance, and the SGD update $\theta_{t + 1} = \theta_t - \alpha_t \triangledown_\theta \widehat{\Loss}_\epsilon(\theta_t)$ use a constant step size $\alpha_t = \alpha = \frac{1}{L_{\theta\theta}\sqrt{T}}$ for $T$ iterations. Given the trajectory of the parameters during optimization $\{\theta_t\}_{t = 1}^T$, then we can bound the asymptotic probability of large gradients for a sufficiently large $T$ as
% If Assumption~\ref{assume:relu_lip} holds, the stochastic gradient $\triangledown_\theta \widehat{\Loss}_\epsilon(\theta_t)$ is unbiased and have bounded variance, we run SGD update $\theta_{t + 1} = \theta_t - \alpha \triangledown_\theta \widehat{\Loss}_\epsilon(\theta_t)$ using a constant step size $\alpha = \frac{1}{L_{\theta\theta}\sqrt{T}}$ for $T$ iterations;
% let $\{\theta_t\}_{t = 1}^T$ be the trajectory of parameters during optimizationm, then we can bound the asymptotical probability of large gradients given a sufficient large value of $T$:
\begin{equation}
\begin{aligned}
\forall \gamma \geq 2, P(\|\triangledown_\theta \Loss_\epsilon(\theta_t)\| > \gamma (\epsilon L_{\theta\x} + \frac{1}{2} D_{\theta\theta} + \frac{1}{2} D_{\theta\x})) < \frac{4}{\gamma^2 - 2\gamma + 4}\;.
\end{aligned} \label{eq:thm_relu_convergence}
\end{equation}
\end{corollary}

% !TEX root = ../main.tex
% !TEX spellcheck = en-US

\section{Proofs}

\subsection{Proof of Proposition~\ref{prop:version_space}} \label{sub:app_proof_v_set}

\begin{proof}

For arbitrary $\W \in \VV_{\epsilon_2}$, we have $\forall i \in [K], \x' \in \Set_{\epsilon_2}(\x)$, $(\w_i - \w_y)\x' \leq 0$ based on the definition of $\VV_\epsilon$.

Since $\epsilon_1 \leq \epsilon_2$, we have $\Set_{\epsilon_1}(\x) \subseteq \Set_{\epsilon_2}(\x)$.
As a result, $\forall i \in [K], \x' \in \Set_{\epsilon_1}(\x)$, $(\w_i - \w_y)\x' \leq 0$.
That is to say, $\W \in \VV_{\epsilon_1}$.
$\W$ is arbitrarily picked, so $\VV_{\epsilon_2} \subseteq \VV_{\epsilon_1}$.

\end{proof}

\subsection{Proof of Theorem~\ref{thm:t_set}} \label{sub:app_proof_t_set}

\begin{proof}

In multi-class logistic regression, as discussed in Section~\ref{subsec:softmax}, the function $g(\x, \W) = \log(1 + \sum_{j \neq y}\exp^{(\w_j - \w_y)\x})$ is a convex function w.r.t. the parameters $\W$, and so is $g_\epsilon(\x, \W)$.
Based on convexity, for any $\W \in \TT_\epsilon$, the statement $0 \in \argmin_\gamma g_\epsilon(\x, \gamma\W)$ is equivalent to the following statement:

\begin{equation}
\begin{aligned}
\forall \Delta\gamma > 0, g_\epsilon(\x, \Delta\gamma\W) \geq g_\epsilon(\x, \zero)\ \mathrm{and}\ g_\epsilon(\x, -\Delta\gamma\W) \geq g_\epsilon(\x, \zero)\;.
\end{aligned} \label{eq:convex_equiv}
\end{equation}

Note that $g_\epsilon(\x, \zero) \equiv \log K$, which means that the loss of the model is independent of both the input and the adversarial budget when $\W = \zero$.
Given $\epsilon_1 \geq \epsilon_2$, we have, $\forall \x, \W, g_{\epsilon_1}(\x, \W) \geq g_{\epsilon_2}(\x, \W)$.
Therefore, for an arbitrary $\W \in \TT_{\epsilon_2}$, we have the following inequality:

\begin{equation}
\begin{aligned}
\forall \Delta\gamma > 0, g_{\epsilon_1}(\x, \Delta\gamma\W) \geq g_{\epsilon_2}(\x, \Delta\gamma\W) \geq g_{\epsilon_2}(\x, \zero) = g_{\epsilon_1}(\x, \zero)\;.
\end{aligned}
\end{equation}

The first inequality is based on $\epsilon_1 \geq \epsilon_2$, the second one is based on (\ref{eq:convex_equiv}) and the last one arises from the fact that, $\forall \epsilon, g_\epsilon(\x, \zero)$ is a constant.
Similarly, we also have $g_{\epsilon_1}(\x, -\Delta\gamma\W) \geq g_{\epsilon_1}(\x, \zero)$.
Therefore, we have $\W \in \TT_{\epsilon_1}$, which means $\TT_{\epsilon_2} \subseteq \TT_{\epsilon_1}$.

To prove the second half of Theorem~\ref{thm:t_set}, one barrier is that we do not have an analytical form for $g_\epsilon(\x, \W)$.
Instead, we introduce a lower bound $\underline{g_\epsilon}(\x, \W)$  of $g_\epsilon(\x, \W)$, which has an analytical form.
We consider the perturbation $\x' = \x + \epsilon \frac{(\w_m - \w_y)^{\frac{q}{p}}}{\|\w_m - \w_y\|^{\frac{q}{p}}_q}$ \footnote{$l_q$ is the dual norm of $l_p$, i.e., $\frac{1}{p} + \frac{1}{q} = 1$}, where $m = \argmax_j \|\w_j - \w_y\|_q$.
It can be verified that $\x' \in \Set_\epsilon^{(p)}(\x)$.
Therefore, we set $\underline{g_\epsilon}(\x, \W) = g(\x', \W)$, which is a valid lower bound of $g_\epsilon(\x, \W)$.
Then, the analytical expression of $\underline{g_\epsilon}(\x, \W)$ can be written as

\begin{equation}
\begin{aligned}
\underline{g_\epsilon}(\x, \W) = \log(1 + \exp^{(\w_m - \w_y)\x + \epsilon\|\w_m - \w_y\|_q} + \sum_{j \neq y, j \neq m} \exp^{(\w_j - \w_y)\x + \epsilon (\w_j - \w_y) \frac{(\w_m - \w_y)^{\frac{q}{p}}}{\|\w_m - \w_y\|^{\frac{q}{p}}_q}})\;.
\end{aligned} \label{eq:g_underline}
\end{equation}

Since $m = \argmax_j \|\w_j - \w_y\|_q$, then $(\w_j - \w_y) \frac{(\w_m - \w_y)^{\frac{q}{p}}}{\|\w_m - \w_y\|^{\frac{q}{p}}_q} \leq \|\w_m - \w_y\|_q$.
As a result, if $\epsilon$ is large enough, the second term inside the logarithm of (\ref{eq:g_underline}) will dominate the summation and $\lim_{\epsilon \to \infty} \underline{g_\epsilon}(\x, \W) = \infty$.
More specifically, we can find $\bar{\epsilon} = \frac{\log(K - 1) - (\w_m - \w_y)\x}{\|\w_m - \w_y\|_q}$, such that, $\forall \epsilon > \bar{\epsilon}$, $\W$, then $\underline{g_\epsilon}(\x, \W) \geq \log K = g_\epsilon(\x, \zero)$.

Now, $\forall \W \in \R^{m \times K}, \Delta\gamma > 0, \epsilon \geq \bar{\epsilon}$, we have $g_\epsilon(\x, \Delta\gamma \W) \geq \underline{g}_\epsilon(\x, \Delta\gamma \W) \geq g_\epsilon(\x, \zero)$.
Similarly, we have $g_\epsilon(\x, -\Delta\gamma \W) \geq g_\epsilon(\x, \zero)$.
As a result, we have, $\forall \W \in \R^{m \times K}, g_\epsilon(\x, \W) \geq g_\epsilon(\x, \zero)$, so $\zero \in \argmin_{\W} g_\epsilon(\x, \W)$.
Based on (\ref{eq:convex_equiv}), we have $\TT_\epsilon = \R^{m \times K}$.

\end{proof}

\subsection{Proof of Proposition~\ref{prop:lip}} \label{sub:app_proof_lip}

\begin{proof}

%Consider $\Loss_\epsilon(\theta)$ is the average of $g_\epsilon(\x, \theta)$ over the dataset.
Recall that $\Loss_\epsilon(\theta)$ is the average of $g_\epsilon(\x, \theta)$ over the dataset.
Therefore, to prove Proposition~\ref{prop:lip}, we only need to prove the following inequalities for any data point $\x$:

\begin{equation}
\begin{aligned}
\|g_\epsilon(\x, \theta_1) - g_\epsilon(\x, \theta_2)\| & \leq L_\theta \|\theta_1 - \theta_2\| \;, \\
\|\triangledown_\theta g_\epsilon(\x, \theta_1) - \triangledown_\theta g_\epsilon(\x, \theta_2)\| & \leq L_{\theta\theta} \|\theta_1 - \theta_2\| + 2\epsilon L_{\theta\x}\;.
\end{aligned} \label{eq:g_lip}
\end{equation}

To prove the first inequality, we introduce the adversarial examples for parameter $\theta_1$ and $\theta_2$:
\begin{equation}
\begin{aligned}
\x_1 &= \argmax_{\x' \in \Set^{(p)}_\epsilon(\x)} g(\x', \theta_1)\;, \\
\x_2 &= \argmax_{\x' \in \Set^{(p)}_\epsilon(\x)} g(\x', \theta_2)\;.
\end{aligned}
\end{equation}
Therefore, $g_\epsilon(\x, \theta_1) = g(\x_1, \theta_1)$ and $g_\epsilon(\x, \theta_2) = g(\x_2, \theta_2)$.

By definition,  we have $g(\x_1, \theta_1) \geq g(\x_2, \theta_1)$ and $g(\x_2, \theta_2) \geq g(\x_1, \theta_2)$.
As a result, $\|g_\epsilon(\x, \theta_1) - g_\epsilon(\x, \theta_2)\| = \|g(\x_1, \theta_1) - g(\x_2, \theta_2)\|$.
If $g(\x_1, \theta_1) - g(\x_2, \theta_2) \leq 0$, we have
\begin{equation}
\begin{aligned}
\|g_\epsilon(\x, \theta_1) - g_\epsilon(\x, \theta_2)\| = g(\x_2, \theta_2) - g(\x_1, \theta_1) \leq g(\x_2, \theta_2) - g(\x_2, \theta_1) \leq L_{\theta}\|\theta_1 - \theta_2\|\;.
\end{aligned}
\end{equation}
Similarly, if $g(\x_1, \theta_1) - g(\x_2, \theta_2) \geq 0$, we have
\begin{equation}
\begin{aligned}
\|g_\epsilon(\x, \theta_1) - g_\epsilon(\x, \theta_2)\| = g(\x_1, \theta_1) - g(\x_2, \theta_2) \leq g(\x_1, \theta_1) - g(\x_1, \theta_2) \leq L_{\theta}\|\theta_1 - \theta_2\|\;.
\end{aligned}
\end{equation}
This proves the first inequality in (\ref{eq:g_lip}).
The bound is tight, and equality is achieved when, for example, $\x_1 = \x_2$.

The second inequality in (\ref{eq:g_lip}) is more straightforward. We have
\begin{equation}
\begin{aligned}
\|\triangledown_\theta g_\epsilon(\x, \theta_1) - \triangledown_\theta g_\epsilon(\x, \theta_2)\|
&= \|\triangledown_\theta g(\x_1, \theta_1) - \triangledown_\theta g(\x_2, \theta_2)\| \\
&= \|\triangledown_\theta g(\x_1, \theta_1) - \triangledown_\theta g(\x_1, \theta_2) + \triangledown_\theta g(\x_1, \theta_2) - \triangledown_\theta g(\x_2, \theta_2)\| \\
&\leq \|\triangledown_\theta g(\x_1, \theta_1) - \triangledown_\theta g(\x_1, \theta_2)\| + \|\triangledown_\theta g(\x_1, \theta_2) - \triangledown_\theta g(\x_2, \theta_2)\| \\
&\leq L_{\theta\theta}\|\theta_1 - \theta_2\| + L_{\theta\x}\|\x_1 - \x_2\|_p \\
&\leq L_{\theta\theta}\|\theta_1 - \theta_2\| + 2\epsilon L_{\theta\x}\;.
\end{aligned} \label{eq:prove_lip_4}
\end{equation}
The last inequality in (\ref{eq:prove_lip_4}) is satisfied because both $\x_1$ and $\x_2$ belong to $\Set^{(p)}_{\epsilon}(\x)$.
This bound is tight, and equality is reached only when $\|\x_1 - \x_2\|_p = 2\epsilon$.

\end{proof}

\subsection{Proof of Theorem~\ref{thm:convergence}} \label{subsec:convergence_proof}

\begin{proof}

Let $\sigma^2$ to denote the variance of stochastic gradient $\triangledown_{\theta} \widehat{\Loss}_\epsilon(\theta)$. Based on the assumption that $\triangledown_{\theta} \widehat{\Loss}_\epsilon(\theta)$ is unbiased, we have
\begin{equation}
\begin{aligned}
\E [\triangledown_{\theta} \widehat{\Loss}_\epsilon(\theta)] &= \triangledown_{\theta} \Loss_\epsilon(\theta)\;, \\
\E \|\triangledown_{\theta} \widehat{\Loss}_\epsilon(\theta)\|^2 &= \|\triangledown_{\theta} \Loss_\epsilon(\theta)\|^2 + \sigma^2\;.
\end{aligned} \label{eq:expectation}
\end{equation}

Proposition~\ref{prop:lip} shows that $\Loss_\epsilon(\theta)$ is continuous. Therefore, we introduce $\tilde{\theta_t}(u) = \theta_t + u(\theta_{t + 1} - \theta_t)$ and derive an upper bound of $\Loss_\epsilon(\theta_{t + 1}) - \Loss_\epsilon(\theta_t)$ by first order Taylor expansion and using the update rule $\theta_{t + 1} = \theta_t - \alpha_t \triangledown_{\theta} \widehat{\Loss}_\epsilon(\theta_t)$.
This yields
\begin{equation}
\begin{aligned}
\Loss_\epsilon(\theta_{t + 1}) - \Loss_\epsilon(\theta_t) &= \int_0^1 \langle \theta_{t + 1} - \theta_t, \triangledown_\theta \Loss_\epsilon(\tilde{\theta_t}(u)) \rangle d_u\\
&= \int_0^1 \langle -\alpha_t \triangledown_\theta \widehat{\Loss}_\epsilon(\theta_t), \triangledown_\theta \Loss_\epsilon(\tilde{\theta_t}(u)) \rangle d_u \\
&= \int_0^1 \langle -\alpha_t \triangledown_\theta \widehat{\Loss}_\epsilon(\theta_t), \triangledown_\theta \Loss_\epsilon(\tilde{\theta_t}(u)) - \triangledown_\theta \Loss_\epsilon(\theta_t) \rangle d_u + \langle -\alpha_t \triangledown_\theta \widehat{\Loss}_\epsilon(\theta_t), \triangledown_\theta \Loss_\epsilon(\theta_t) \rangle \\
&\leq \int_0^1 \alpha_t \|\triangledown_\theta \widehat{\Loss}_\epsilon(\theta_t)\|\|\triangledown_\theta \Loss_\epsilon(\tilde{\theta_t}(u)) - \triangledown_\theta \Loss_\epsilon(\theta_t)\| d_u - \alpha_t \langle \triangledown_\theta \widehat{\Loss}_\epsilon(\theta_t), \triangledown_\theta \Loss_\epsilon(\theta_t) \rangle \\
&\leq \int_0^1 \alpha_t \|\triangledown_\theta \widehat{\Loss}_\epsilon(\theta_t)\| (L_{\theta\theta}\|\tilde{\theta_t}(u) - \theta_t\| + 2\epsilon L_{\theta\x}) d_u - \alpha_t \langle \triangledown_\theta \widehat{\Loss}_\epsilon(\theta_t), \triangledown_\theta \Loss_\epsilon(\theta_t) \rangle \\
&= \frac{1}{2} \alpha_t^2 L_{\theta\theta} \|\triangledown_\theta \widehat{\Loss}_\epsilon(\theta_t)\|^2 + 2\epsilon L_{\epsilon\x} \alpha_t \|\triangledown_\theta \widehat{\Loss}_\epsilon(\theta_t)\| - \alpha_t \langle \triangledown_\theta \widehat{\Loss}_\epsilon(\theta_t), \triangledown_\theta \Loss_\epsilon(\theta_t) \rangle\;.
\end{aligned}
\end{equation}
Here, the first inequality comes from H\"older's Inequality; the second one follows the conclusion of Proposition~\ref{prop:lip}.

By taking the expectation over the noise introduced by SGD, we have
\begin{equation}
\begin{aligned}
\E [\Loss_\epsilon(\theta_{t+1})] - \E [\Loss_\epsilon(\theta_t)] &\leq \frac{1}{2} \alpha_t^2 L_{\theta\theta} (\|\triangledown_\theta \Loss_\epsilon(\theta_t)\|^2 + \sigma^2) + 2\epsilon L_{\theta\x} \alpha_t \|\triangledown_\theta \Loss_\epsilon(\theta_t)\| - \alpha_t \|\triangledown_\theta \Loss_\epsilon(\theta_t)\|^2 \\
&= (\frac{1}{2} \alpha_t^2 L_{\theta\theta} - \alpha_t) \|\triangledown_\theta \Loss_\epsilon(\theta_t)\|^2 + 2\epsilon L_{\theta\x} \alpha_t \|\triangledown_\theta \Loss_\epsilon(\theta_t)\| + \frac{1}{2} \alpha_t^2 \sigma^2 L_{\theta\theta} \\
&\leq -\frac{1}{2} \alpha_t \|\triangledown_\theta \Loss_\epsilon(\theta_t)\|^2 + 2\epsilon L_{\theta\x} \alpha_t \|\triangledown_\theta \Loss_\epsilon(\theta_t)\| + \frac{1}{2} \alpha_t^2 \sigma^2 L_{\theta\theta}\;.
\end{aligned} \label{eq:single_iter}
\end{equation}
We use the approximation $\E \|\triangledown_\theta \widehat{\Loss}_\epsilon(\theta)\| \simeq \|\triangledown_\theta \Loss_\epsilon(\theta)\|$ because the variance arises mainly from the term $\|\triangledown_\theta \widehat{\Loss}_\epsilon(\theta_t)\|^2$.
The last inequality is based on the fact that $\alpha_t = \alpha = \frac{1}{L_{\theta\theta}\sqrt{T}}$, so $\alpha_t L_{\theta\theta} = \frac{1}{\sqrt{T}} \leq 1$.

Let us now sum (\ref{eq:single_iter}) over $t \in [T]$. This gives
\begin{equation}
\begin{aligned}
\sum_{t = 0}^T \left[ \frac{1}{2}\alpha_t \|\triangledown_\theta \Loss_\epsilon(\theta_t)\|^2 - 2\epsilon L_{\theta\x} \alpha_t \|\triangledown_\theta \Loss_\epsilon(\theta_t)\| \right] &\leq \Loss_\epsilon(\theta_0) - \E[\Loss_\epsilon(\theta_T)] + \frac{T}{2} \alpha_t^2 \sigma^2 L_{\theta\theta} \\
&\leq \Loss_\epsilon(\theta_0) - \Loss_\epsilon(\theta^*) + \frac{T}{2} \alpha_t^2 \sigma^2 L_{\theta\theta}\;.
\end{aligned}
\end{equation}
We use $\theta^*$ to denote the global minimum since $\Loss_\epsilon(\theta)$ is lower bounded.
By introducing $\alpha_t = \alpha = \frac{1}{L_{\theta\theta}\sqrt{T}}$ into the formulation, we obtain
\begin{equation}
\begin{aligned}
\frac{1}{T} \sum_{t = 0}^T \left[ \frac{1}{2} \|\triangledown_\theta \Loss_\epsilon(\theta_t)\|^2 - 2\epsilon L_{\theta\x} \|\triangledown_\theta \Loss_\epsilon(\theta_t)\| \right] &\leq \frac{1}{\alpha T} \left[ \Loss_\epsilon(\theta_0) - \Loss_\epsilon(\theta^*) \right] + \frac{1}{2} \alpha \sigma^2 L_{\theta\theta} \\
&= \frac{1}{\sqrt{T}} \left[ L_{\theta\theta}(\Loss_\epsilon(\theta_0) - \Loss_\epsilon(\theta^*)) + \frac{1}{2} \sigma^2 \right]\;.
\end{aligned} \label{eq:average}
\end{equation}

Since the righthand side of (\ref{eq:average}) converges to $0$ as $T \to +\infty$, we have
\begin{equation}
\begin{aligned}
\lim_{T \to +\infty} \frac{1}{T} \sum_{t = 0}^T \left[ \frac{1}{2} \|\triangledown_\theta \Loss_\epsilon(\theta_t)\|^2 - 2\epsilon L_{\theta\x} \|\triangledown_\theta \Loss_\epsilon(\theta_t)\| \right] \leq 0\;.
\end{aligned} \label{eq:limit}
\end{equation}

Let us define $h(\theta_t) = \frac{1}{2} \|\triangledown_\theta \Loss_\epsilon(\theta_t)\|^2 - 2\epsilon L_{\theta\x} \|\triangledown_\theta \Loss_\epsilon(\theta_t)\|$ for notation simplicity.
Then, inequality (\ref{eq:limit}) shows that $\E_t [h(\theta_t)] \leq 0$ when $T$ is large enough.
% $h(\theta_t)$ is non-positive when $\|\triangledown_\theta \Loss_\epsilon(\theta_t)\| \in [0, 4\epsilon L_{\theta\x}]$, with the minimum value of $-2 \epsilon^2 L_{\theta\x}^2$ when $\|\triangledown_\theta \Loss_\epsilon(\theta_t)\| = 2\epsilon L_{\theta\x}$, it is positive when $\|\triangledown_\theta \Loss_\epsilon(\theta_t)\| > 4\epsilon L_{\theta\x}$.

Let $\|\triangledown_\theta \Loss_\epsilon(\theta_t)\| = \gamma \epsilon L_{\theta\x}$, then we have $h(\theta_{t}) = (\frac{1}{2} \gamma^2 - 2 \gamma) \epsilon^2 L_{\theta\x}^2$.
$h(\theta_t)$ is monotonically increasing when $\theta_t \geq 2 \epsilon L_{\theta\x}$, so when $\gamma \geq 2$, $h(\theta_{t}) \geq (\frac{1}{2} \gamma^2 - 2 \gamma) \epsilon^2 L_{\theta\x}^2$.
Considering $h(\theta_t) \geq -2 \epsilon^2 L_{\theta\x}^2$, we then have
\begin{equation}
\begin{aligned}
\E_t [h(\theta_t)] > -2 \epsilon^2 L_{\theta\x}^2 (1 - P(\|\triangledown_\theta \Loss_\epsilon(\theta_t)\| > \gamma \epsilon L_{\theta\x})) + (\frac{1}{2} \gamma^2 - 2 \gamma) \epsilon^2 L_{\theta\x}^2 P(\|\triangledown_\theta \Loss_\epsilon(\theta_t)\| > \gamma \epsilon L_{\theta\x})\;.
\end{aligned} \label{eq:e_bound}
\end{equation}

Finally, by rearranging (\ref{eq:e_bound}) and using $\E_t [h(\theta_t)] \leq 0$, we obtain
\begin{equation}
\begin{aligned}
\forall \gamma > 2,\ P(\|\triangledown_\theta \Loss_\epsilon(\theta_t)\| > \gamma \epsilon L_{\theta\x}) < \frac{4}{\gamma^2 - 2\gamma + 4}\;.
\end{aligned}
\end{equation}

\end{proof}

\subsection{Proof of Theorem~\ref{thm:logistic}} \label{sub:app_proof_logistic}

To prove Theorem~\ref{thm:logistic}, let us first introduce the following lemma.

\begin{lemma} \label{lemma:eigenvalue}
Given a vector set $\{\x_i\}_{i = 1}^N$ and scalar sets $\{a_i\}_{i = 1}^N$, $\{b_i\}_{i = 1}^N$, we define $\A = \sum_{i = 1}^{N} a_i \x_i \x^T_i$ and $\B = \sum_{i = 1}^{N} b_i \x_i \x^T_i$.
If, $\forall i,~ a_i \geq b_i$, then $\forall \m \in \R^m$, $\m^T\A\m \geq \m^T\B\m$. Furthermore, the largest and the smallest eigenvalues of $\A$ are no smaller than those of $\B$.
\end{lemma}

\begin{proof}
Because $\forall\ i, a_i \geq b_i$, we have $\forall\m \sum_{i = 1}^N (a_i - b_i)(\x^T_i \m)^2 \geq 0$, which can be re-organized into $\m^T\A\m \geq \m^T\B\m$.

$\A$ is a symmetric matrix, so the largest eigenvalue $\lambda_1(\A)$ is $\max_{\|\m\|_2 = 1} \m^T\A\m = \max_{\|\m\|_2 = 1} \sum_{i = 1}^N a_i (\x_i^T \m)^2$.
Similarly, we have $\lambda_1(\B) = \max_{\|\m\|_2 = 1} \sum_{i = 1}^N b_i (\x_i^T \m)^2$.
Let $\m_\B \in \argmax_{\|\m\|_2 = 1} \sum_{i = 1}^N b_i (\x_i^T \m)^2$. Then we have
\begin{equation}
\begin{aligned}
\lambda_1(\B) = \sum_{i = 1}^N b_i (\x_i^T \m_\B)^2 \leq \sum_{i = 1}^N a_i (\x_i^T \m_\B)^2 \leq \max_{\|\m\|_2 = 1} \sum_{i = 1}^N a_i (\x_i^T \m)^2 = \lambda_1(\A)\;.
\end{aligned}
\end{equation}

In the same way as for the largest eigenvalue, the smallest eigenvalue of $\A$ and $\B$ are $\lambda_m(\A) = \min_{\|\m\|_2 = 1} \sum_{i = 1}^N a_i (\x_i^T \m)^2$ and $\lambda_m(\B) = \min_{\|\m\|_2 = 1} \sum_{i = 1}^N b_i (\x_i^T \m)^2$, respectively.
Let $\m_\A \in \argmin_{\|\m\|_2 = 1} \sum_{i = 1}^N a_i (\x_i^T \m)^2$. Then we have
\begin{equation}
\begin{aligned}
\lambda_m(\A) = \sum_{i = 1}^N a_i (\x_i^T \m_\A)^2 \geq \sum_{i = 1}^N b_i (\x_i^T \m_\A)^2 \geq \min_{\|\m\|_2 = 1} \sum_{i = 1}^N b_i (\x_i^T \m)^2 = \lambda_m(\B)\;.
\end{aligned}
\end{equation}

\end{proof}

Let us now go back to Theorem~\ref{thm:logistic}.

\begin{proof}
We first calculate the first and second derivatives of $\Loss_\epsilon(\w)$ in Equation~\ref{eq:adv_loss_logistic}, as

\begin{equation}
\begin{aligned}
\triangledown_{\w} \Loss_\epsilon(\w) &= \frac{1}{N} \sum_{i = 1}^N - \frac{1}{1 + e^{y_i\w^T\x_i - \epsilon}} y_i\x_i \;,\\
\triangledown^2_{\w} \Loss_\epsilon(\w) &= \frac{1}{N} \sum_{i = 1}^N \frac{e^{y_i\w^T\x_i - \epsilon}}{(1 + e^{y_i\w^T\x_i - \epsilon})^2} y_i^2 \x_i\x_i^T = \frac{1}{N} \sum_{i = 1}^N \frac{e^{y_i\w^T\x_i - \epsilon}}{(1 + e^{y_i\w^T\x_i - \epsilon})^2} \x_i\x_i^T\;.
\end{aligned}
\end{equation}

The second equality of $\triangledown^2_{\w} \Loss_\epsilon(\w)$ is satisfied because $y_i$ is either $+1$ or $-1$.
The dataset $\{(\x_i, y_i)\}_{i = 1}^N$ is linearly separable under adversarial budget $\Set^{(p)}_{\hat{\epsilon}}(\x)$, so, $\forall i,~y_i\w^T\x_i \geq \hat{\epsilon}$.
When $\epsilon \leq \hat{\epsilon}$, $e^{y_i\w^T\x_i - \epsilon} > 1$ and monotonically decreases with $\epsilon$.
As a result, $\frac{e^{y_i\w^T\x_i - \epsilon}}{(1 + e^{y_i\w^T\x_i - \epsilon})^2}$ monotonically increases with $\epsilon$ in the range $[0, \hat{\epsilon}]$.

Based on Lemma~\ref{lemma:eigenvalue}, $\forall \m \in \R^m$, $\m^T\triangledown^2_{\w} \Loss_\epsilon(\w)\m$ increases with $\epsilon$, and so do the largest and the smallest eigenvalues of the Hessian matrix $\triangledown^2_{\w} \Loss_\epsilon(\w)$.

\end{proof}

% !TEX root = ../main.tex
% !TEX spellcheck = en-US

\section{Additional Experiments} \label{sec:app_experiment}

\subsection{Experimental Details} \label{subsec:hyperparam}

For MNIST, we set the step size of PGD to $0.01$ and the number of iterations to $\epsilon / 0.01 + 10$.
For CIFAR10, we set the number of PGD iterations to 10 and the step size to $\epsilon / 4$.
The network architectures we use are the same as the ones in~\cite{ye2019adversarial}.
We provide the details in Table~\ref{tbl:architecture} and use a factor $w$ to control the width of the network.
Unless specified, the LeNet models on MNIST have a width factor of $16$, the VGG and ResNet18 models on CIFAR10 have a width factor of $8$.

\begin{table}[ht]
\centering
\small
\begin{tabular}{|c|c|}
\hline
Name & Architecture \\
\hline
MNIST, LeNet & Conv($2w$), Conv($4w$), FC($196w$, $64w$), FC($64w$, $10$) \\
\hline
\multirow{2}{*}{CIFAR, VGG} & Conv($4w$) $\times$ $2$, M, Conv($8w$) $\times$ $2$, M, Conv($16w$) $\times$ $3$, M \\
& Conv($32w$) $\times$ $3$, M, Conv($32w$) $\times$ $3$, M, A, FC($32w$, $10$) \\
\hline
CIFAR, ResNet18 & ResNet18 in~\cite{he2016deep}, which uses a width $w=16$ \\
\hline
\end{tabular}
\vspace{0.2cm}
\caption{Network architectures. Conv, FC, M and A represent convolutional layers, fully-connected layers, max-pooling layers and average pooling layers, respectively. The parameter of the convolutional layers indicates the number of output channels. The parameters of the fully-connected layers indicate the number of input and output neurons. The kernel sizes of the max-pooling layers and average pooling layers are always 2. $w$ corresponds to the width factor mentioned in Section~\ref{sec:minima} of the main paper.} \label{tbl:architecture}
\end{table}

We train the models for $100$ epochs on MNIST and $200$ epochs on CIFAR10.
Unless explicitly mentioned, for LeNet models on MNIST, we use Adam~\cite{kingma2014adam} with a learning rate of $1 \times 10^{-4}$.
For VGG models on CIFAR10, we also use Adam, with an initial learning rate of $1 \times 10^{-3}$, decreased exponentially to $1 \times 10^{-4}$ between the $100$th epoch and the $150$th epoch, and then fixed to $1 \times 10^{-4}$ after 150 epochs.
For ResNet18 models on CIFAR10, we use accelerated SGD with a momentum factor of $0.9$. The initial learning rate is $0.1$ and is divided by 10 after $100$ and $150$ epochs.

\begin{figure*}[!ht]
\centering
\begin{subfigure}{0.24\textwidth}
\includegraphics[scale = 0.2]{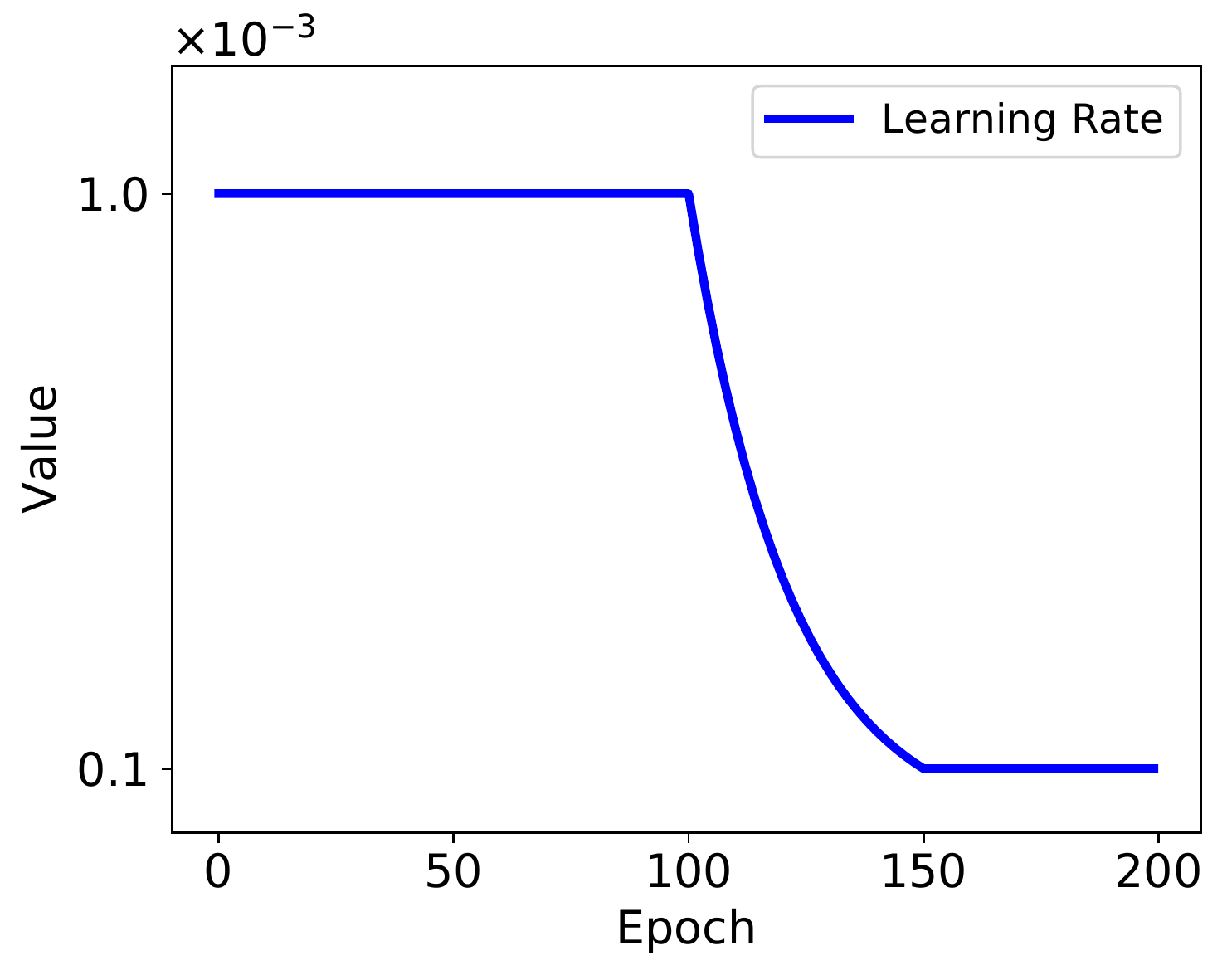}
\captionsetup{font=scriptsize}
\caption{CIFAR10, VGG, Vanilla}
\end{subfigure}
\begin{subfigure}{0.24\textwidth}
\includegraphics[scale = 0.2]{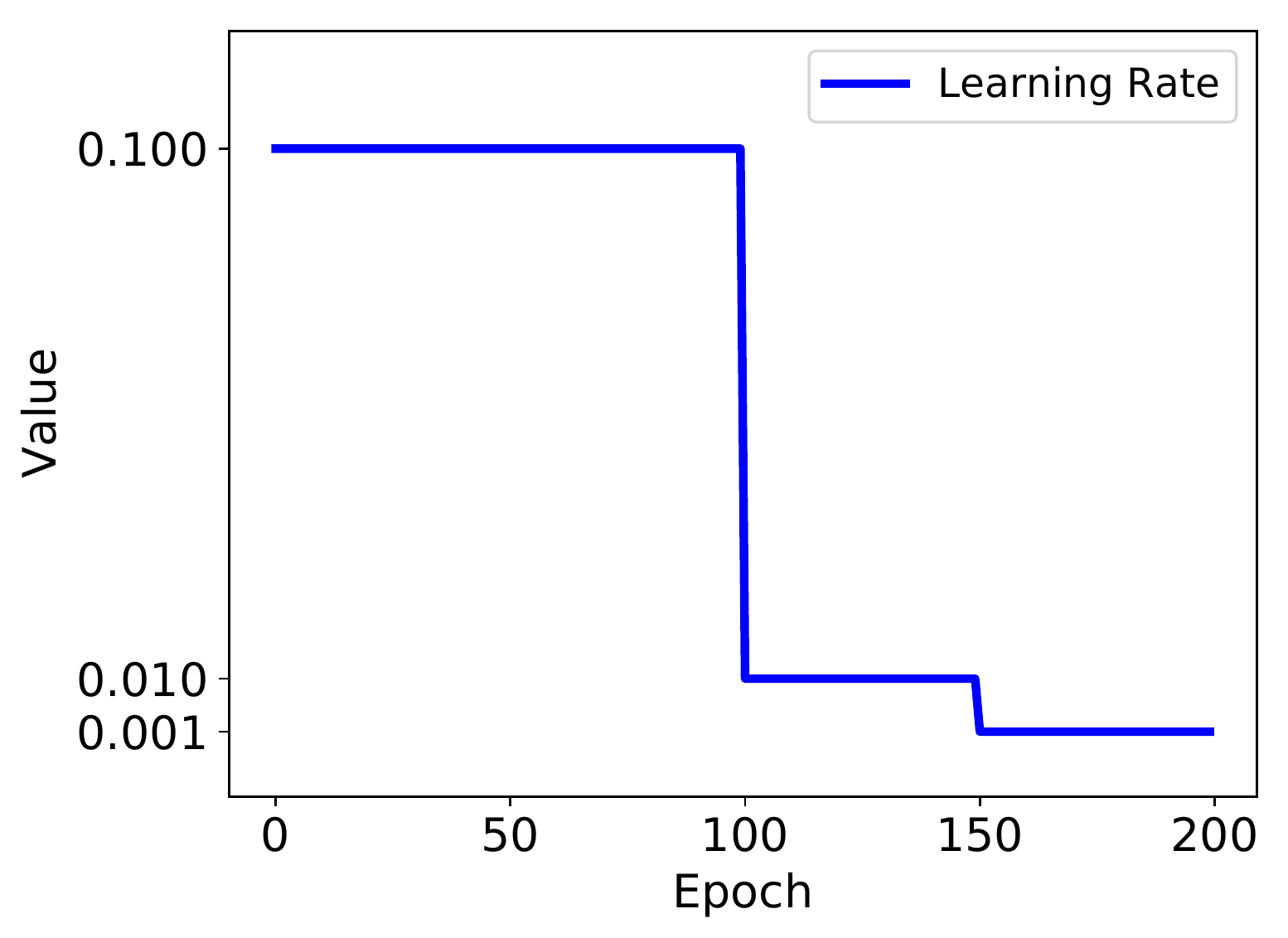}
\captionsetup{font=scriptsize}
\caption{CIFAR10, ResNet18, Vanilla}
\end{subfigure}
\begin{subfigure}{0.24\textwidth}
\includegraphics[scale = 0.2]{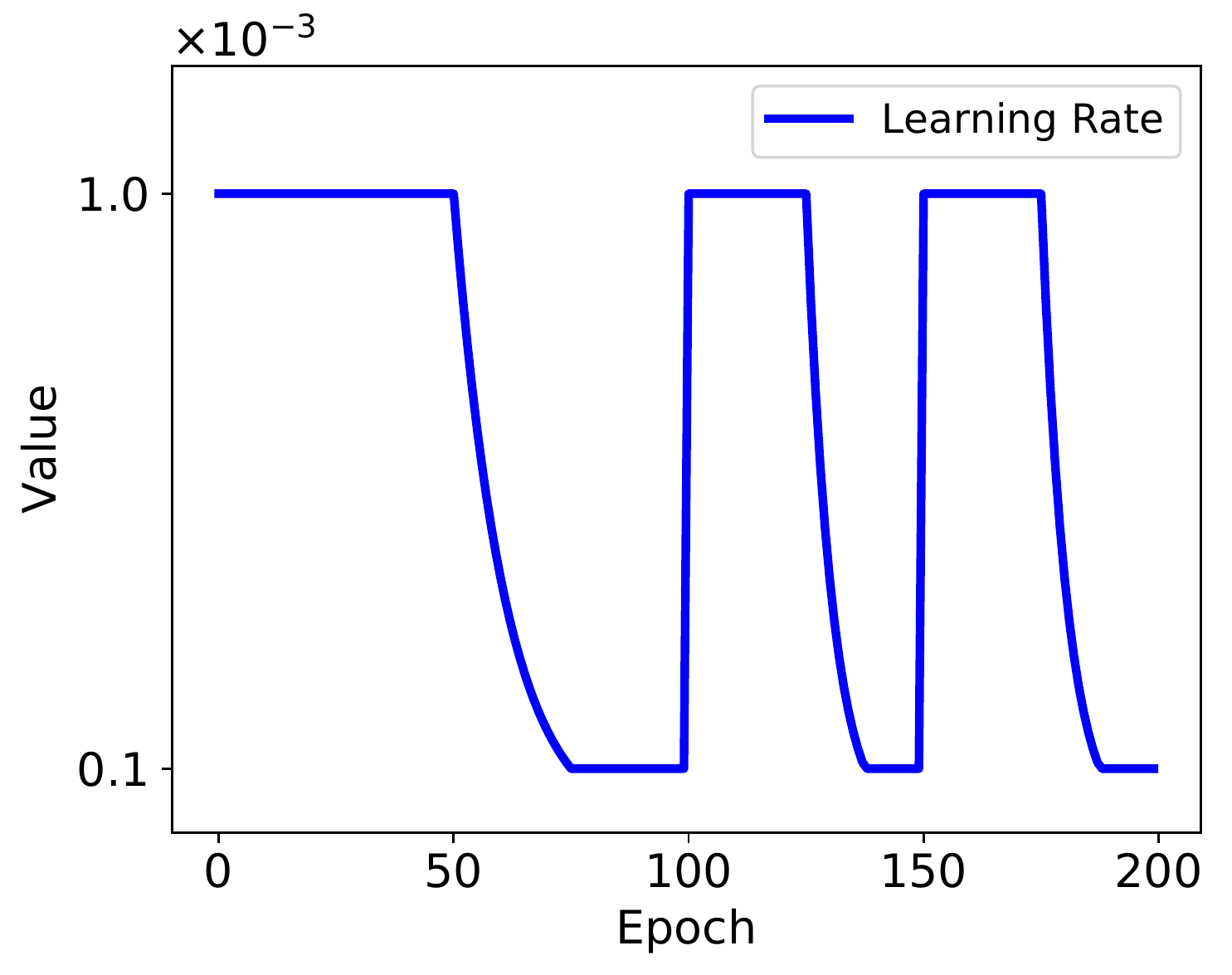}
\captionsetup{font=scriptsize}
\caption{CIFAR10, VGG, Periodic}
\end{subfigure}
\begin{subfigure}{0.24\textwidth}
\includegraphics[scale = 0.2]{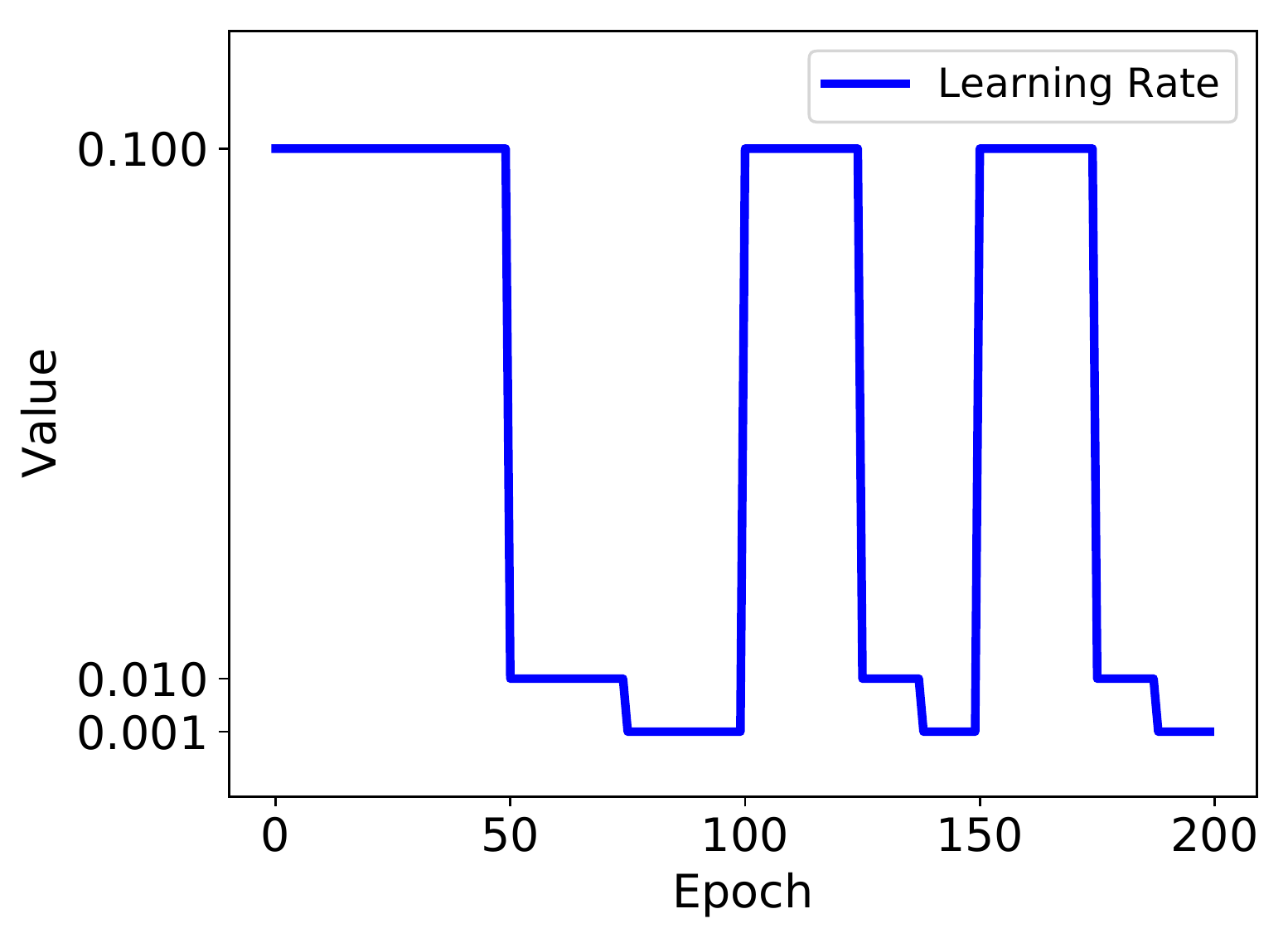}
\captionsetup{font=scriptsize}
\caption{CIFAR10, ResNet18, Periodic}
\end{subfigure}
\caption{Learning rate scheduling for VGG-8 and ResNet18-8 for CIFAR10 classification.} \label{fig:lr_plot}
\end{figure*}

% \textbf{Experiments in Section~\ref{sec:intro}} For experiments demonstrated in Figure~\ref{fig:example}, we use ResNet18 models on CIFAR10 dataset and set the size of the adversarial budget to be $8/255$.
% The model width is $16$ for \textit{VT-16} / \textit{AT-16} and $32$ for \textit{AT-32} / \textit{Semi-AT-32}.
% \textit{Semi-AT-32} follows the practice of~\cite{carmon2019unlabeled}, it augments the training set by $500$K pseudo-labelled images from 80M-TI dataset~\cite{torralba200880}\footnote{Available on \href{https://github.com/yaircarmon/semisup-adv}{https://github.com/yaircarmon/semisup-adv}}.
% During the training, we select $50$K images, $5$K for each label, from these $500$K images to keep the balance between the original CIFAR10 training set and the augmented one.
% In the curves of Figure~\ref{fig:example}, we report the model's robust accuracy on the original CIFAR10 training and test sets.

\paragraph{Experiments in Section~\ref{sec:pas}.} The details of the adversarial attacks used in this section are demonstrated below:

\begin{itemize}
    \item PGD: for MNIST models, PGD with 50 iterations, the step size is $0.1$; for CIFAR10 models, PGD with 10 iterations, the step size is $2 / 255$.
    \item PGD100: PGD with 100 iterations, the step size is $0.1$ for MNIST models and $1 / 255$ for CIFAR10 models.
    \item APGD100-CE: AutoPGD with 100 iterations and cross-entropy loss. We use the default settings in~\cite{croce2020reliable}, i.e., $\rho = 0.75$, $\alpha = 0.75$.
    \item APGD100-DLR: AutoPGD with 100 iterations and difference-of-logit-ratio loss. We use the default settings in~\cite{croce2020reliable}, i.e., $\rho = 0.75$, $\alpha = 0.75$.
    \item Square5K: SquareAttack with 5000 iterations. We use the default settings in~\cite{andriushchenko2019square}, i.e., we use the \textit{margin loss}.
\end{itemize}

For the results in Table~\ref{tbl:eps_scheduling_results}, we fine-tune the weight-decay factor, choosing $1 \times 10^{-3}$ as the optimal value.
In periodic settings, the learning rate and the adversarial budget are reset after $100$ and $150$ epochs.
The scheduling in each period is scaled proportionally.
We plot the learning rate scheduling curves for VGG-8 and ResNet18-8 in Figure~\ref{fig:lr_plot} for both the vanilla and periodic settings.
Regarding the scheduling of $\epsilon$, we do not fully explore the value range of the hyper-parameters in the cosine and linear schedulers.
We use $\epsilon_{min} = 0$ for all experiments.
For the MNIST experiments, we set $\epsilon_{max} = 0.6$ for the cosine scheduler and $\epsilon_{max} = 0.8$ for the linear one.
For the CIFAR10 experiments, we set $\epsilon_{max} = 16 / 255$ for both the cosine and linear schedulers.
We plot the curves for $\epsilon_{cos}(d)$ and $\epsilon_{lin}(d)$ in Figure~\ref{fig:eps_plot}.

\begin{figure}[!ht]
\centering
\begin{subfigure}{0.24\textwidth}
\includegraphics[scale = 0.2]{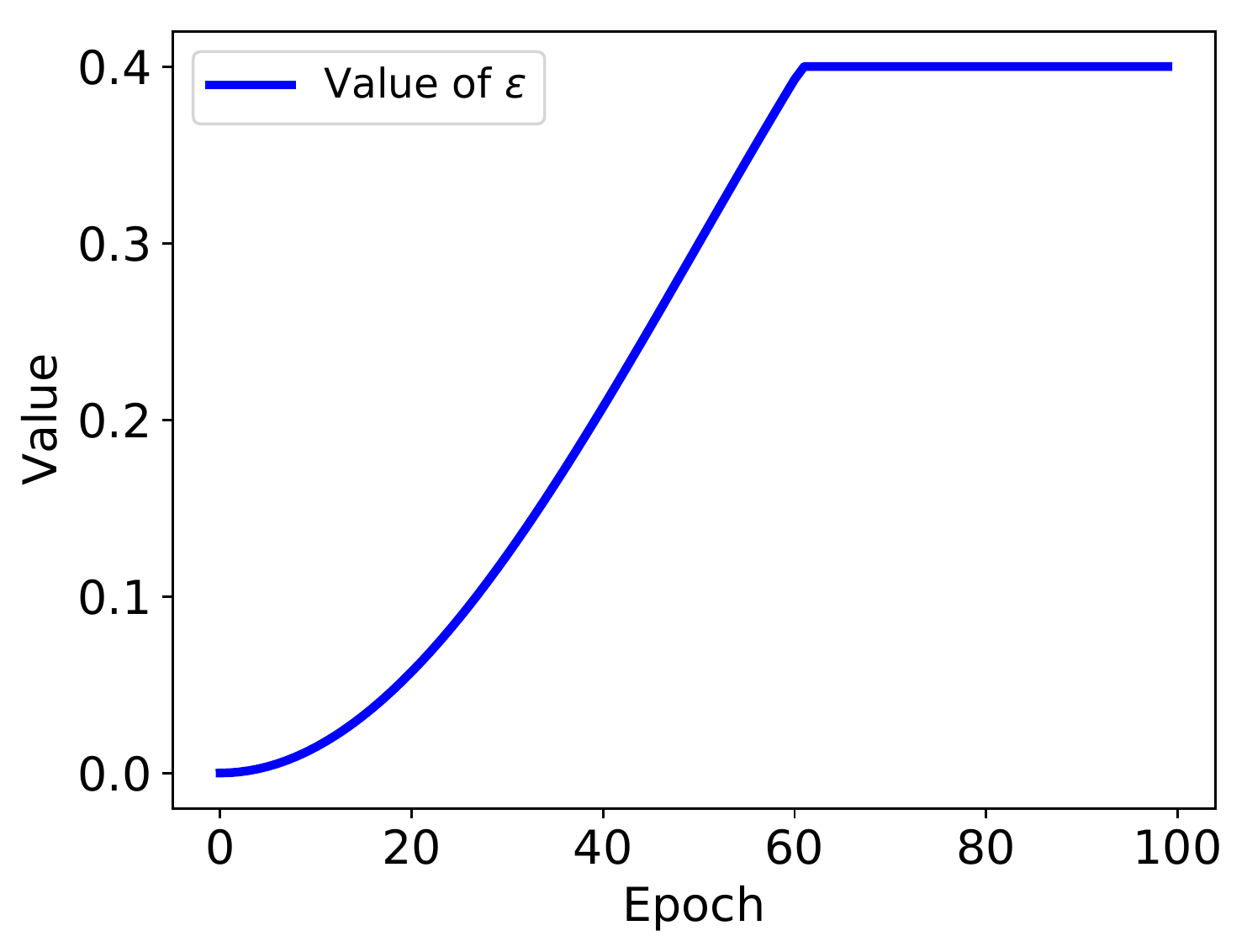}
\captionsetup{font=scriptsize}
\caption{MNIST, Cosine}
\end{subfigure}
\begin{subfigure}{0.24\textwidth}
\includegraphics[scale = 0.2]{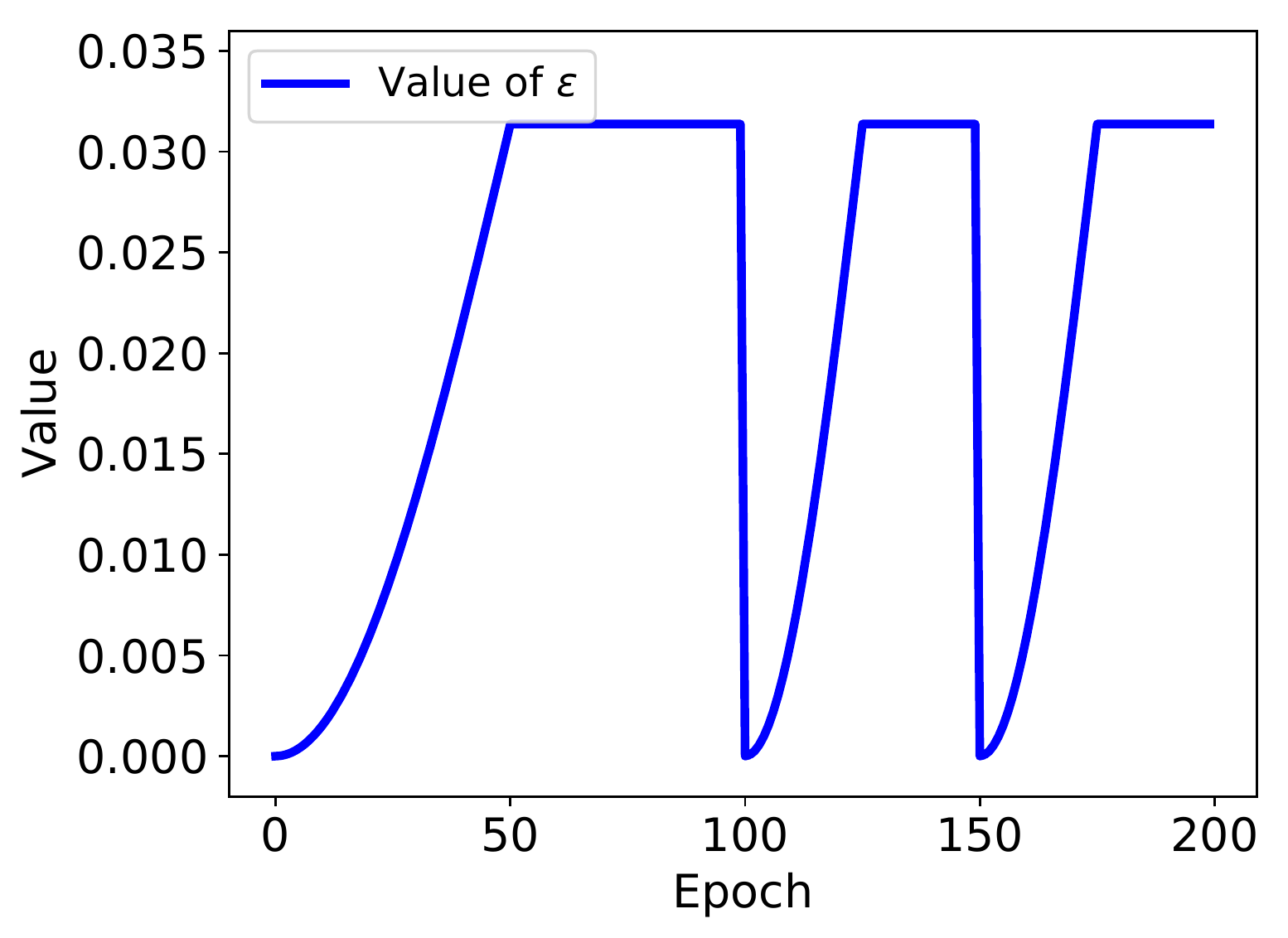}
\captionsetup{font=scriptsize}
\caption{CIFAR10, Cosine}
\end{subfigure}
\begin{subfigure}{0.24\textwidth}
\includegraphics[scale = 0.2]{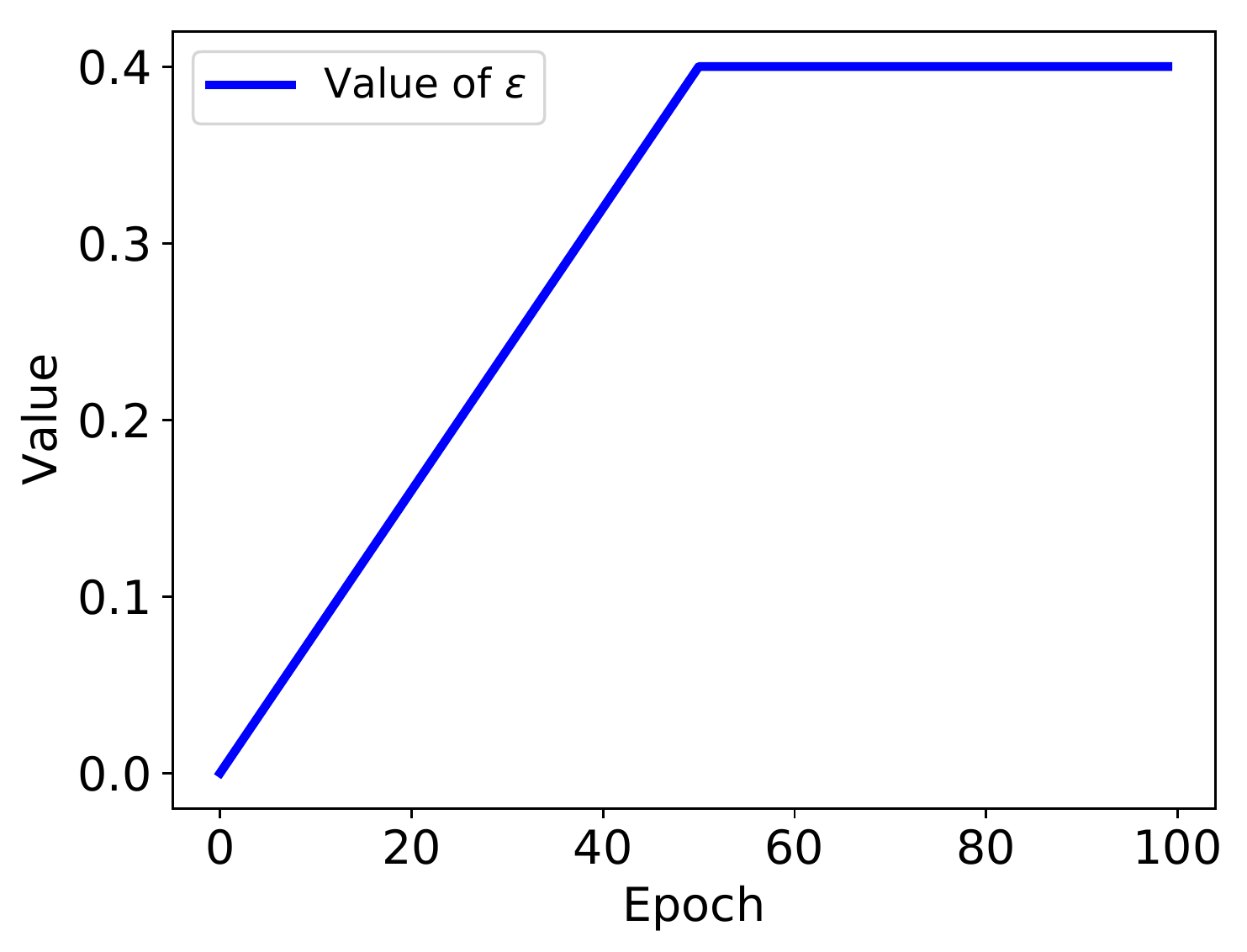}
\captionsetup{font=scriptsize}
\caption{MNIST, Linear}
\end{subfigure}
\begin{subfigure}{0.24\textwidth}
\includegraphics[scale = 0.2]{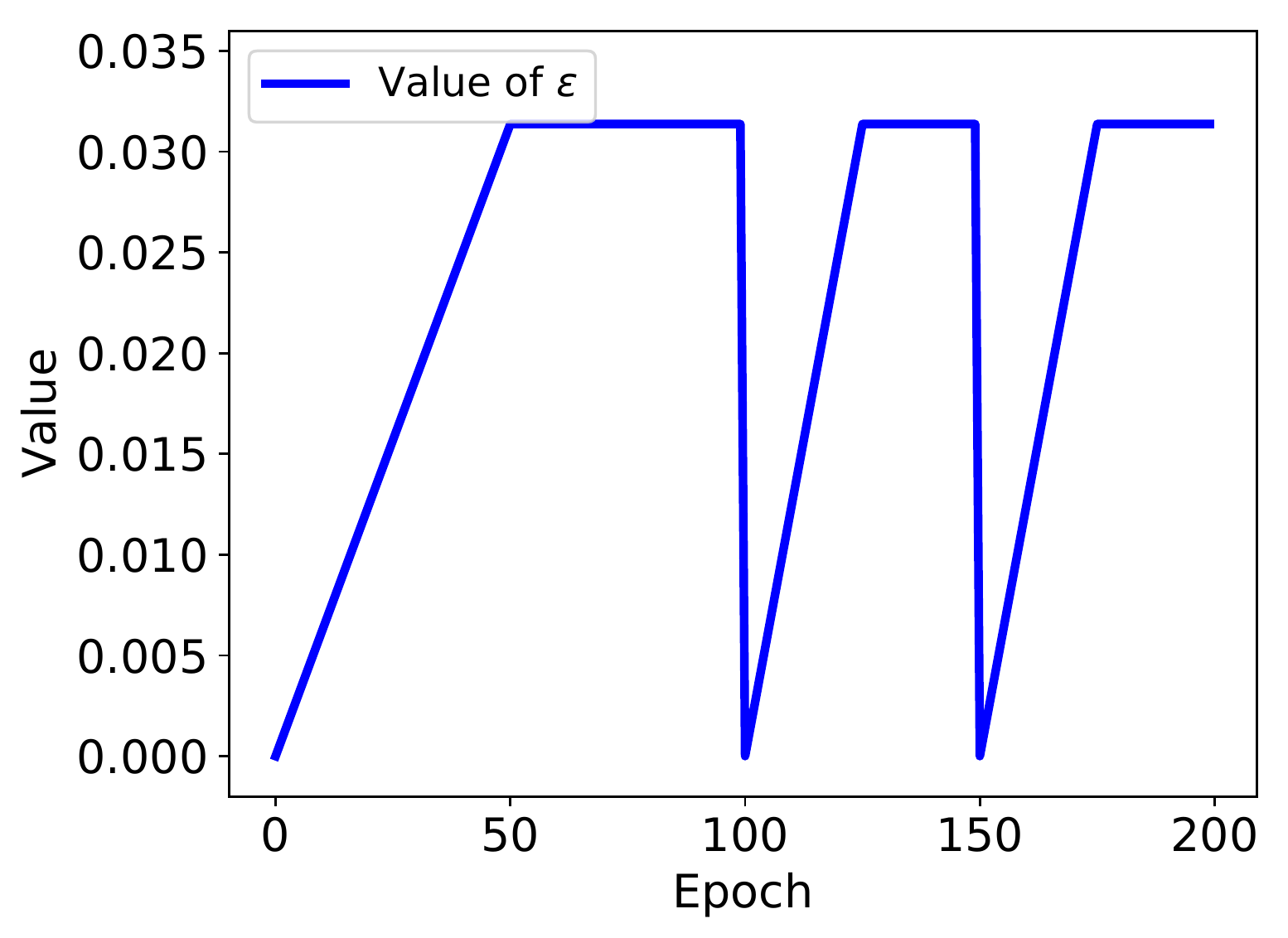}
\captionsetup{font=scriptsize}
\caption{CIFAR10, Linear}
\end{subfigure}
\caption{Adversarial budget scheduling for MNIST and CIFAR10 models.} \label{fig:eps_plot}
\end{figure}

\subsection{Additional Experimental Results} \label{subsec:add_exp}

\subsubsection{Additional Results for Section~\ref{subsec:early_training}} \label{subsubsec:analysis_lenet}

To complement the results on CIFAR10 models in Section~\ref{subsec:early_training}, in Figure~\ref{fig:analysis_lenet}, we provide a numerical analysis on the first and last 500 training mini-batches of MNIST under different values of $\epsilon$.
We observe the same phenomenon: in the early stages of training, a large adversarial budget leads to smaller gradient magnitudes and slows down the training; in the final stages of training, a large adversarial budget yields severe gradient scattering, indicated by larger gradient magnitudes.

\begin{figure}[!ht]
\centering
\begin{tabular}{cccc}
\begin{subfigure}{0.26\textwidth}
\includegraphics[scale = 0.22]{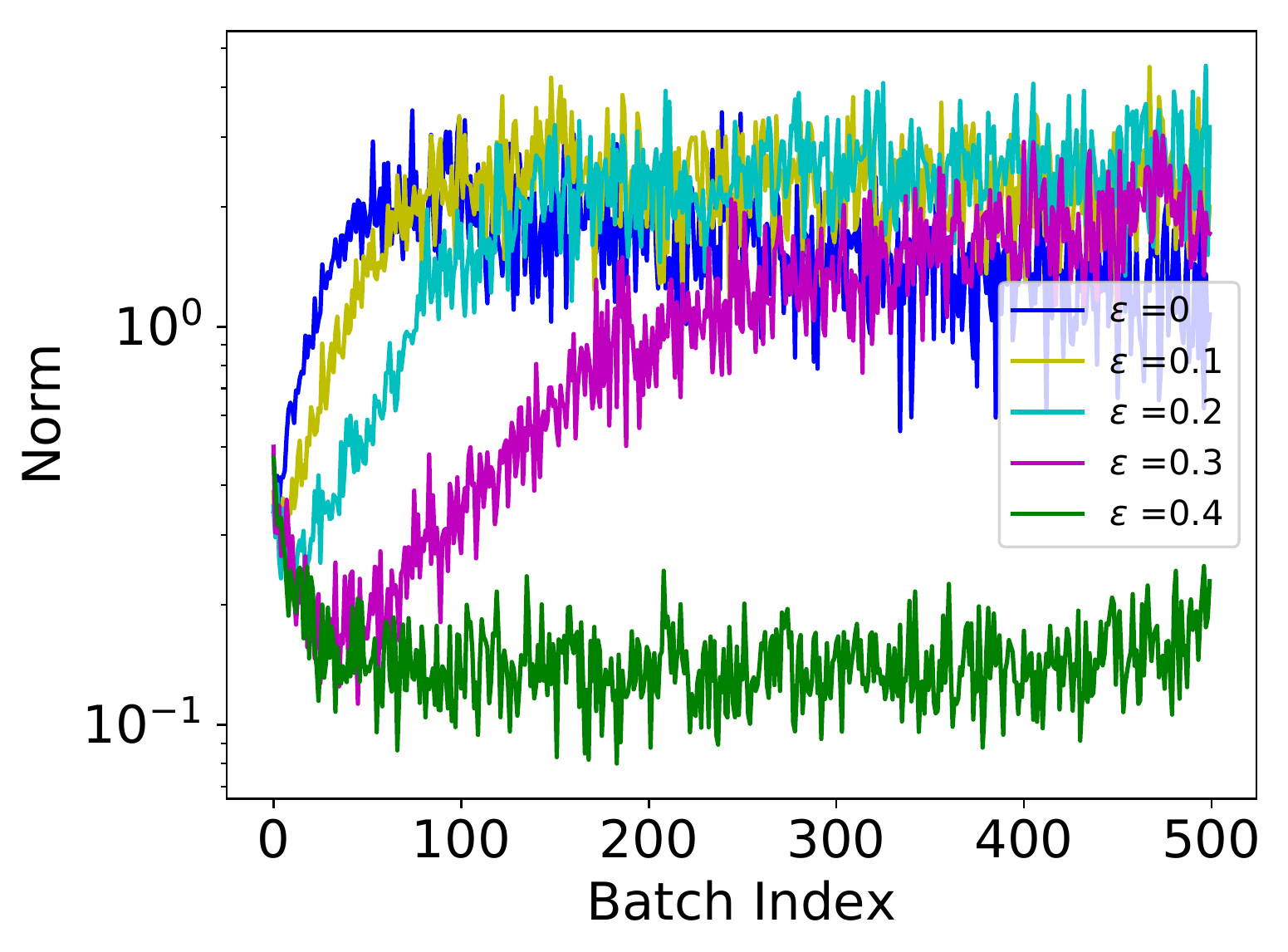}
\captionsetup{font=small}
\caption{$\|\triangledown_\theta \widehat{\Loss}_\epsilon(\theta)\|$, first 500.} \label{subfig:lenet_grad_early}
\end{subfigure}
\begin{subfigure}{0.24\textwidth}
\includegraphics[scale = 0.22]{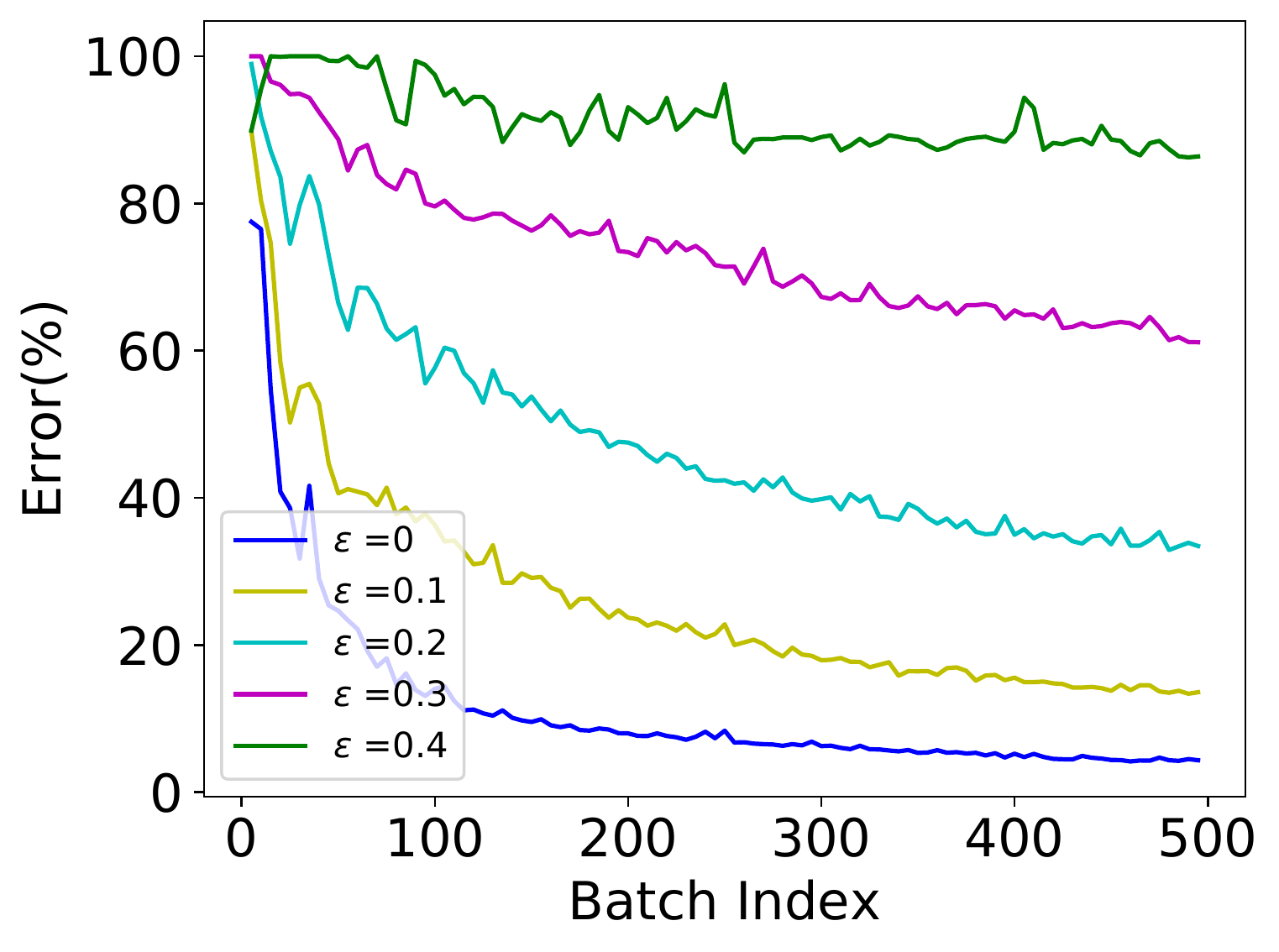}
\captionsetup{font=small}
\caption{$\Error_\epsilon(\theta)$, first 500.} \label{subfig:lenet_error_early}
\end{subfigure}
\begin{subfigure}{0.24\textwidth}
\includegraphics[scale = 0.22]{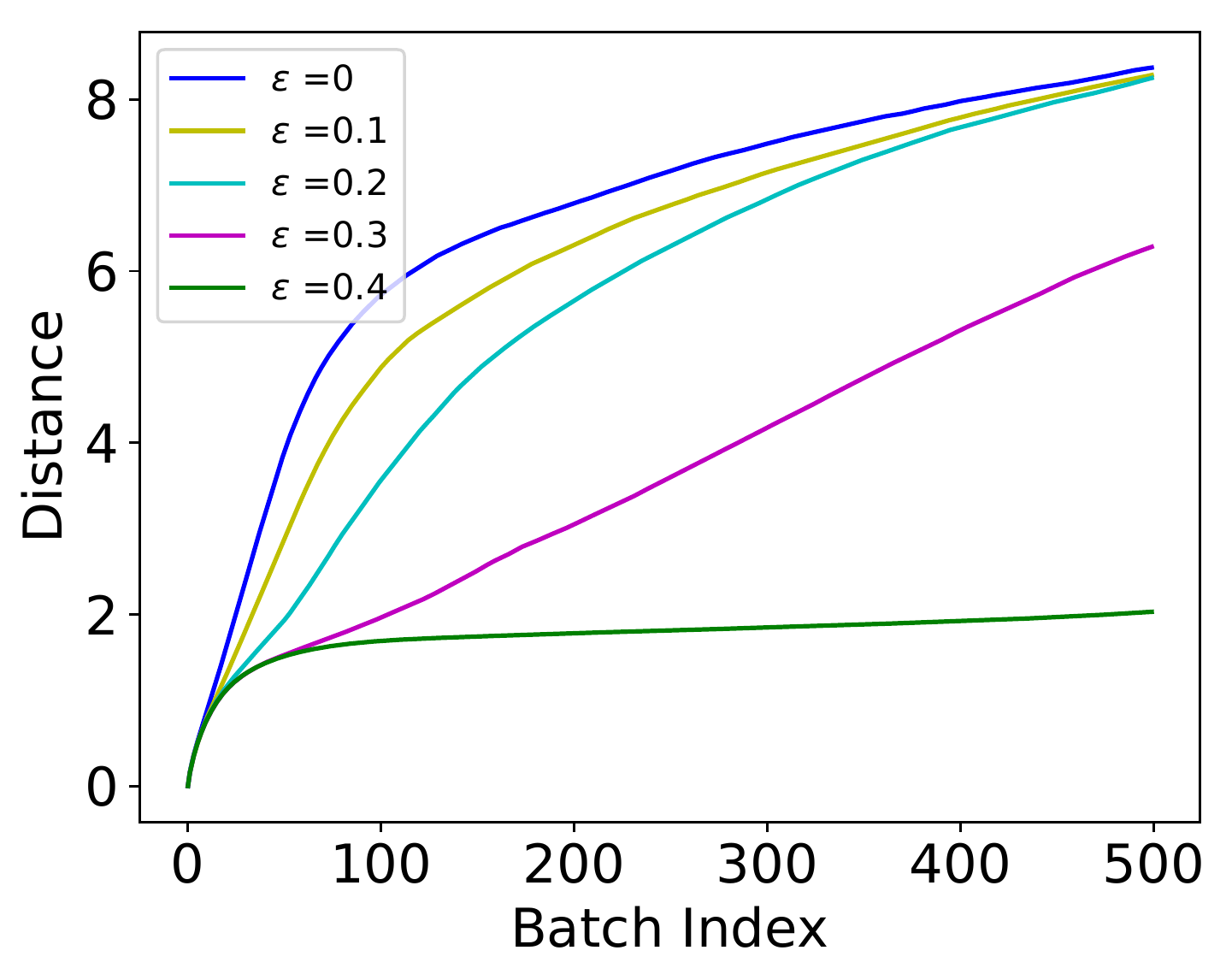}
\captionsetup{font=small}
\caption{$\|\theta - \theta_0\|$, first 500.} \label{subfig:lenet_distance_early}
\end{subfigure}
\begin{subfigure}{0.24\textwidth}
\includegraphics[scale = 0.22]{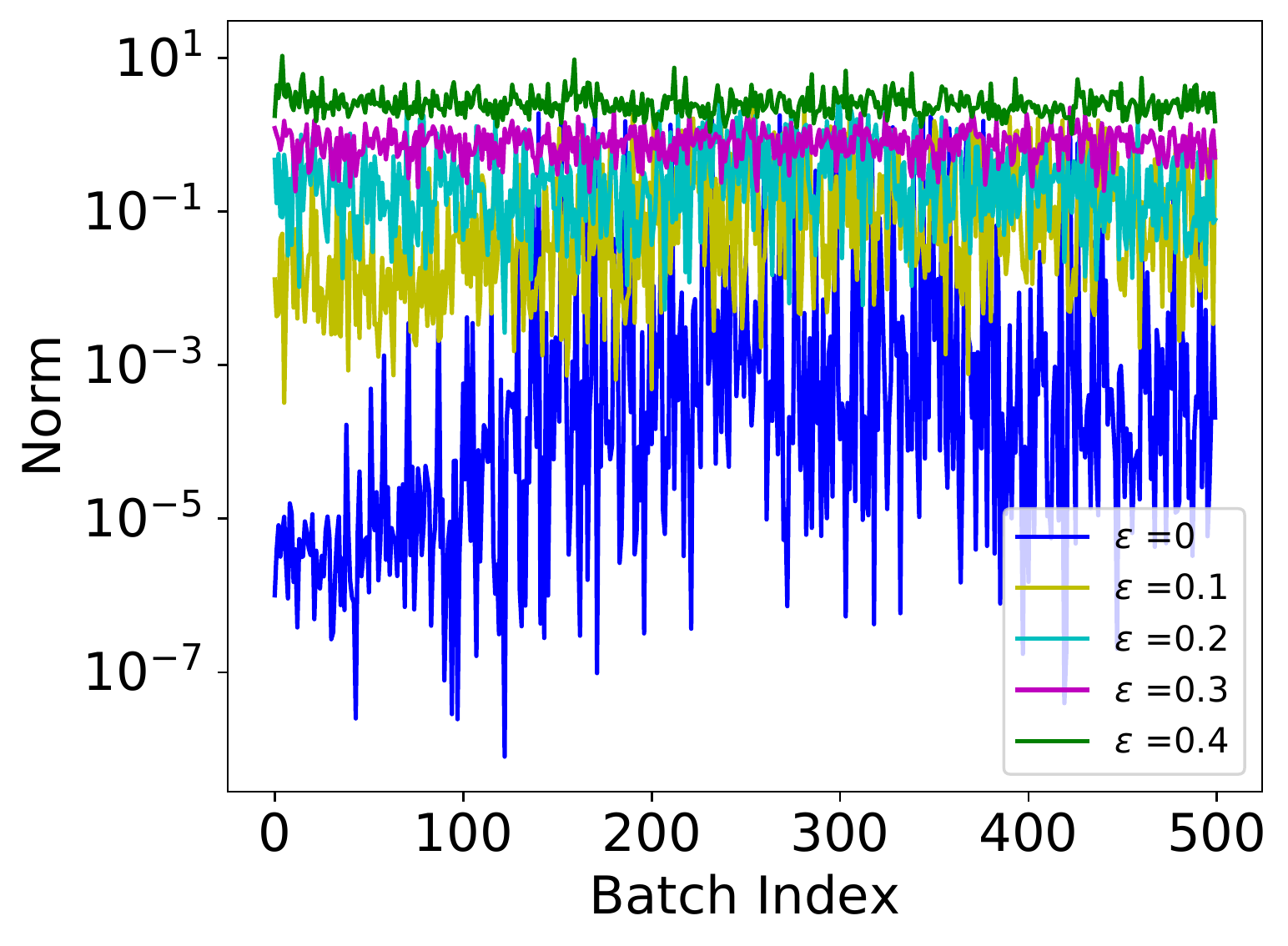}
\captionsetup{font=small}
\caption{$\|\triangledown_\theta \widehat{\Loss}_\epsilon(\theta)\|$, last 500.} \label{subfig:lenet_grad_final}
\end{subfigure}
\end{tabular}
\caption{Norm of the stochastic gradient $\|\triangledown_\theta \widehat{\Loss}_\epsilon(\theta)\|$, robust training error $\Error_\epsilon(\theta)$, distance from the initial point $\|\theta - \theta_0\|$ during the first or last 500 mini-batch updates for MNIST models.} \label{fig:analysis_lenet}
\end{figure}

% \subsubsection{Distance from the Initial Point} \label{subsubsec:init_distance}

% We track the distance the models have moved from the initial point in Figure~\ref{fig:init_distance} under different sizes of adversarial budgets.
% We use the same inital points for the same architecture and track the first 500 batches for LeNet-16 models and the first 2000 batches for ResNet18-8 models.
% We can clearly see that the models under smaller adversarial budgets move further than the ones under larger adversarial budgets.

% \begin{figure}
% \centering
% \includegraphics[scale = 0.45]{figs/param_scan/lenet_init/distance_500.pdf}~~~~~~
% \includegraphics[scale = 0.45]{figs/param_scan/resnet_init/distance_2000.pdf}
% \caption{Euclidean distance between the initial parameters and parameters during training on LeNet-16 models (left) and ResNet18-8 models (right) under different adversarial budget magnitudes.} \label{fig:init_distance}
% \end{figure}

\subsubsection{Additional Results for Section~\ref{sec:hessian_eps}} \label{subsubsec:add_hessian}

\begin{figure}[!ht]
\begin{minipage}{.48\textwidth}
\centering
\includegraphics[scale = 0.35]{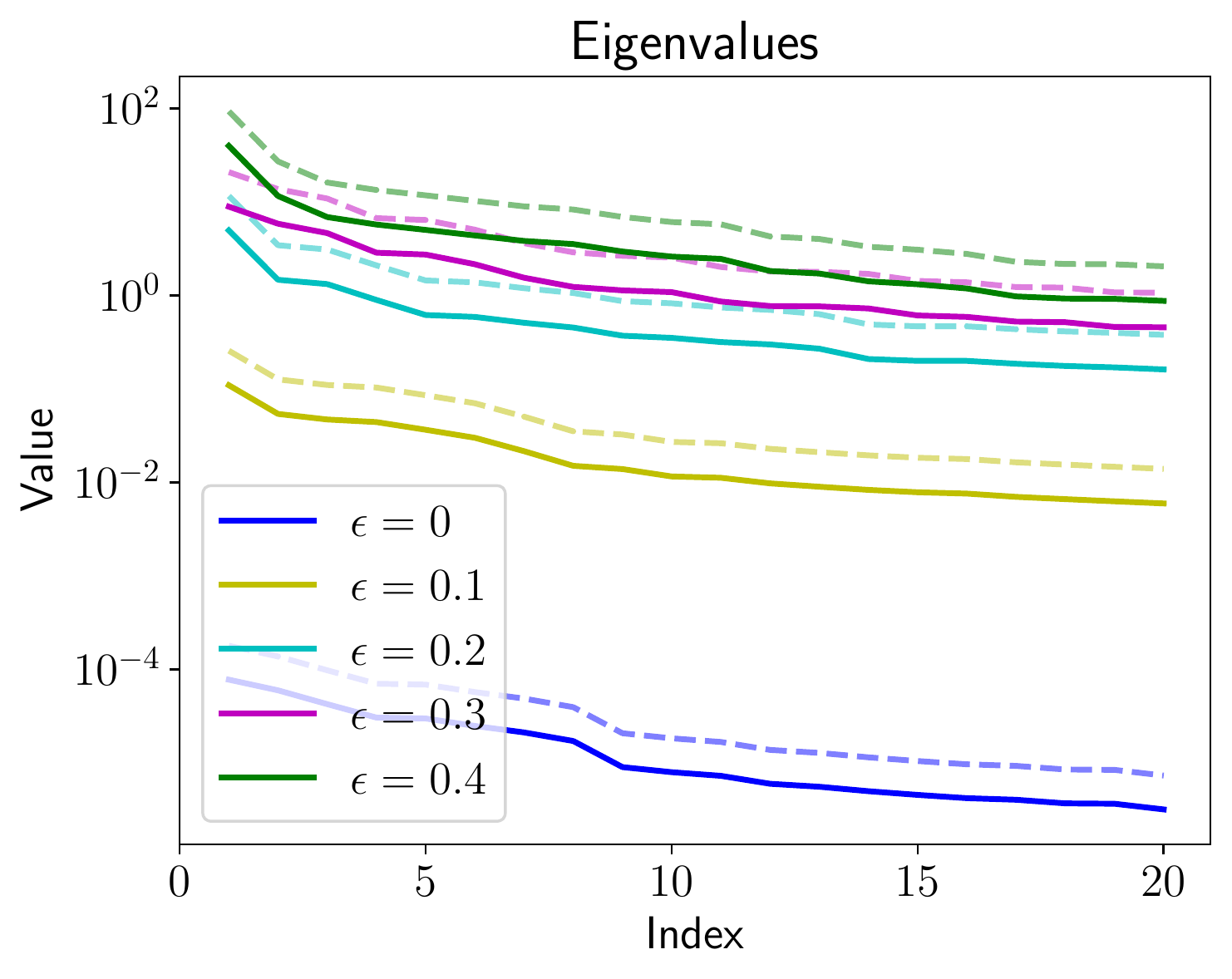}
\caption{Top 20 eigenvalues of the Hessian matrix for LeNet models. Both normalized (solid) and original (dashed) values are shown.} \label{fig:eigen_eps_mnist}
\end{minipage}
\begin{minipage}{.04\textwidth}
\end{minipage}
\begin{minipage}{.48\textwidth}
\centering
\includegraphics[scale = 0.35]{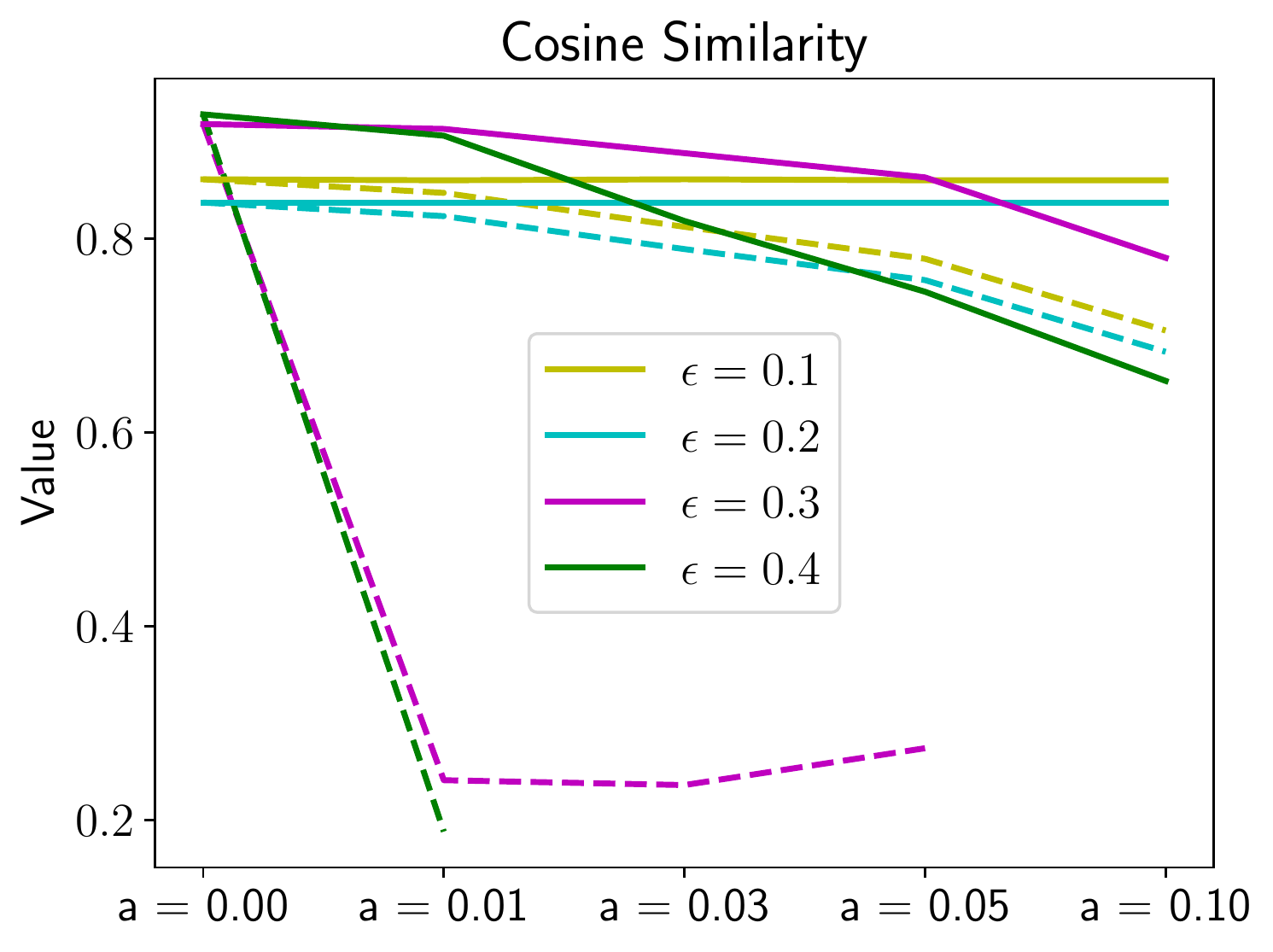}
\caption{Cosine similarity between perturbations $\x'_{a\vi} - \x$ and $\x'_{-a\vi} - \x$. $\vi$ can be either the top eigenvector (dashed) or randomly picked (solid).} \label{fig:sim_mnist}
\end{minipage}
\end{figure}

\begin{figure}
\centering
\includegraphics[scale = 0.4]{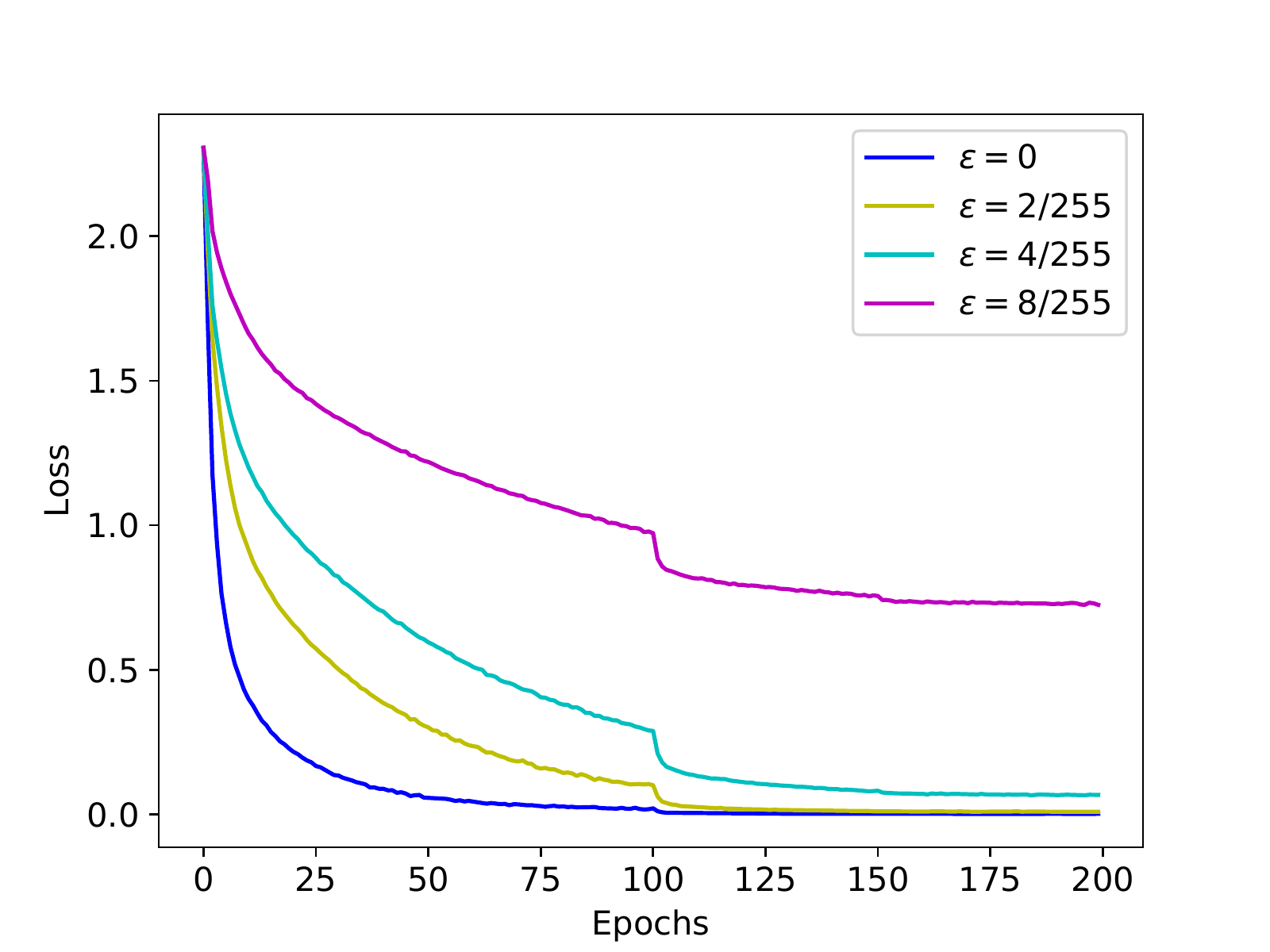}
~~
\includegraphics[scale = 0.4]{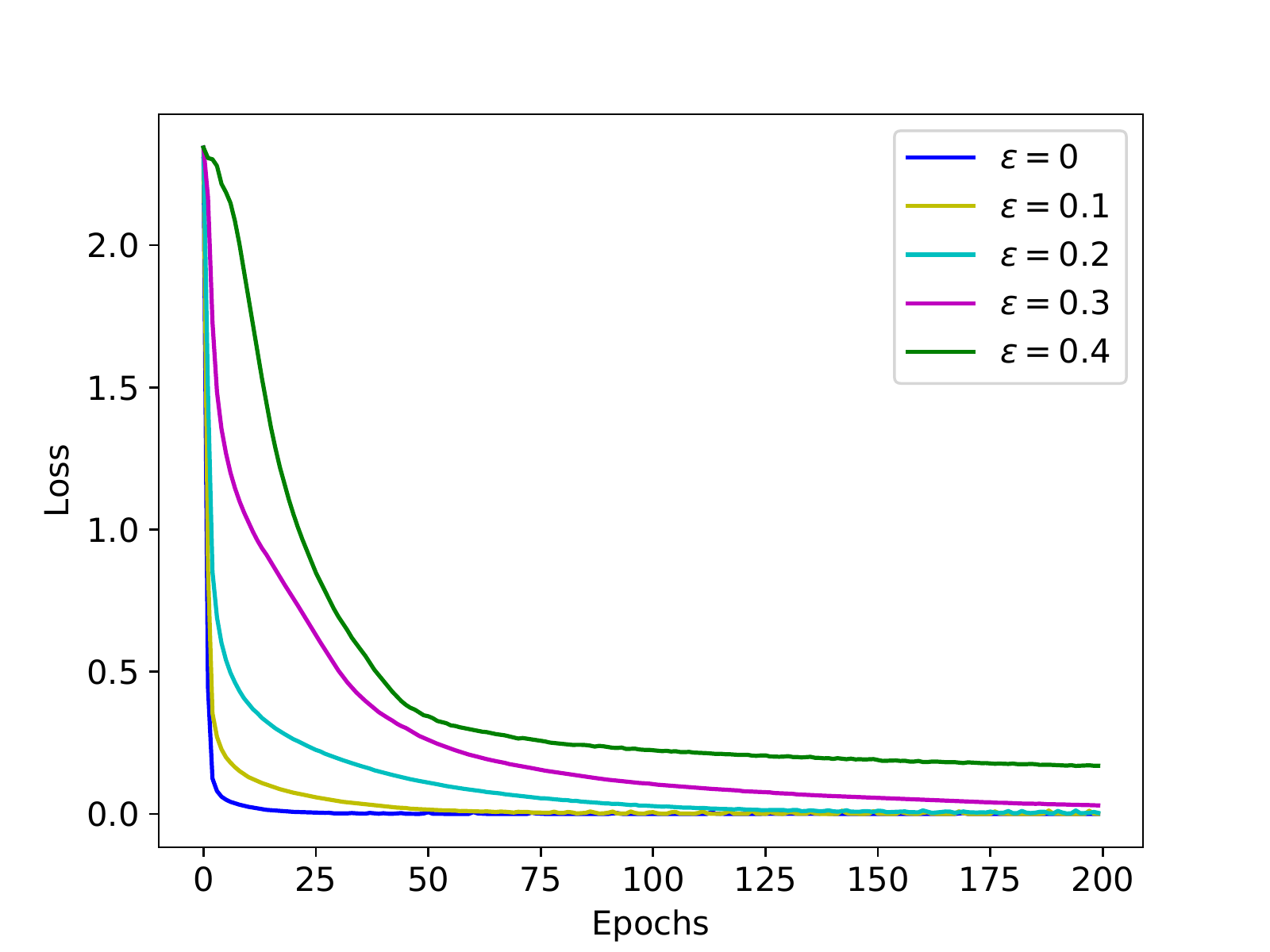}
\caption{The learning curves of training loss for CIFAR10 models (left) and MNIST models (right) under different values of $\epsilon$.} \label{fig:train_curve}
\end{figure}

In Figure~\ref{fig:eigen_eps_mnist}, we provide the Hessian spectrum analysis for LeNet models on MNIST under various adversarial budgets.
As in Figure~\ref{fig:eigen_eps_cifar10}, the top eigenvalues, both the original and normalized values, of the Hessian matrix of our trained models are larger in the presence of larger adversarial budgets.

Note that the magnitudes of $\Loss_\epsilon(\theta)$ with different $\epsilon$ are similar.
To show this, we randomly sample $10$ $\theta$ and calculate the value of $\Loss_\epsilon(\theta)$ over the training set.
For CIFAR10 models, the mean values are $2.3034$, $2.3044$, $2.3053$, $2.3071$ when $\epsilon$ is $0$, $2 / 255$, $4 / 255$ and $8 / 255$, respectively.
For MNIST models, the mean values are $2.3029$, $2.3414$, $2.3424$, $2.3429$, $2.3432$ when $\epsilon$ is $0$, $0.1$, $0.2$, $0.3$ and $0.4$, respectively.
In Figure~\ref{fig:train_curve}, we plot the learning curves of the training loss; this clearly shows that the magnitudes of $\Loss_\epsilon(\theta)$ during training with different $\epsilon$ values are similar.
Furthermore, the range of values of $\Loss_\epsilon(\theta)$ during training is smaller under large values of $\epsilon$.
As a result, the increased curvature under large adversarial budgets is not caused by the magnitudes of the function $\Loss_\epsilon(\theta)$, and we empirically observed that optimization in adversarial training cannot be facilitated by tuning the learning rate.

Ideally, the sharpness of the minima is depicted by the condition number of its Hessian matrix.
However, in the context of deep neural networks, the eigenvalue with the smallest absolute value of the Hessian matrix is almost zero, which renders the computation of the condition number both algorithmically and numerically unstable~\cite{ghorbani2019investigation}.
Instead, the spectral norm and the nuclear norm of the Hessian matrix are typically used as quantitative metrics for the sharpness of the minima~\cite{dinh2017sharp}.
Figure~\ref{fig:eigen_eps_cifar10} and Figure~\ref{fig:eigen_eps_mnist} thus demonstrate that the obtained minima are shaper when $\epsilon$ is larger.

\begin{figure}[!ht]
\centering
\begin{tabular}[t]{cccc}
    \begin{subfigure}{0.23\textwidth}
    \includegraphics[scale = 0.18]{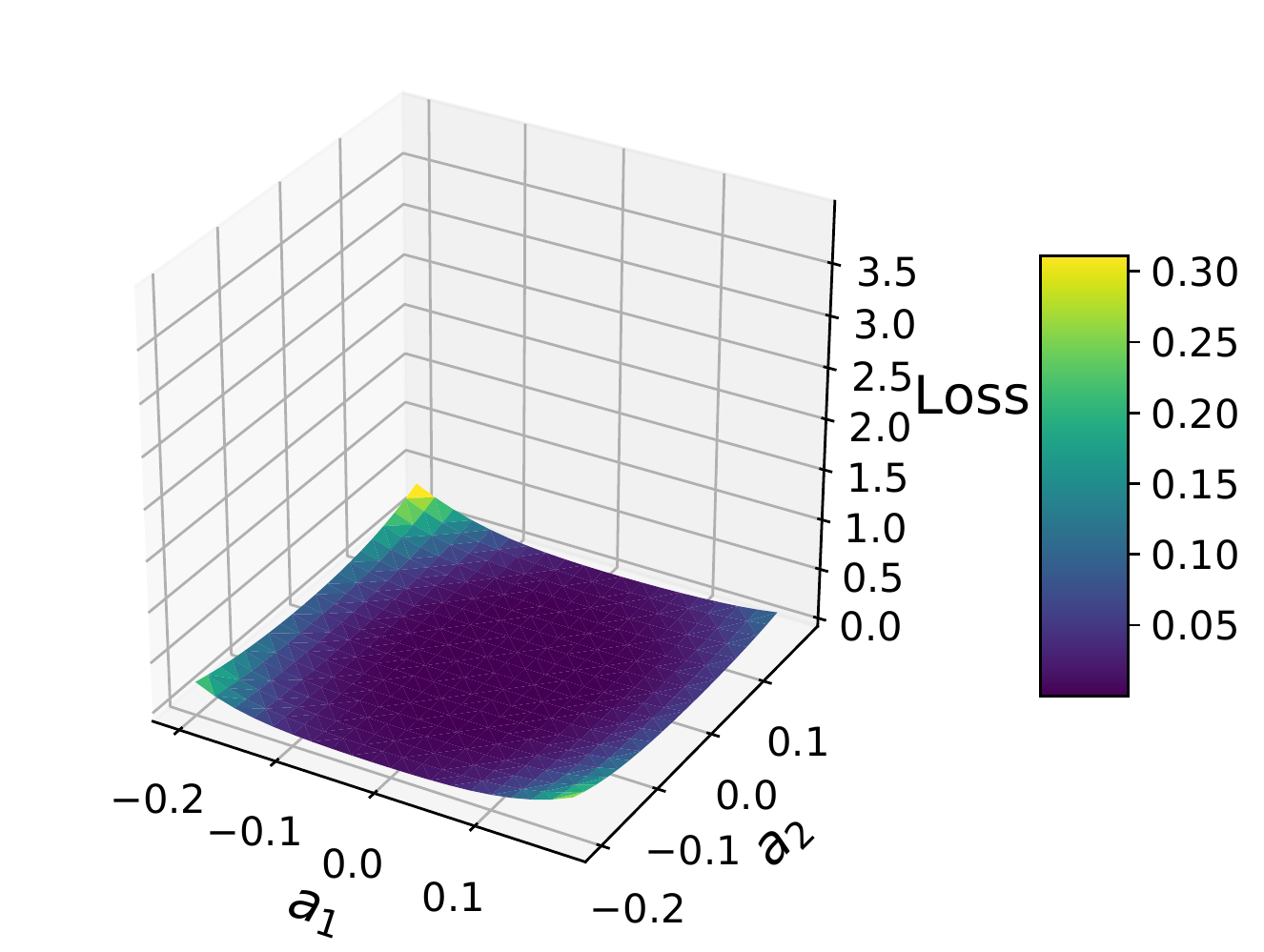}
    \captionsetup{font=small}
    \caption{MNIST, $\epsilon = 0.0$} \label{fig:landscape1_eps_mnist}
    \end{subfigure} &
    \begin{subfigure}{0.23\textwidth}
    \includegraphics[scale = 0.18]{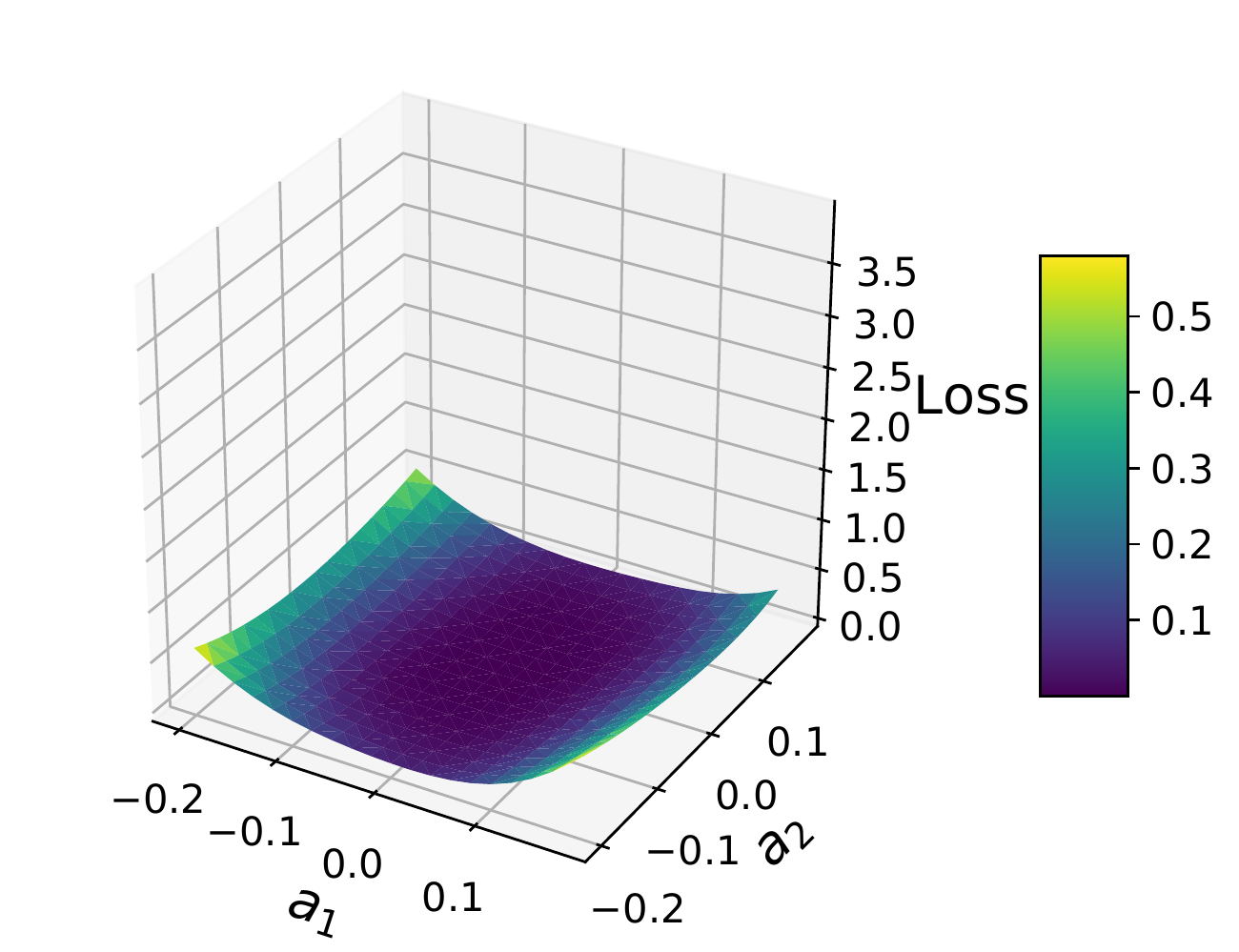}
    \captionsetup{font=small}
    \caption{MNIST, $\epsilon = 0.1$} \label{fig:landscape2_eps_mnist}
    \end{subfigure} &
    \begin{subfigure}{0.23\textwidth}
    \includegraphics[scale = 0.18]{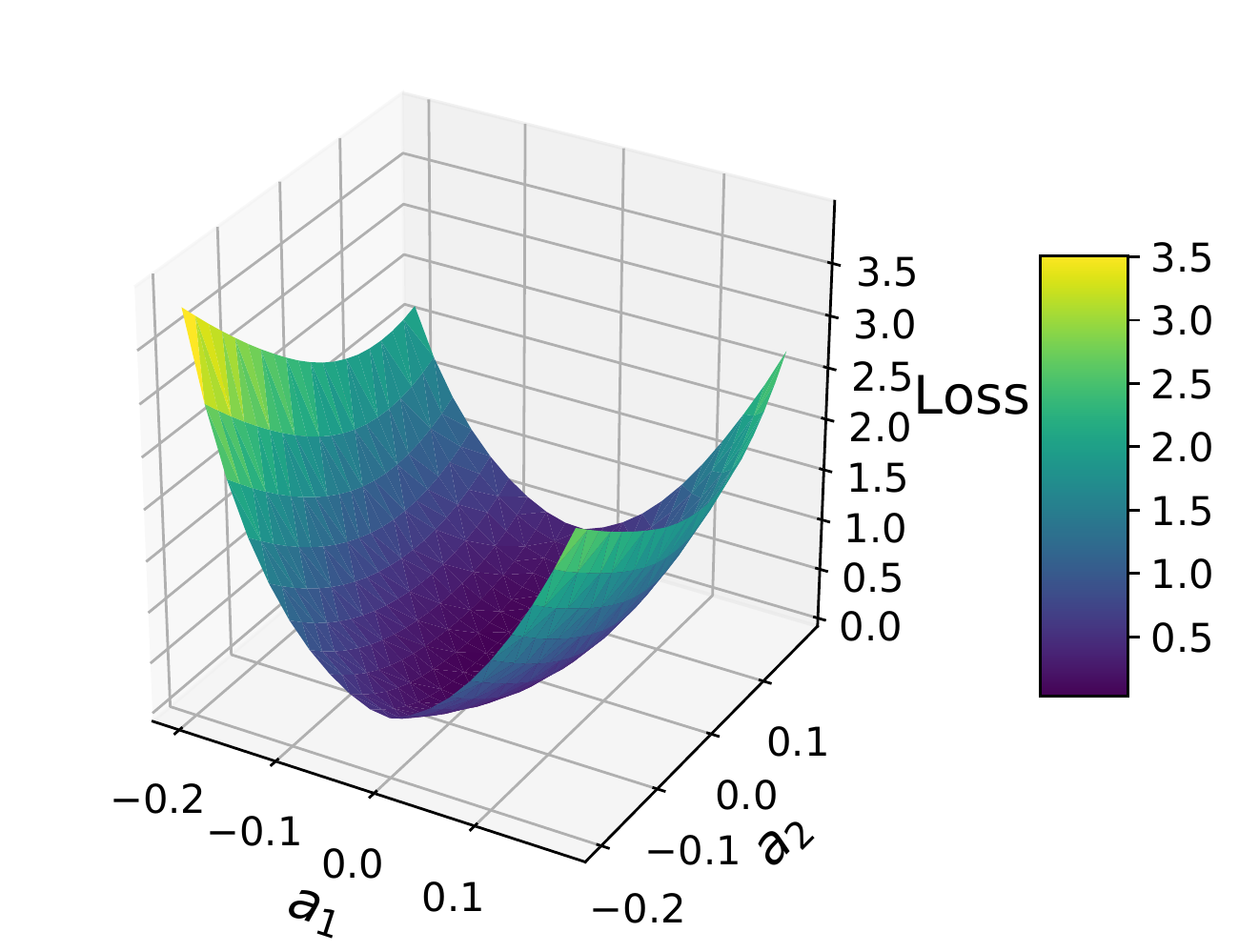}
    \captionsetup{font=small}
    \caption{MNIST, $\epsilon = 0.3$} \label{fig:landscape3_eps_mnist}
    \end{subfigure} &
    \begin{subfigure}{0.23\textwidth}
    \includegraphics[scale = 0.18]{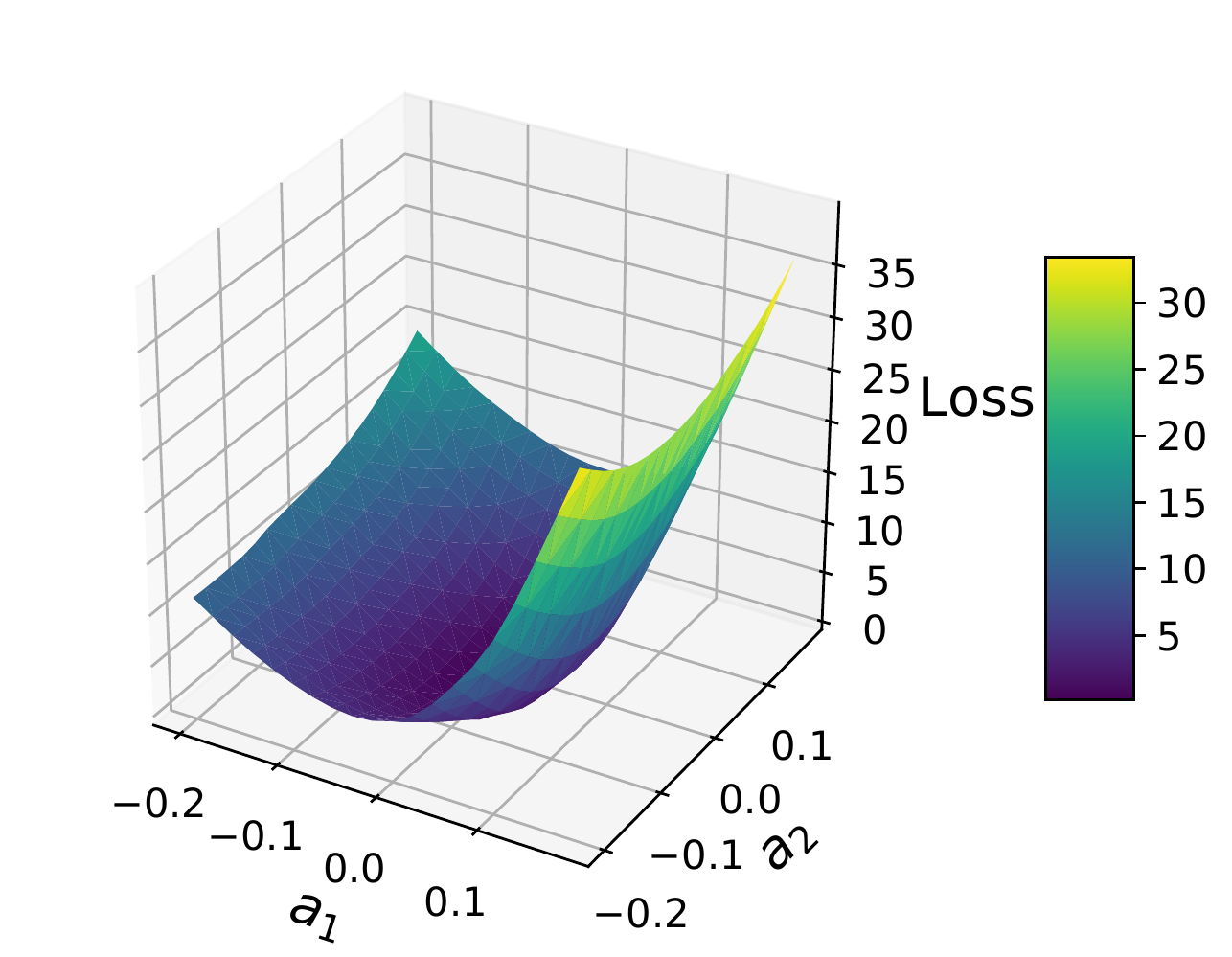}
    \captionsetup{font=small}
    \caption{MNIST, $\epsilon = 0.4$} \label{fig:landscape4_eps_mnist}
    \end{subfigure} \\
    \begin{subfigure}{0.23\textwidth}
    \includegraphics[scale = 0.18]{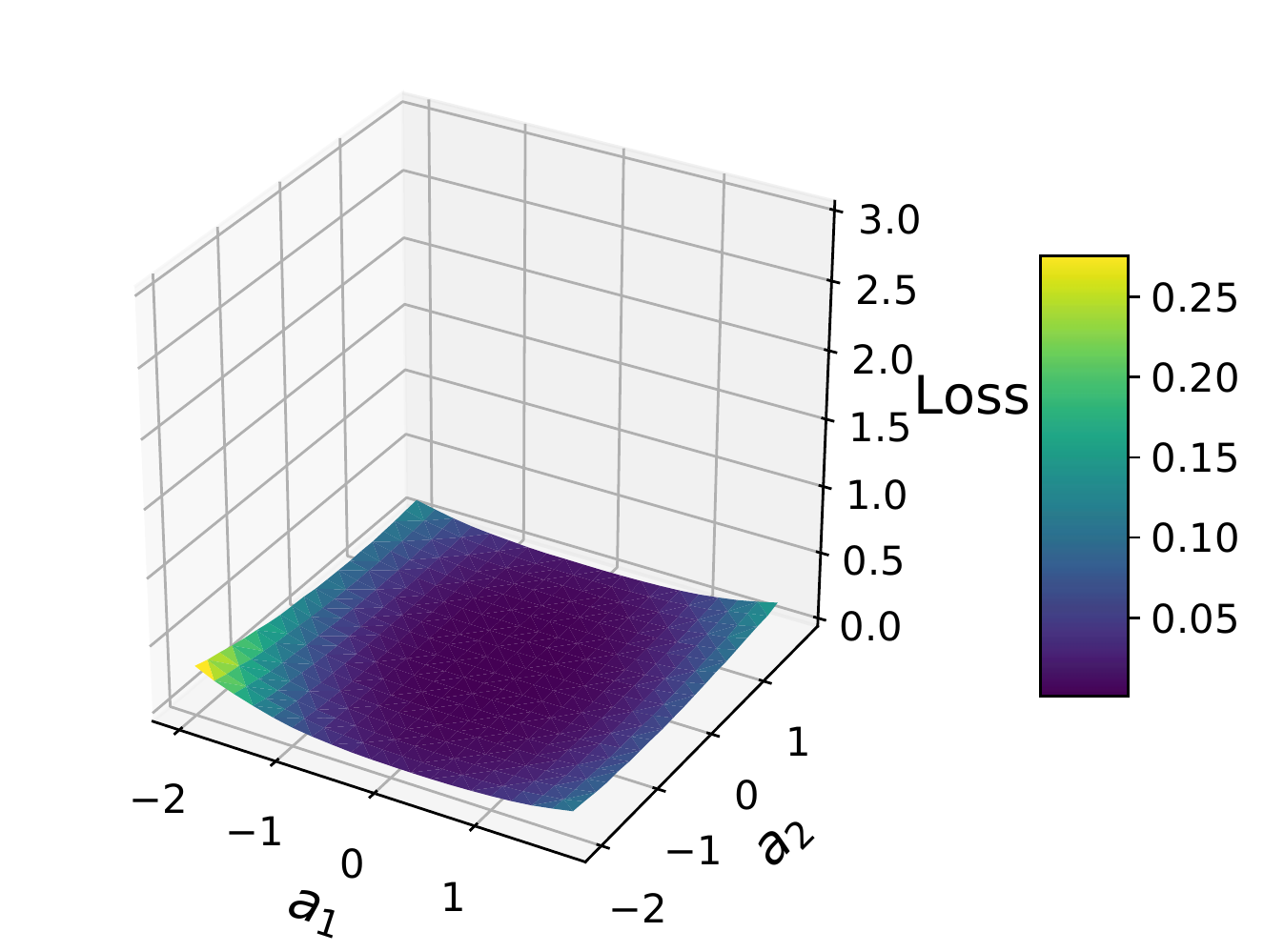}
    \captionsetup{font=small}
    \caption{CIFAR10, $\epsilon = 0$} \label{fig:landscape1_eps_cifar}
    \end{subfigure} &
    \begin{subfigure}{0.23\textwidth}
    \includegraphics[scale = 0.18]{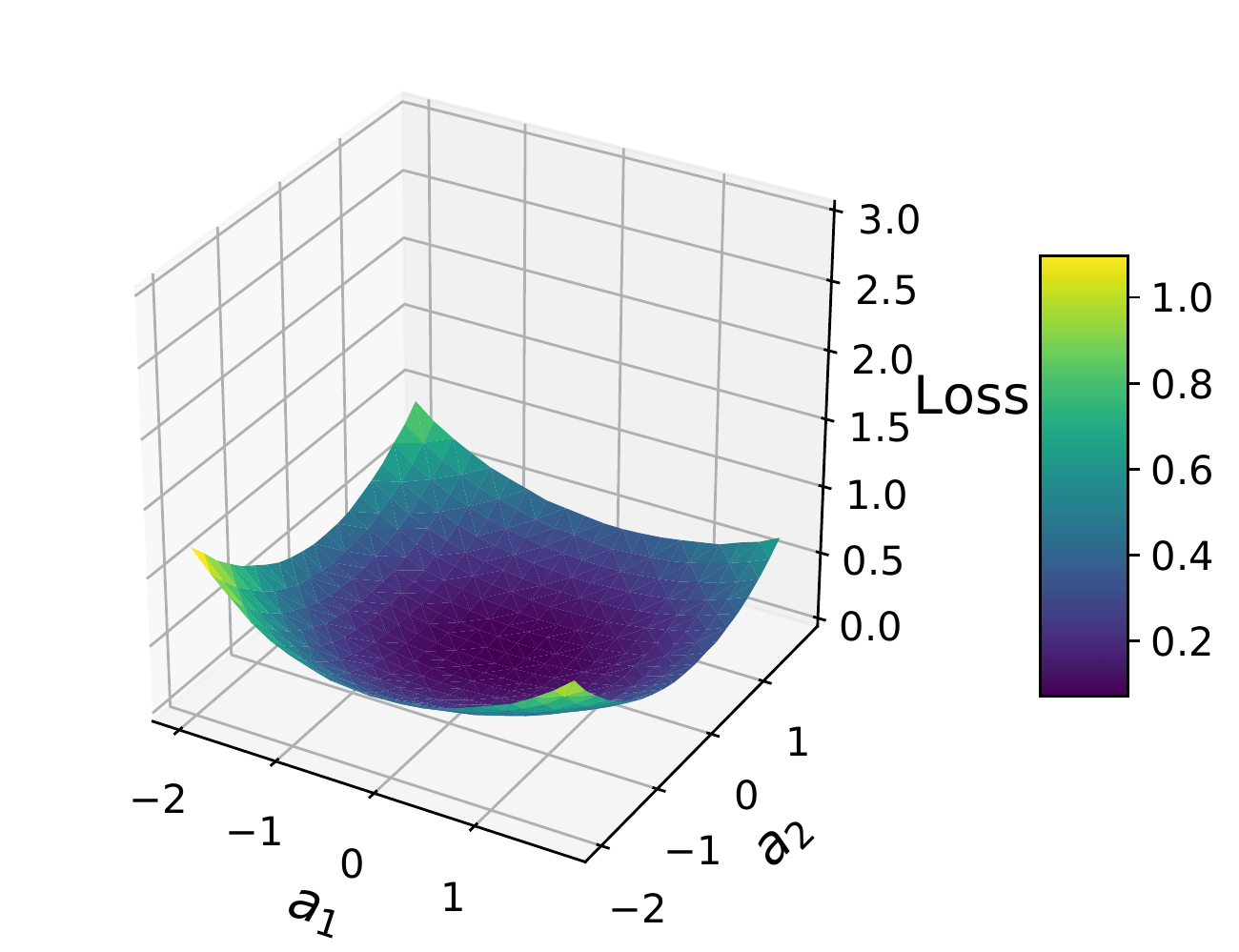}
    \captionsetup{font=small}
    \caption{CIFAR10, $\epsilon = 2 / 255$} \label{fig:landscape2_eps_cifar}
    \end{subfigure} &
    \begin{subfigure}{0.23\textwidth}
    \includegraphics[scale = 0.18]{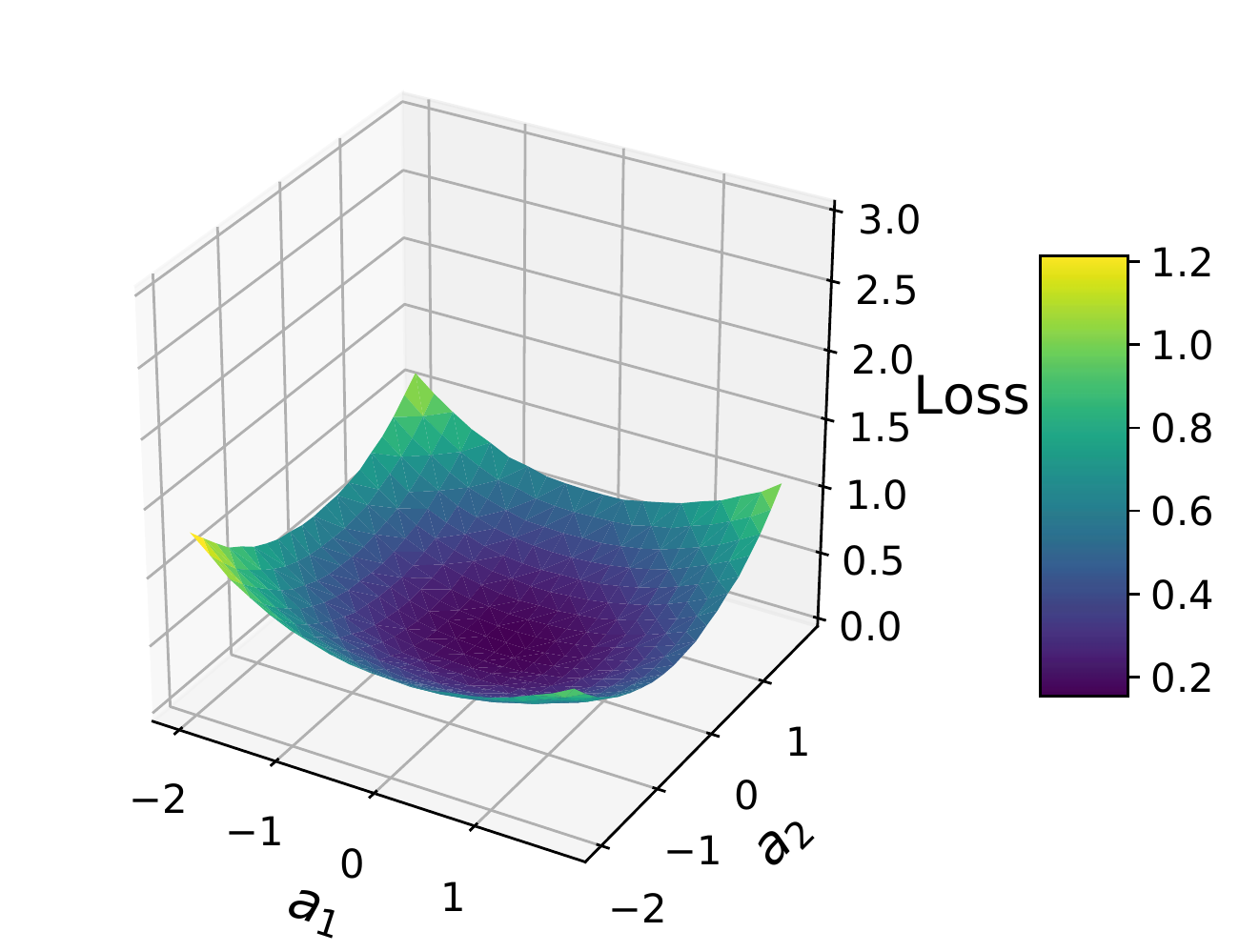}
    \captionsetup{font=small}
    \caption{CIFAR10, $\epsilon = 4 / 255$} \label{fig:landscape3_eps_cifar}
    \end{subfigure} &
    \begin{subfigure}{0.23\textwidth}
    \includegraphics[scale = 0.18]{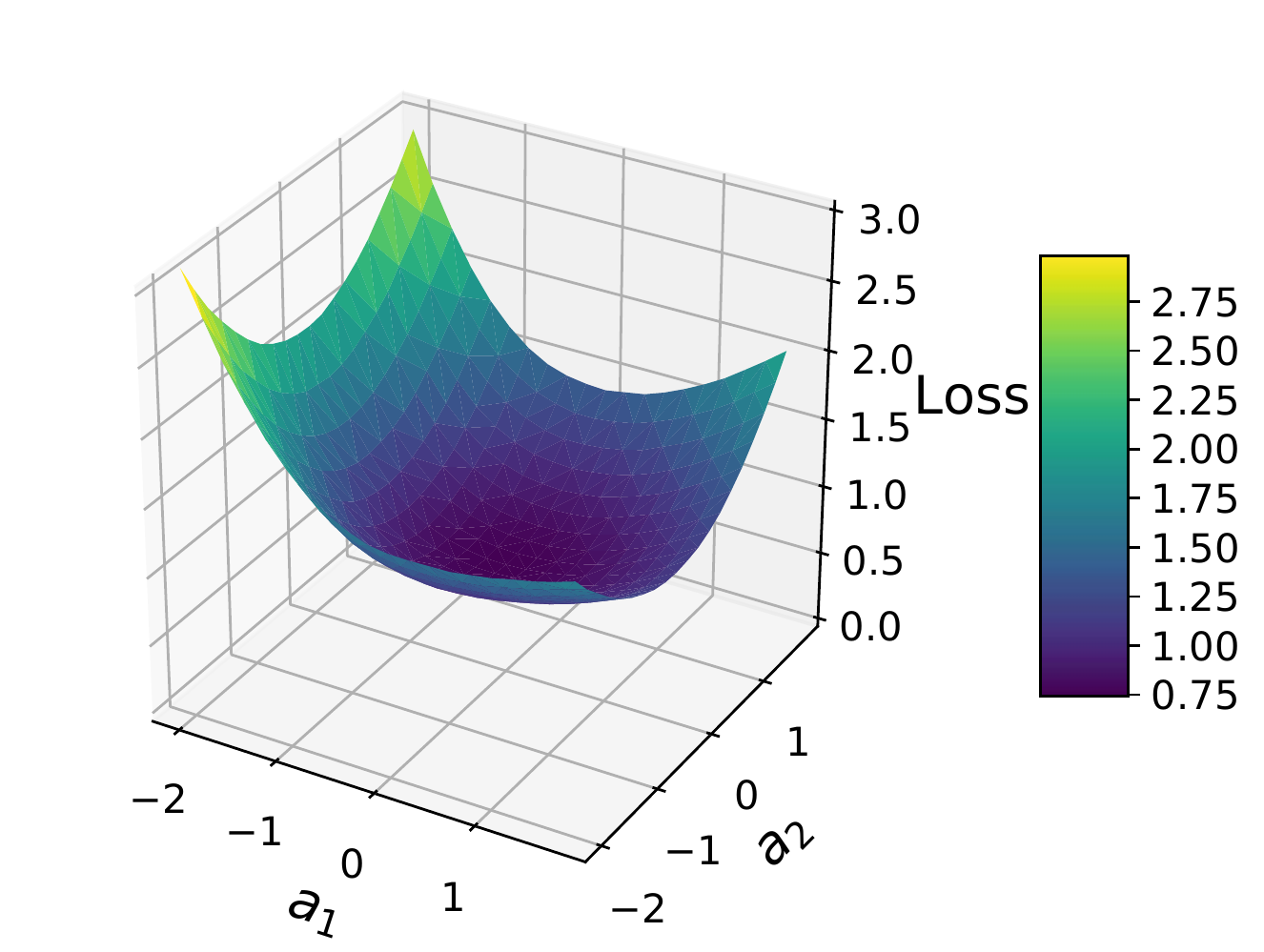}
    \captionsetup{font=small}
    \caption{CIFAR10, $\epsilon = 8 / 255$} \label{fig:landscape4_eps_cifar}
    \end{subfigure}
\end{tabular}
\caption{Loss landscape $\Loss_\epsilon(\theta + a_1 \vi_1 + a_2 \vi_2)$ under different adversarial budgets. $\theta$, $\vi_1$, $\vi_2$ are the parameter, and the first and second unit eigenvectors of the Hessian matrix. (Note that the z-scale for $\epsilon = 0.4$ in the MNIST case differs from the others.)} \label{fig:hessian_3d}
\end{figure}

In Figure~\ref{fig:sim_mnist}, we report the cosine similarity of input perturbations when we move the model parameter $\theta$ in two opposite directions.
As in Figure~\ref{fig:sim_cifar10}, we see high similarity of the perturbations and high robust accuracy when $\vi$ is a random direction.
By contrast, when $\vi$ is the first eigenvector of the Hessian matrix, we see a sharp decrease in both the perturbation similarity and robust accuracy as the value of $a$ increases.
In Figure~\ref{fig:sim_mnist}, we only plot the perturbation similarity when the robust accuracy on the training set is higher than $70\%$; otherwise the model parameters can no longer be considered to be in a small neighborhood of the original ones.

Figure~\ref{fig:hessian_3d} shows 3D visualizations of $\Loss_\epsilon(\theta)$ under different values of $\epsilon$ in the parameter neighborhood of our obtained MNIST and CIFAR10 models on the training set.
We study the curvature in the directions of the top $2$ eigenvectors.
The curvature clearly increases with $\epsilon$ and the corresponding minima become sharper.

\subsubsection{Additional Results for Section~\ref{sec:pas}} \label{subsubsec:app_pas}

% \begin{wrapfigure}{R}{0.5\textwidth}
%     \vspace{-0.5cm}
%     \begin{center}
%         \includegraphics[scale = 0.4]{plots/lenet_lr_warmup.pdf}
%     \end{center}
%     \caption{Mean and standard derivation of the test error when we use learning rate warmup and constant adversarial budget during training. The best performance (mean and standard derivation) by constant learning rate and warmup in adversarial budget is highlighted by a blue bar.} \label{fig:lr_warmup_lenet16}
%     \vspace{-0.5cm}
% \end{wrapfigure}

Here, we compare the performance of warmup in the adversarial budget and warmup in the learning rate.
As in Figure~\ref{fig:lr_lenet16}, we use the LeNet model on MNIST and set the target adversarial budget size to $0.4$.
Our warmup period consists of the first $10$ epochs: the learning rate starts at $0$ and linearly increases to the final value in the warmup period; the learning rate remains constant after the warmup period.
In Figure~\ref{fig:lr_warmup_lenet16}, we show the robust accuracy on the test set when the final learning rate is set to $1 \times 10^{-4}$, $3 \times 10^{-4}$ and $1 \times 10^{-3}$.
For comparison, we show the best performance obtained when using warmup in the adversarial budget with a blue line.
We run each experiment $5$ times.

When the final learning rate is $1 \times 10^{-4}$, the learning rate wamup performance is not as good as warmup in the adversarial budget.
When the final learning rate is $3 \times 10^{-4}$ or $1 \times 10^{-3}$, the variance of the performance becomes large.
Learning rate warmup can sometimes yield good performance but sometimes fails to converge.

\begin{figure}[!ht]
    \begin{minipage}{.46\textwidth}
        \centering
        \includegraphics[scale = 0.4]{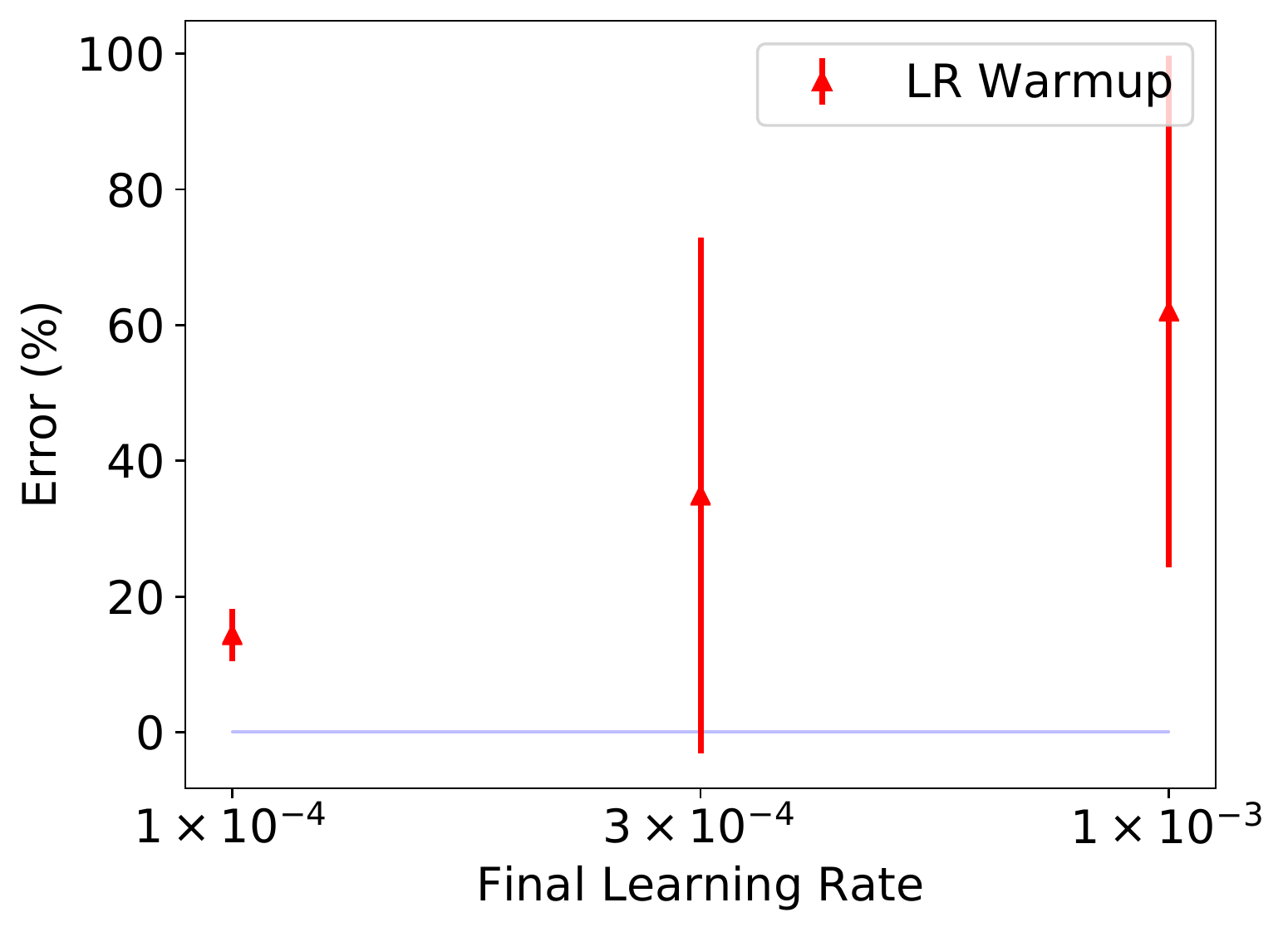}
        \caption{Mean and standard deviation of the test error on MNIST models when we use learning rate warmup but constant adversarial budget. The best performance by constant learning rate but adversarial budget warmup is depicted by a blue line.} \label{fig:lr_warmup_lenet16}
    \end{minipage}
    \begin{minipage}{.06\textwidth}
    ~~
    \end{minipage}
    \begin{minipage}{.48\textwidth}
        \centering
        \includegraphics[scale = 0.4]{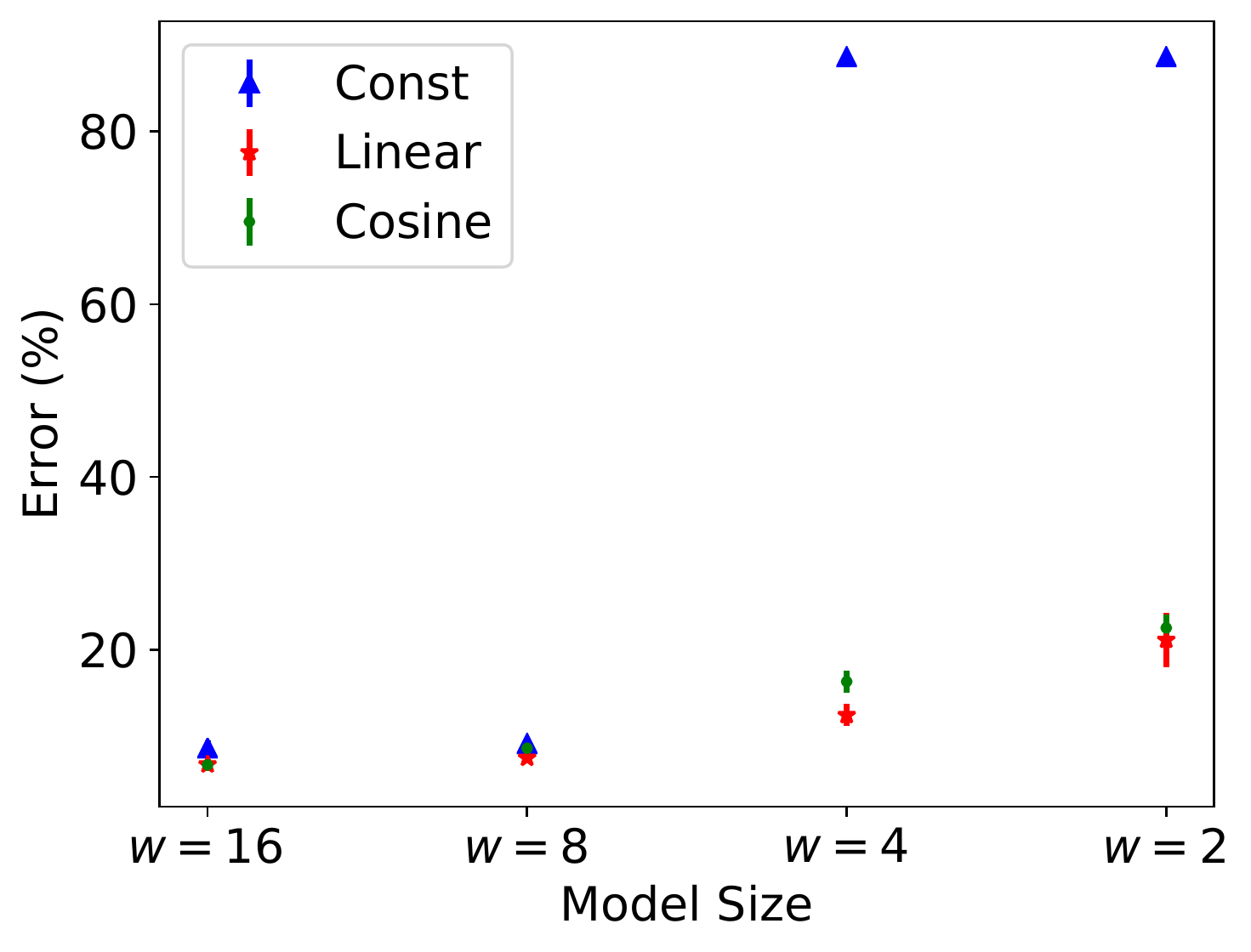}
        \caption{Mean and standard deviation of the test error with LeNet-16 models of different sizes on MNIST, using different adversarial budget scheduling schemes.} \label{fig:size_compare}
    \end{minipage}
\end{figure}

\subsubsection{Robustness v.s. Model Capacity} \label{subsubsec:capacity_app}

In Figure~\ref{fig:size_compare}, we report the performance of LeNet models of different width factors $w$ using different schedulers for $\epsilon$.
The adversarial budget size $\epsilon$ at test time is $0.4$.
We set the learning rate in Adam to be $10^{-4}$, because, for constant $\epsilon$ during training, it yields the best performance.
Both Cosine and Linear schedulers outperform using a constant $\epsilon$ in all cases.
When the model size is small, e.g., $w = 4$ and $w = 2$, using a constant $\epsilon$ during training fails to converge, but the Cosine and Linear schedulers still yield competitive results.

\subsubsection{Connectivity of Different Minima} \label{subsec:connect}

The minima reached in the loss landscape of vanilla training have been found to be well connected~\cite{draxler2018essentially, garipov2018loss}.
That is, if we train two neural networks under the same settings but different initializations, there exists a path connecting the resulting two models in the parameter space such that all points along this path have low loss.
In this section, we study the connectivity of different trained models in adversarial training.
Similarly to~\cite{garipov2018loss}, we parameterize the path joining two minima using a \textit{general Bezier curve}.
Let $\theta_0$ and $\theta_n$ be the parameters of two separately-trained models, and $\{\widehat{\theta}_i\}_{i = 1}^{n - 1}$ the parameters of $(n-1)$ trainable intermediate models. Then, an $n$-order Bezier curve is defined as a linear combination of these $(n + 1)$ points in parameter space, i.e., 
\begin{equation}
\BB(t) = (1 - t)^n \theta_0 + t^n \theta_n + \sum_{i = 1}^{n - 1} \left(\begin{matrix}n \\ t\end{matrix} \right) (1 - t)^{n - i} t^i \widehat{\theta}_i\;. \label{eq:curve}
\end{equation}

$\BB(t)$ is a smooth curve, and $\BB(0) = \theta_0$ and $\BB(1) = \theta_n$.
We train $\{\widehat{\theta}_i\}_{i = 1}^{n - 1}$ by minimizing the average loss along the path: $\E_{t \sim U[0,1]}\Loss_\epsilon(\BB(t))$, where $U[0, 1]$ is the uniform distribution between $0$ and $1$.
We use the Monte Carlo method to estimate the gradient of this expectation-based function and minimize it using gradient-based optimization.
We use second-order Bezier curves to connect MNIST model pairs and fourth-order Bezier curves to connect CIFAR10 model pairs.
When evaluating the models on the learned curves, we re-estimate the running mean and variance in the batch normalization layer based on the training set.
The results are reported based on the evaluation mode of the models, and we turn off data augmentation to avoid stochasticity.

\begin{figure}
\centering
\begin{tabular}{cc}
    \begin{subfigure}{0.4\textwidth}
    \includegraphics[scale = 0.36]{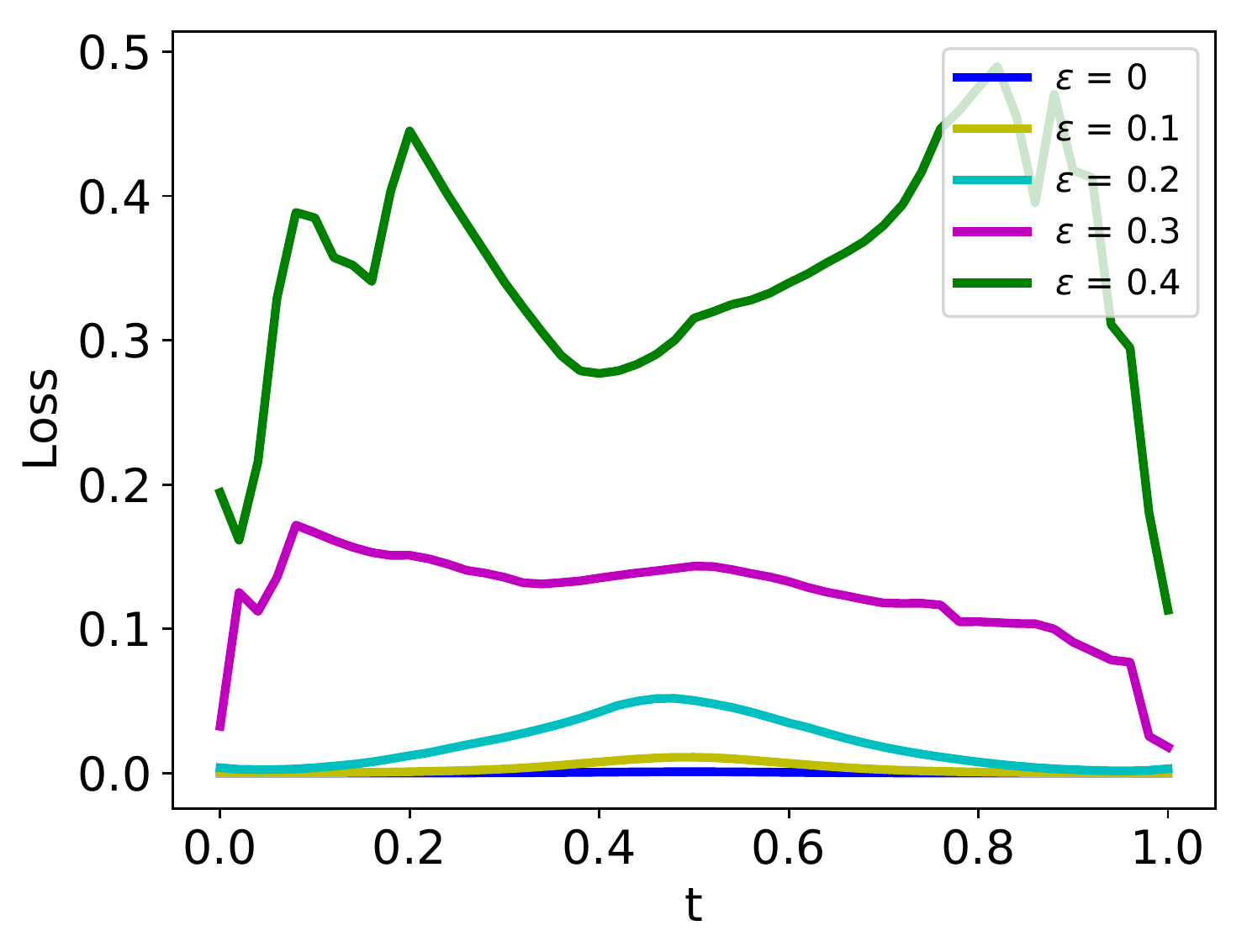}
    \caption{MNIST, training loss.}
    \end{subfigure}
    \begin{subfigure}{0.4\textwidth}
    \includegraphics[scale = 0.36]{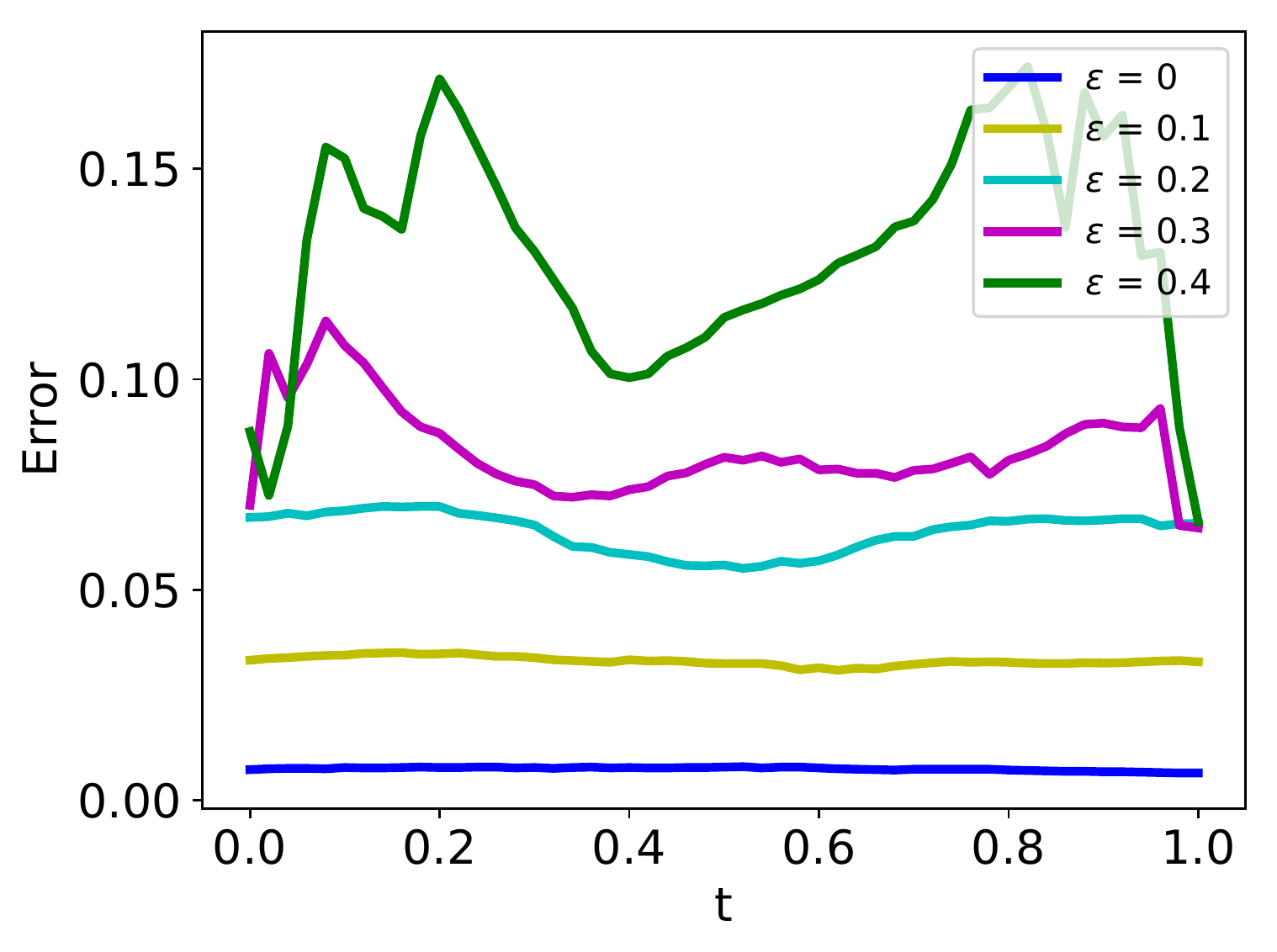}
    \caption{MNIST, test error.}
    \end{subfigure} \\
    \begin{subfigure}{0.4\textwidth}
    \includegraphics[scale = 0.36]{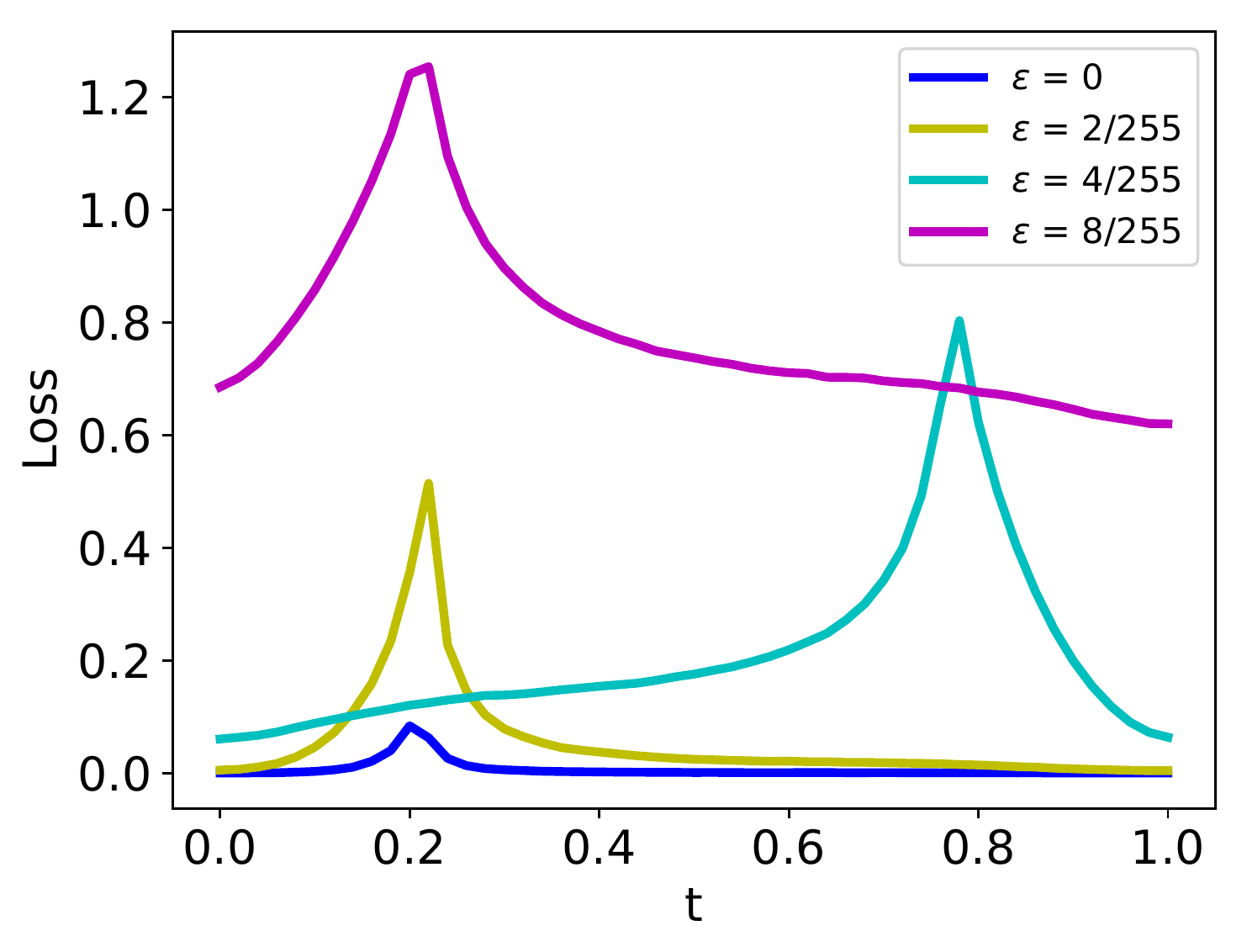}
    \caption{CIFAR10, training loss.}
    \end{subfigure}
    \begin{subfigure}{0.4\textwidth}
    \includegraphics[scale = 0.36]{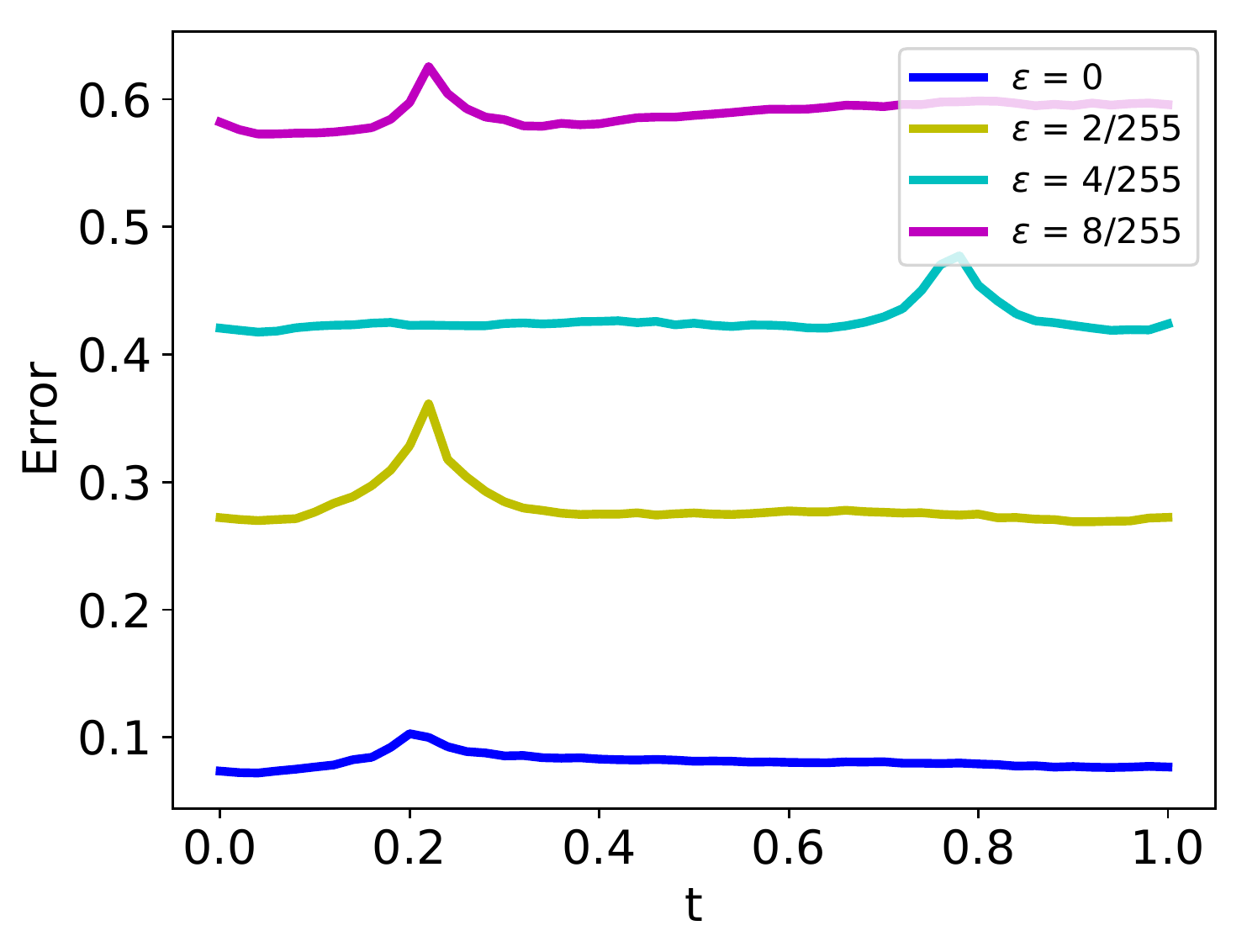}
    \caption{CIFAR10, test error.}
    \end{subfigure}
\end{tabular}
\caption{Training loss and test error along the path connecting the minima of two independently-trained models.} \label{fig:connect_curve}
\end{figure}

In Figure~\ref{fig:connect_curve}, following~\cite{garipov2018loss}, we plot the training loss and test error along the learned curve, as a function of $t$ in Equation (\ref{eq:curve}).
For vanilla training or when the adversarial budget is small, we can easily find flat curves connecting different minima.
However, the learned curves are not flat anymore when the adversarial budget increases.
This indicates that the minima are less well-connected under adversarial training, and that it is more difficult for the optimizer to find a minimum.

\end{document}